\documentclass[a4paper,12pt]{article}
\usepackage[ruled]{algorithm2e}
\usepackage{algorithmic}
\usepackage{amssymb,,amsmath,epsfig,amsthm,authblk,dsfont,bm}
\usepackage{enumitem}
\usepackage{smile}
\usepackage[numbers,sort&compress]{natbib}
\usepackage[margin=1in]{geometry}
\usepackage{setspace,color}
\usepackage{tikz}
\usepackage{epstopdf}
\usepackage{grffile}
\usepackage{makecell}
\usepackage{multirow}
\usepackage{graphbox}
\usepackage{lineno}

\usepackage{hyperref}
\hypersetup{
	colorlinks=true,
	linkcolor=blue,
	citecolor=blue,
	filecolor=blue,      
	urlcolor=blue
}

\providecommand{\keywords}[1]{\textbf{Key words:} #1}

\newcommand{\bone}{\mathbf{1}}

\newcommand{\Sig}{\mathrm{Sig}}

\allowdisplaybreaks
\begin{document}
	\title{PottsMGNet: A Mathematical Explanation of Encoder-Decoder Based Neural Networks }
	\date{}
	\author{
 Xue-Cheng Tai\thanks{NORCE Norwegian Research Centre, Nyg\r{a}rdstangen, NO-5838 Bergen, Norway Email: xtai@norceresearch.no, xuechengtai@gmail.com. The work of Xue-Cheng Tai is partially supported by NSFC/RGC grant N-HKBU214-19 and NORCE Kompetanseoppbygging program.}, 
 Hao Liu\thanks{Corresponding author. Department of Mathematics, Hong Kong Baptist University, Kowloon Tong, Hong Kong. Email: haoliu@hkbu.edu.hk. The work of Hao Liu is partially supported by HKBU 179356 and NSFC 12201530.}, 
 Raymond Chan\thanks{Department of Mathematics, City University of Hong Kong, Hong Kong. 
 Hong Kong Centre for Cerebro-Cardiovascular Health Engineering. Email:
 raymond.chan@cityu.edu.hk. The work of Raymond Chan is partially supported by
HKRGC GRF grants CityU1101120, CityU11309922 and CRF grant C1013-21GF.}
	}
	\maketitle
	
	\begin{abstract}
		For problems in image processing and many other fields, a large class of effective neural networks has encoder-decoder-based architectures. Although these networks have made impressive performances, mathematical explanations of their architectures are still underdeveloped. In this paper, we study the encoder-decoder-based network architecture from the algorithmic perspective and provide a mathematical explanation. We use the two-phase Potts model for image segmentation as an example for our explanations.  We associate the segmentation problem with a control problem in the continuous setting. Then, multigrid method and operator splitting scheme, the PottsMGNet, are used to discretize the continuous control model. We show that the resulting discrete PottsMGNet is equivalent to an encoder-decoder-based network. With minor modifications, it is shown that a number of the popular encoder-decoder-based neural networks are just instances of the proposed PottsMGNet.  By incorporating the Soft-Threshold-Dynamics into the PottsMGNet as a regularizer, the PottsMGNet has shown to be robust with the network parameters such as network width and depth and achieved remarkable performance on datasets with very large noise. In nearly all our experiments, the new network always performs better or as good on accuracy and dice score than existing networks for image segmentation.   
	\end{abstract}
	
	\section{Introduction}
	Deep neural networks have demonstrated great performances in many image processing tasks, such as image segmentation \cite{ronneberger2015u,long2015fully,chen2017deeplab, badrinarayanan2017segnet, liu2022deep,zhou2019unetpp}, image denoising \cite{zhang2017beyond,anwar2019real}, etc. Although deep neural networks have provided remarkable results, mathematical explanations for their success are still underdeveloped. 
	
	A popular explanation or intuition for network architecture is the encoder-decoder framework which  decomposes the network into two parts: the encoder part and the decoder part. When a high-dimensional variable is passed to the network, the encoder part first uses several computational layers and downsampling layers to extract useful low-dimensional features. Then these features are passed to the decoder to reconstruct the desired high-dimensional output. Such an architecture is widely used in networks for image segmentation, such as UNet \cite{ronneberger2015u}, UNet++ \cite{zhou2019unetpp}, DeepLab \cite{chen2017deeplab}, and SegNet \cite{badrinarayanan2017segnet}. While the encoder-decoder framework only provides a general picture of the network architecture, i.e., the decoder does a dimension reduction and the decoder does a data reconstruction, more detailed explanations, such as what is the role of each layer in the network, and what mathematical model we are solving by this network, is still unclear. 
	
	For image segmentation, a large class of models is based on theories of min cut/ max flow. The continuous max flow and min cut problems are studied in \cite{bae2011global,yuan2010study}. Stemming from Markov random fields, many efficient methods are designed based on graph cuts \cite{boykov2001fast}. A very important mathematical model for image segmentation is the Potts model  \cite{potts1952some}.  Application of the Potts model for segmentation and classification \cite{boykov2006graph,boykov2004experimental,yuan2010continuous} have been extensively studied for discrete graph settings and can be efficiently solved by max flow algorithms. The new survey \cite{tai2021potts} contains a comprehensive overview of the Potts model and its fast algorithms. The Potts model was first proposed for statistical mechanics in \cite{potts1952some}, and can be taken as a generalization of the two-state Ising model to lattice \cite{pock2009convex}. It is used for binary graph cuts in \cite{boykov2001fast}. In \cite{boykov2001fast,boykov2006graph}, efficient graph cut algorithms, equipped with fast graph cut techniques, are proposed for min cut problems such as image segmentation with the Potts model.
	It has been shown in \cite{bae2011global,yuan2010study} that the Potts model is equivalent to a continuous min cut and max flow problem: If the Potts model is discretized with certain approximations, it reduces to existing graph cut models, see \cite{wei2018new} for some more detailed explanations. In \cite{yuan2010study}, the alternating direction method of multipliers (ADMM) is used to solve the Potts model. However, the convergence of the algorithm is not guaranteed. Recently, in \cite{sun2021efficient}, based on the Eckstein-Bertsekas \cite{eckstein1992douglas} and Fortin-Glowinski \cite{fortin1983chapter} splitting techniques, two novel preconditioned ADMMs for the Potts model with guaranteed convergence are proposed.
	
	Operator-splitting methods are powerful tools for solving complicated optimization problems. In general, an operator-splitting method decomposes a complicated problem into several easy-to-solve subproblems so that each subproblem either has a closed-form solution or can be solved efficiently. Based on how these subproblems are solved, operator-splitting methods can be divided into parallel splitting methods \cite{lu1992parallel} and sequential splitting methods \cite{glowinski2003finite,marchuk1990splitting}. As indicated by the name, parallel splitting methods solve all subproblems in parallel and then combine the results by averaging; sequential splitting methods solve subproblems sequentially. In fact, ADMM is a special type of operator-splitting method. Operator-splitting methods have been applied in numerical methods for partial differential equations \cite{glowinski2019finite,liu2019finite}, inverse problems \cite{glowinski2015penalization}, computational fluid dynamics \cite{bonito2016operator,mrad2022splitting}, obstacle problems \cite{liu2022fast}, surface reconstructions \cite{he2020curvature} and image processing \cite{liu2021color,duan2022fast,deng2019new,liu2022operator}. Compared to ADMM, operator-splitting methods have fewer parameters and are not sensitive to the choices of parameters. For problems from image processing, it is shown in \cite{deng2019new,duan2022fast} that operator-splitting methods are more efficient. We refer readers to \cite{glowinski2017splitting,glowinski2019fast,glowinski2016some} for a comprehensive discussion on operator-splitting methods.
	
	In this paper, we provide a mathematical explanation of encoder-decoder-based convolutional neural networks for image segmentation from the perspective of mathematical models and algorithms. We use the two-phase image segmentation as an example, but the explanations can be generalized to multiphase and other classification problems as well. We solve the two-phase segmentation problem using the Potts model with the length represented by Soft-Threshold-Dynamics \cite{alberti1998non,merriman1992diffusion,liu2022deep}. We derive the Euler-Lagrange equation of the problem and associate it with an initial value problem with control variables. We propose a novel operator-splitting method, the hybrid splitting, which combines parallel splitting and sequential splitting to discretize the control problem. By incorporating the hybrid splitting method with the multigrid method and with proper choices of the control variables, we obtain a scheme we call PottsMGNet, which has the same architecture as an encoder-decoder-based neural network. Each layer in the network corresponds to a substep of the splitting scheme, and its convolutional kernels and biases correspond to the control variables.
	Our contributions can be summarized as follows:
	\begin{itemize}
		\item We provide a clear and concise mathematical explanation of a large class of encoder-decoder-based neural networks, which are essentially some operator-splitting schemes with multigrid method for some optimal control problems. These explanations are important for designing and improving neural networks, providing clear guidelines for the choice of the number of layers and neurons on each layer, and making the networks more explainable, with each part having a clear mathematical meaning.
		
		\item Our proposed PottsMGNet is a multigrid-based numerical splitting scheme for solving control problems. With proper settings, our framework recovers the widely used neural networks in the literature for image segmentation problems, just with different activation functions. Each layer in the network corresponds to a substep of the splitting scheme, and its convolutional kernels and biases are the control variables. 
		
		\item Numerical tests on different datasets show that PottsMGNet is robust with respect to network parameters and performs better or as well as existing popular encoder-decoder neural networks. It can handle data with large noise and is more robust than other networks with similar architectures.
		
		\item We also develop a novel hybrid splitting method for solving time-evolutional equations. Our hybrid splitting method incorporates parallel splitting and sequential splitting and we prove that it is first-order accurate. In the current work, they are used to get the neural networks. They can be used for solving other problems as well. 
	\end{itemize}
	
	This work has been inspired by a wealth of pioneering research seeking mathematical explanations for neural networks. In particular, we would like to mention several key references, including \cite{Weinan2017, Weinan2020,  sussillo2014opening, funahashi1989approximation, Masters1993, Zhou2020a, He2019a, ruthotto2020deep,  Haber2018, Haber2018b, Cheng2023a, Benning2019,celledoni2021structure,benning2019deep}. 
	It has been established that neural networks possess universal approximation properties \cite{Masters1993, Zhou2020a}. The work of E and co-authors \cite{Weinan2017, Weinan2020} proposed to view neural networks as continuous dynamical systems and as special discretizations of continuous problems. This has been a source of inspiration for our work. Sussillo and Barak \cite{sussillo2014opening} explored low-dimensional dynamics in high-dimensional recurrent neural networks. Ruthotto et al. \cite{ruthotto2020deep} introduced deep neural networks motivated by partial differential equations, and ODE concepts are used in Haber and Ruthotto \cite{Haber2018} to obtain stable architectures for deep neural networks. The relation between deep neural networks and control problems is investigated in \cite{benning2019deep}. Many researchers have observed that UNet and other encoder-decoder neural networks are related to multiscale techniques \cite{Haber2018b, Haber2018}. Haber and Ruthotto \cite{Haber2018b} proposed multiscale methods for convolutional neural networks. In \cite{He2019a}, MgNet is proposed as a unified framework that uses multigrid linear operators as feature extraction operators in traditional convolutional neural networks. The similarity between neural networks and operator splitting methods are mentioned and utilized in \cite{lan2022dosnet}. Continuous UNet using a second-order ODE has been proposed in \cite{Cheng2023a}.  A systematic review is provided in \cite{celledoni2021structure}. However, none of the aforementioned references have explained neural networks as an operator splitting discretization with a multigrid spatial approximation for the continuous Potts model considered as a control problem, as we have done in this work.
	
	This paper is structured as follows.
	In Section \ref{sec.PottsNet}, we introduce the Potts model. 
 In Section \ref{sec.control}, we formulate a control problem to solve the Potts model and our framework to learn the control variables. We introduce PottsMGNet, an operator splitting scheme, for the control problem and its connections to neural networks in Section \ref{sec.PottsMGNet}. In Section \ref{sec.dis}, we discuss the discretization and solution to each subproblem. In Section \ref{sec.relation}, we show that PottsMGNet is in fact a neural network and discuss how to modify it to recover existing popular neural networks for image segmentation. In Section \ref{sec.experiment}, we present our numerical experiment results, and we conclude this paper in Section \ref{sec.conclusion}.
	
	\paragraph{Notation:} Throughout this paper, we use lowercase letters to denote scalar variables and functions. Bold letters are used to denote vectors and vector-valued functions. Capital letters are used to denote tensors and operators. Calligraphic letters are used to denote sets.
	
	\section{Introduction to the Potts model}
	\label{sec.potts}
	\label{sec.PottsNet}
	
	In this section, we  give a brief introduction to the Potts model \cite{potts1952some,tai2021potts}. See Appendix \ref{sec.potts.derivation} for  more detailed derivations in getting this model. 
	
	We use image segmentation to present the Potts model and its numerical algorithms. Let $\Omega$ be the image domain. The continuous two-phase Potts model is of the form:
	\begin{align}
		\begin{cases}
			\min\limits_{\Omega_0,\Omega_1} \left\{ \sum_{k=0}^1 \displaystyle\int_{\Omega_k} g_k(\xb) d\xb+\frac{1}{2} \displaystyle\sum_{k=0}^1 |\partial \Omega_k| \right\},\\
			\Omega_0\cup \Omega_1=\Omega,\ 
			\Omega_0\cap \Omega_1=\emptyset,
		\end{cases}
		\label{eq.potts2.0}
	\end{align}
	where $|\partial\Omega_k|$ is the perimeter of $\Omega_k$, and $g_k$'s are nonnegative weight functions depending on the input image.  By solving (\ref{eq.potts2.0}), the image domain is segmented into two regions: $\Omega_0$ and $\Omega_1$.
	
	By utilizing a regularized softmax operator \cite{liu2022deep} and techniques from threshold dynamics \cite{ alberti1998non,merriman1992diffusion,esedog2015threshold,liu2022deep}, the two-phase Potts model (\ref{eq.potts2.0}) can be approximated by the following problem (see Appendix \ref{sec.potts.derivation} for the derivation)
	\begin{multline}
		\min_{v(\xb)\in [0,1]} \Bigg[\int_{\Omega} v gd\xb+ \varepsilon\int_{\Omega} (v\ln v +(1-v)\ln (1-v))d\xb  
		\\ 
		+ \eta\int_{\Omega} v(\xb)(G_\sigma\ast (1-v))(\xb) d\xb\Bigg],
		\label{eq.potts2.4}
	\end{multline}
	for some constant $\varepsilon>0$ and $\eta\geq0$, where $g=g_1-g_2$, $G_\sigma$ is the Gaussian kernel $
	G_\sigma(\xb)=\frac{1}{2\pi\sigma^2} \exp\left(-\frac{\|\xb\|^2}{2\sigma^2}\right)$ for some given $\sigma>0$.  As usual $\ast$ denotes the convolution operator. By minimizing (\ref{eq.potts2.4}) and denoting the minimizer by $u$, the image domain is segmented into two regions corresponding to $u\leq 0.5$ and $u>0.5$. In (\ref{eq.potts2.4}), the second integral is used to get the regularized softmax operator, see Appendix \ref{sec.potts.derivation} for detailed explanations about this. The third integral approximates the perimeter of the two regions. One can show that as $\varepsilon\rightarrow 0$, we have $u(\xb)\in\{0,1\}$ for any $\xb\in\Omega$, see Lemma \ref{lem.limit} in Appendix \ref{sec.potts.derivation}.
	
	If the minimizer $u$ exists, it can be proven that $u\in (0,1)$ and it satisfies the Euler-Lagrangian equation:
	\begin{align}
		\varepsilon \ln \frac {u}{1-u} + \eta G_\sigma*(1-2u) +g = 0, \quad \forall \xb \in \Omega.
		\label{eq.Potts.EL}
	\end{align} 
	The corresponding gradient flow equation on the time interval $(0,T]$ with a proper initial state $u_0(\xb)$ is given as:
	\begin{align}
		\begin{cases}
			u_t=-\varepsilon \ln \frac {u}{1-u} - \eta G_\sigma*(1-2u) -g, \ (\xb,t) \in \Omega\times (0,T], \\ 
			u(\xb, 0)      = u_0, \ \xb \in \Omega. 
		\end{cases}		
		\label{eq.Potts.EL.ivp}
	\end{align} 
	\begin{remark}
		Following the formulations in Appendix \ref{sec.potts.derivation}, we can easily extend our model and algorithms proposed later to multiphase segmentation and classification problems. To make the presentations clear, we will stay with the two-phase problem in this work. 
	\end{remark}

	\section{A control problem for the Potts model}
	\label{sec.control}

	For image segmentation, we assume $f$ is a given image defined on the image domain $\Omega$. We take the initial function $u_0=H(f)$ for some appropriate operator $H$ to be specified later. Based on (\ref{eq.Potts.EL}), we consider the following gradient flow equation for the Potts model with control variables $W(\xb,t)$ and $d(\xb,t)$: 
	\begin{equation}
		\left\{
		\begin{array}{l}
			\displaystyle \frac {\partial u}{\partial t}  = -\varepsilon \ln  \frac u {1-u} - \eta G_\sigma*( 1-2u) 
			+ W(\xb,t) \ast u +  d(\xb,t) , 
			\ (\xb,t) \in \Omega\times (0,T], \\ 
			u(\xb, 0)      = H(f), \ \xb \in \Omega. 
		\end{array}
		\right.  
		\label{eq:potts.control.k=2}
	\end{equation}
	To clarify, in (\ref{eq:potts.control.k=2}), we are considering a Potts model with two labels on a domain $\Omega\subset \mathbb{R}^2$. The function $u(\xb,t):\Omega \times [0,T]\rightarrow \mathbb{R}$ represents the label of each point $\xb$ in the domain at time $t$. The term $W(\xb,t) * u$ is the convolution of $u$ with a weight function $W(\xb,t): D \times [0,T]\rightarrow \mathbb{R} $ for some domain $D\subset \RR^2$, which essentially represents a weighted average of the labels of neighboring points. The support of $W(\xb,t)$ is normally small and the domain $D$ may differ from $\Omega$. The parameters $\varepsilon$ and $\eta$ control the approximation of the binary functions and are used to regularize the labels. Specifically, the length regularization with the term associated with this is used to ensure that neighboring points have similar labels, thereby promoting spatial coherence of the labeling. The parameter $\eta$ controls the strength of this regularization, with larger values promoting more spatial coherence. 
	
	The function $g(\xb,t)$ in (\ref{eq.Potts.EL}) is replaced by $W(\xb,t) * u + d(\xb,t)$ in (\ref{eq:potts.control.k=2}), where $W(\xb,t)$ and $d(\xb,t)$ are control functions that can be used to steer the final state $u(\xb,T)$ to some desirable state. In this way, $W(\xb,t)$ and $d(\xb,t)$ are treated as control variables that can be adjusted to achieve the desired labeling of the domain at the final time $T$. In (\ref{eq:potts.control.k=2}), the control variable $W(\xb,t)$ is applied to $u$ by convolution since convolution is widely used in deep learning methods for image processing. We remark that the convolution in (\ref{eq:potts.control.k=2}) (and those discussed in Section \ref{sec.PottsMGNet}) can be replaced by any other linear operations.
	
	In this context, we are considering an optimal control approach to obtain a segmentation operator for two-phase image segmentation. Specifically, we are given a set of $I$ images $f_i:\Omega \rightarrow \mathbb{R}^3$ that have been segmented into foreground and background regions, denoted by $v_i:\Omega \rightarrow \{0,1\}^2$ for $i=1,2,\dots,I$. Denote  $\cN_1$ as the mapping  from $ f\rightarrow u(\xb,T)$, i.e. $\cN_1(f,\theta_1) = u(x,T)$ with $u(x,T)$ being the solution of (\ref{eq:control}) at time $T$, where $\theta_1=(W(\xb,t), d(\xb,t))$ denotes the control variables. We seek to find a set of control variables $\theta_1$ that minimize the distance between the segmentation result obtained by applying the mapping $\mathcal{N}_1$ to each image $f_i$ and its corresponding ground truth segmentation $v_i$:
	\begin{equation}\label{eq:control}
		\min_{\theta_1} \sum_{i=1}^I \mathcal{L}(\mathcal{N}_1(f_i,\theta_1) , v_i ).
	\end{equation}
	The distance measure used is the cross entropy, which is a common choice for comparing probability distributions:
	$$
	\mathcal{L}(u,v) = -u  \ln v - (1-u)\ln(1-v).
	$$
	Formally, the optimization problem we solve is given by (\ref{eq:control}), where $\mathcal{L}(\cdot,\cdot)$ measures the distance between two segmentation results, and $\mathcal{N}_1(f_i,\theta_1)$ represents the segmentation result obtained by applying the mapping $\mathcal{N}_1$ to the image $f_i$ using the control variables $\theta_1$. Since the control variables $\theta_1$ appear in the mapping $\mathcal{N}_1$, we can view (\ref{eq:control}) as an optimal control problem with multiple targeting states $v_i$ and control variables $\theta_1$. The control equation (\ref{eq:potts.control.k=2}) contains nonlinear and bilinear terms, which makes the problem more challenging to solve. However, the controllability of a similar control equation has been shown for a single targeting state in \cite{Khapalov2001}, which justifies the suitability of our approach.
	
	\section{Our PottsMGNet}
	\label{sec.PottsMGNet}
	To solve the optimal control problem (\ref{eq:control}) numerically, we need two discretizations: one for the control variables $\theta_1=(W(\xb,t), b(\xb,t))$ and another one for the corresponding state variable $u$. To make the presentation clear, we will only consider the discretization of $\theta_1$ now and leave the discretization of the state variable later. So the state variables $u$ are continuous functions in this section and will be discretized in Section \ref{sec.dis}.
	
	\subsection{Multigrid discretizations}
	The image domain $\Omega$ and the convolution kernel domain $D$ may be different. We assume that both $\Omega$ and $D$ have been discretized in the multigrid setting as explained in Appendix \ref{app.multigrid}. As convolution is used, we need to extend the values of $u$ and $\theta_1$ outside the domains. Due to this, we will simply take $\Omega=D=\mathbb{R}^2$ and apply zero padding for $\theta_1$ and $u$.  Correspondingly, $\cT^j, \cV^j$ are the multigrids and spaces over $\mathbb{R}^2$ with the corresponding padding techniques, where $j$ denotes the grid level.  In this paper, we use $\cT^1$ to represent the finest grid (the image's original resolution). The grid gets coarser as $j$ increases. We will use the V-cycle multigrid technique, which consists of a left branch and a right branch. In the left branch, computations are conducted from fine grids to coarse grids sequentially. In the right branch, computations are conducted from coarse grids to fine grids sequentially. See Figure \ref{fig.V} for an illustration of a V-cycle of the multigrid method. Our notations for multigrid follow more from \cite{xu1992iterative, tai2002global,chan1994domain}. 
	\begin{figure}[t!]
		\centering
		\includegraphics[width=0.6\textwidth]{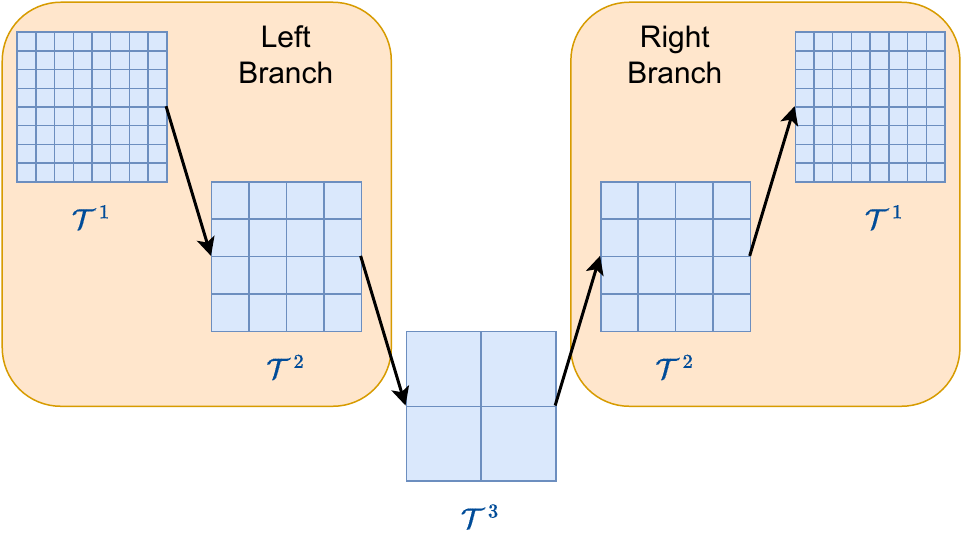}
		\caption{An illustration of a V-cycle of multigrid method.}
		\label{fig.V}
	\end{figure}

	\subsection{A basic decomposition for $\theta_1$}
	\label{sec.BasicAlg}
	
	The approach presented in \cite{tai2003rate,tai1998subspace,tai2002global} for solving the optimization problems in the discrete multigrid setting involves decomposing the search space into many subspaces or subsets. The goal is to search for the minimizer iteratively through these subspaces either in parallel or sequentially. This approach can be seen as decomposing the minimization variable into a sum of variables from different subspaces.
	The traditional V-cycle multigrid method is a sequential iterative procedure that operates over the multigrid subspaces defined in Appendix \ref{app.multigrid} in a specific order, as shown in Figure \ref{fig.V}.

	Traditionally, space decomposition and subspace correction are used to interpret multigrid and domain decomposition methods \cite{xu1992iterative,chan1994domain}.  We will apply the space decomposition ideas \cite{tai2003rate,tai1998subspace,tai2002global} for solving the optimization problems (\ref{eq:control}). This idea has been used to decompose a large space into the sum of smaller spaces and then solve the problems iteratively over the smaller spaces. But here, it is used differently. The control variable $\theta_1$ is decomposed into a sum of variables from subspaces. One purpose for this decomposition is to increase the number of unknowns for the control variables. This is different from earlier usages for the space decomposition ideas. Then we will use a hybrid splitting scheme to split the components of $\theta_1$ into several subproblems, which are solved in parallel or sequentially. The general idea is that all components of $\theta_1$ are gone through when we apply one iteration of the scheme. Details of the hybrid splitting scheme are discussed in Appendix \ref{sec.hybrid}. To make the presentation clear, we will present the ideas step by step in the following on how to decompose the control variable $\theta_1$. Our idea is to decompose $\theta_1$ so that the control variables and operators have a similar form as those in (\ref{eq.general.hybrid.relax}). When there is no ambiguity, we omit control variables' dependency on $\xb$ or $t$ for the simplicity of the presentation.
	
	\begin{enumerate}[label=(\roman*)]
		\item 
		First, we decompose the control variables $W(\xb,t)$ and $d(\xb,t)$ as in the following:
		\begin{align}\label{eq:decomp}
			W(\xb,t)=A(\xb,t)+\tilde{A}(\xb,t), \ d(\xb,t)=b(\xb,t)+\tilde{b}(\xb,t).
		\end{align}
		These variables will be further split next. Above, $A, b$ contain the control variables in the left branch of the multigrid V-cycle, and $\tilde{A}, \tilde b$ contain the control variables in the right branch of the multigrid V-cycle, see later steps for details. We also decompose the operator as follows: 
		\begin{align}
			-\varepsilon \ln  \frac u {1-u} - \eta K*( 1-2u)=S(u)+\tilde{S}(u).
			\label{eq.S.sum}
		\end{align}
		
		\item Second, the control variables are further decomposed as: 
		\begin{align}
			&A=\sum_{j=1}^J A^j,\quad b=\sum_{j=1}^{J-1} b^j, \quad S=\sum_{j=1}^J S^j,\\
			&\tilde{A}=\sum_{j=1}^{J-1} \tilde{A}^j+A^*,	\quad \tilde{b}=\sum_{j=1}^{J-1} \tilde{b}^j+ b^*,\quad \tilde{S}=\sum_{j=1}^{J-1} \tilde{S}^j+S^*.
			\label{eq.decompose.2}
		\end{align}
		Above, $A^j,b^j,S^j, \widetilde{A}^j, \widetilde{b}^j, \widetilde{S}^j$ contain control variables at grid level $j$, $A^*$, $b^*$, and $S^*$ contain control variables that are applied to the output of the V-cycle at the finest mesh, i.e. $A^j, \tilde A^j\in \cV^j, A^*\in \cV^1, b^j, \tilde b^j, b^* \in \mathbb{R}$.
		\item At grid level $j$ in the left branch, we further decompose 
		\begin{align}
			A^{j}=\sum_{k=1}^{c_j} A^{j}_k, \quad b^{j}=\sum_{k=1}^{c_j} b^{j}_k, \quad  S^{j}=\sum_{k=1}^{c_j} S^{j}_k.
		\end{align}
		Above, $c_j$'s are positive integers which are often called the channel number at grid level $j$ and they are fixed parameters for the network. The reason to do these decompositions is to increase the number of unknowns in the control variables which enables us to handle large datasets. We compute $c_j$ intermediate outputs with the $A_k^j$'s. Variables $A_k^j \in \cV^j , b_k^j \in \mathbb{R}$ contain control variables producing the $k$-th intermediate output. This is accomplished via a hybrid splitting scheme with $c_j$ parallel splittings, see Appendix \ref{sec.hybrid} for details. We do the same decomposition for the right branch, i.e.
		\begin{align}
			\tilde A^{j}=\sum_{k=1}^{c_j} \tilde A^{j}_k, \quad \tilde b^{j}=\sum_{k=1}^{c_j} \tilde b^{j}_k, \quad  \tilde S^{j}=\sum_{k=1}^{c_j} \tilde S^{j}_k.
		\end{align}

		\item 	At grid level $j$ for the $k$-th intermediate output, we again further decompose
		\begin{align}
			A^{j}_k=\sum_{s=1}^{c_{j-1}} A^{j}_{k,s}.
		\end{align} 
		The purpose is also to increase the number of unknowns for the control variables.  The $k$-th intermediate output at grid level $j$ is computed using all intermediate outputs from the previous grid level. Variable $A_{k,s}^j$ is used to convolve with the $s$-th intermediate output from grid level $j-1$. We do the same decomposition for the right branch.
		\item The $A^*$ variable is also further decomposed as
		\begin{align}
			A^*=\sum_{s=1}^{c_1} A_s^*,
		\end{align}
		where $A_s^*$ is used to convolve with the $s$-th output from level $1$ of the right branch.
	\end{enumerate}	
	
	After these decompositions, we see that the control variables are decomposed as: 
	\begin{align}
		& A(\xb, t) = \sum_{j=1}^J \sum_{k=1}^{c_j} \sum_{s=1}^{c_{j-1}}  A_{k,s}^j(\xb, t), \quad &&
		\tilde A(\xb, t)  = \sum_{j=1}^{J-1} \sum_{k=1}^{c_j} \sum_{s=1}^{c_{j-1}}  \tilde A _{k,s}^j(\xb, t) +\sum_{s=1}^{c_1} A^*_s(\xb, t), \\ 
		& b(\xb, t) =\sum_{j=1}^J \sum_{k=1}^{c_j}   b_{k}^j(\xb, t), 
		&& \tilde b (\xb, t)= \sum_{j=1}^{J-1} \sum_{k=1}^{c_j} \tilde b_{k}^j (\xb, t)+ \tilde b^*(\xb, t), 
	\end{align}
	and the operators $S(u), \tilde S(u)$ are decomposed as: 
	\begin{align}
		S(u) = \sum_{j=1}^J \sum_{k=1}^{c_j} S_{k}^j (u), \ \tilde{S}(u)=\sum_{j=1}^{J-1} \sum_{k=1}^{c_j} \tilde{S}_{k}^j (u) +S^*(u).
	\end{align}
	The Potts gradient flow is transferred into: 
	\begin{align}
		\begin{cases}
			\frac{\partial u}{\partial t}=A*u+\tilde{A}*u + b +\tilde{b} + S(u)+\tilde{S}(u), \ (\xb,t)\in \Omega\times [0,T],\\
			u(\xb,0)=H(f), \ \xb\in \Omega.
		\end{cases}
		\label{eq.control.full}
	\end{align}
	
	With the decomposition discussed above, (\ref{eq.control.full}) has a similar form as (\ref{eq.general.hybrid.relax}) with $M_j=1, c_{j,1}=c_{2J-j}=c_j$ for $1\leq j \leq J$.  We solve problem (\ref{eq.control.full}) by the hybrid splitting method proposed in Appendix \ref{sec.hybrid}.  Let us divide the time interval $[0,T]$ into $N$ subintervals with time step size $\Delta t = T/N$. We denote our numerical solution at time $t^n=n\Delta t$ by $U^n$. In our scheme, we use $u$ and $v$ to denote intermediate variables. Their superscript $j$ indicates that they are discretized on $\cT^j$, i.e. $v^j_k,u^j_k, \bar{u}^j\in \cV^j$. The resulting scheme by applying Algorithm \ref{alg.hybrid.general} to (\ref{eq.control.full}) for updating $U^{n}$ to $U^{n+1}$ is summarized in Algorithm \ref{alg:alg3}, where the dependency of control variables on $\xb$ is omitted. 
	
	The architecture of Algorithm \ref{alg:alg3} is illustrated in Figure \ref{fig.alg.basic}. The explanations of all indices for operators and variables of the left branch are summarized in Table \ref{tab.V.full}. In Algorithm \ref{alg:alg3}, for each grid level, a relaxation is used to pass information from the left branch to the right branch, as indicated by the green arrows.
	
	\begin{figure}[t!]
		\centering
		\includegraphics[width=0.7\textwidth]{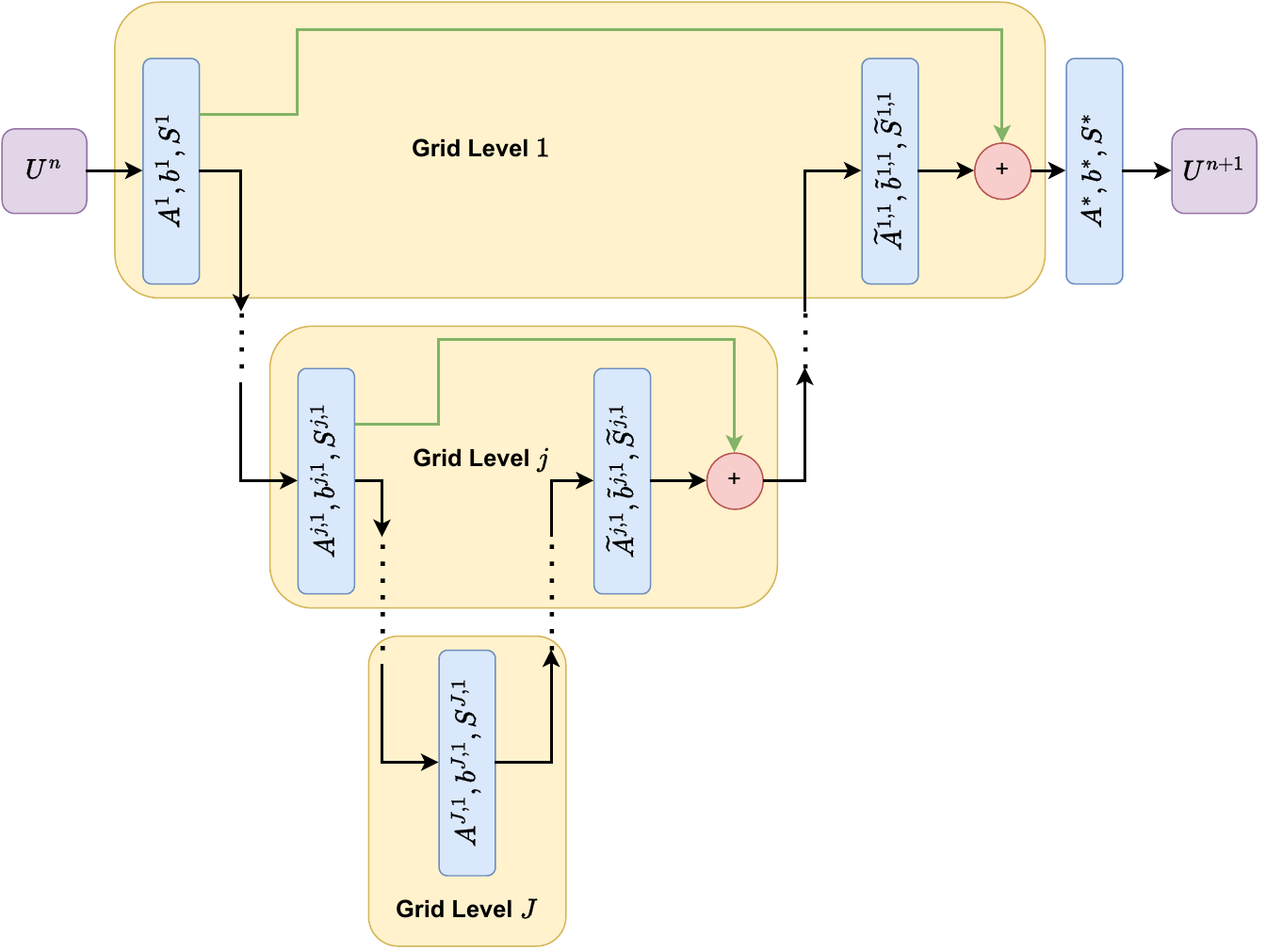}
		\caption{Illustration of Algorithm \ref{alg:alg3}.}
		\label{fig.alg.basic}
	\end{figure}
	
	Let us denote $\theta_2=\{\theta_2^n\}_{n=1}^N$ with 
	\begin{align}
		&\theta_2^n =\left(\{ A^j_{k,s}(\xb, t^n)\}_{j,k,s},\{ \tilde A^j_{k,s}(\xb, t^n)\}_{j,k,s}, \{A_s^*(\xb, t^n)\}_s, \{b_k^j(\xb, t^n)\}_{j,k}, \{ \widetilde{b}_k^j(\xb, t^n)\}_{j,k}, \widetilde{b}^*(\xb, t^n)\right).  
		\label{eq.theta2}
	\end{align}
	We also denote $\cN_2$ as the mapping:
	$$
	\cN_2:f\rightarrow H(f)\rightarrow U^1 \rightarrow \cdots \rightarrow U^N, 
	$$
	which maps $f$ to $U^N$ by applying Algorithm \ref{alg:alg3} $N$ times. Parameters $\theta_2$ are learned by solving
	\begin{equation}
		\min_{\theta_2} \sum_{i=1}^I \mathcal{L}(\mathcal{N}_2(f_i,\theta_2) , v_i ).
		\label{eq.min.basic}
	\end{equation}
	In (\ref{eq.min.basic}), $\theta_2$ is a space decomposition representation for a discretization of $\theta_1$. The operation procedure $\cN_2$ is a numerical scheme solving (\ref{eq:potts.control.k=2}). We can see that problem (\ref{eq.min.basic}) is a discretization of the continuous Potts model (\ref{eq:control}) with some proper decomposition of the control variables. 
	
	\begin{algorithm}[th!]
		\caption{A basic V-Cycle multigrid control algorithm}\label{alg:alg3}
		\begin{algorithmic}
			\STATE \textbf{Data:} The solution $U^n$ at time step $t^n$.
			\STATE \textbf{Result:} The computed solution $U^{n+1}$ at time step $t^{n+1}$.
			\STATE {\bf Set} $c_{0} = 1, v^{0} = U^n, v_1^0 = U^n$. \\
			\FOR {$j = 1,  \cdots, J$}
			\FOR{$k=1,2,\cdots c_j$}
			\STATE Compute $v_k^{j}$ by solving
			\begin{align}
				\frac {v_k^{j} - v^{j-1}} {2^{j-1}c_{j} \Delta t } =   
				\sum_{s=1}^{c_{j-1}} 
				A_{k,s}^{j}(t^n)  * v_s^{j-1}   + b_k^j(t^n)   + S_{k}^{j}(v_k^{j} ).
				\label{eq.basic.v}
			\end{align}
			
			\ENDFOR
			
			Compute $v^{j}$ as
			\[
			v^{j} = \frac 1 {c_{j}} \sum_{k=1}^{c_{j}} v_k ^ {j} . 
			\]
			\ENDFOR
   
			{\bf Set} $\bar{u}^J= v^J$ and $ \bar{u}_k^J = v_k^J, k =1, 2, \cdots c_J$.\\
			\FOR {$j = J-1,  \cdots, 1$}
			\FOR {$k=1,2,\cdots c_j$}
			\STATE Compute $u_k^{j}$ by solving 
			\begin{align}
				&\frac {u_k^{j} - \bar{u}^{j+1}} {2^j c_{j} \Delta t } =   
				\sum_{s=1}^{c_{j+1}} 
				\tilde A_{k,s}^{j}(t^n)  * \bar{u}_s^{j+1}   +\tilde{ b}_k^j(t^n)   + \tilde S_{k}^j(u_k^{j} ).
				\label{eq.basic.u}
			\end{align}

			\ENDFOR

			Compute $\bar{u}^j_k$ and $\bar{u}^j$ as
			\begin{align}
				\bar{u}^j_k =\frac{1}{2} u^j_k + \frac{1}{2} v_k^j
			\end{align}
			for $k=1,...,c_j$, and 
			
			\[
			\bar{u}^{j} = \frac{1}{c_j}\sum_{j=1}^{c_j} \bar{u}^{j}_k.
			\]
			\ENDFOR
			
			Compute $U^{n+1}$ by solving
			\[
			\frac {U^{n+1} - \bar{u}^{1}} {\Delta t } =   
			\sum_{s=1}^{c_{1}} 
			A_{s}^{*}(t^n)  * u_s^{1}   + b^*(t^n)   +  S^*(U^{n+1} ).
			\]
		\end{algorithmic}
	\end{algorithm}

	\subsection{Relationship to neural networks:}  
	For readers familiar with convolutional neural networks, it is immediately evident that the first two terms on the right-hand side of (\ref{eq.basic.v}) and (\ref{eq.basic.u}) are computing the same thing as is done by the `conv2d' function in PyTorch. 
	The architecture of Algorithm \ref{alg:alg3} is equivalent to a simple encoder-decoder-based neural network with $2J$ layers:
	\begin{enumerate}[label=(\roman*)]
		\item The left and right branches of the V-cycle correspond to the encoder and decoder in neural networks, respectively. In the left branch, computations from fine grids to coarse grids are conducted, which is an encoding process. In the right branch, computations from coarse grids to fine grids are conducted, which is a decoding process. In general, Algorithm \ref{alg:alg3} is equivalent to a neural network with $2J$ layers: $2J-1$ layers for the encoder and decoder, and 1 final layer.
		\item Computations at grid level $j$ in the left branch corresponds to the $j$-th layer of the corresponding network. At this grid level, the number of parallel splittings $c_j$ corresponds to the width of the $j$-th layer of the network. The index $k$ corresponds to the $k$-th channel of this layer.
		\item The relaxation in the right branch corresponds to the skip pathways between encoders and decoders in the network. 
	\end{enumerate}
	
	\begin{table}[ht!]
		\centering
		\begin{tabular}{c|c|c|c|c}
			\hline
			\makecell{For $A^{j}_{k,s},b^{j}_{k},S^{j}_k$,\\ $A^{j,l}_{k,s},b^{j,l}_{k},S^{j,l}_k$} & $j$ & $l$ & $k$ & $s$\\
			\hline
			\makecell{ Index meaning: \\index of} & grid levels & \makecell{sequential \\splittings} & \makecell{parallel \\splittings} & \makecell{output from \\ the previous substep}\\
			\hline
			\makecell{For $u^{j}_k,v^{j}_k$,\\ $u^{j,l}_k,v^{j,l}_k$} &$j$ & $l$ & $k$ &-\\
			\hline
			\makecell{ Index meaning: \\index of} & grid levels & \makecell{sequential \\splittings} & \makecell{parallel \\splittings} & -\\
			\hline
		\end{tabular}
		\caption{Explanation of indices for kernels and variables in the left branch of Algorithm \ref{alg:alg3} and \ref{alg.V.full}.}
		\label{tab.V.full}
	\end{table}
	
	\subsection{A general decomposition for $\theta_1$}
	\label{sec.FullAlg}
	
	In Algorithm \ref{alg:alg3}, excluding the final substep, we only have two sequential steps at each grid level: one in the left branch and one in the right branch. As a consequence, the equivalent network has two layers for every grid level: one layer in the encoder and one layer in the decoder. In practice, many popular networks use several layers for each grid level, for example, UNet has four layers for each grid level. We then modify the splitting strategy of Algorithm \ref{alg:alg3} to generalize it so that it is equivalent to more general networks. 
	We decompose the control variables $\theta_1$ as follows:
	\begin{enumerate}[label=(\roman*)]
		\item[(i)--(ii)] The first two decompositions are the same as the first two decompositions in Section \ref{sec.BasicAlg}, after which we get the decomposition of $A,b,S$ and $\tilde{A},\tilde{b},\tilde{S}$.
		
		\setcounter{enumi}{2}
		\item At grid level $j$, let $L_j$ be a positive integer representing the number of substeps we want to perform at grid level $j$ in the left (and right) branch. We decompose 
		\begin{align}
			A^{j}=\sum_{l=1}^{L_j} A^{j,l}, \quad b^j=\sum_{l=1}^{L_j} b^{j,l},\quad  S^j=\sum_{l=1}^{L_j} S^{j,l}.
		\end{align}
		We use a sequential splitting to divide the operators into $L_j$ sets, where $A^{j,l},b^{j,l}$ and $S^{j,l}$ are the sets of operators used at the $l$-th sequential substep. We do the same decomposition for the right branch.
		\item At grid level $j$ and the $l$-th sequential substep, we decompose 
		\begin{align}
			A^{j,l}=\sum_{k=1}^{c_j} A^{j,l}_k, \quad b^{j,l}=\sum_{k=1}^{c_j} b^{j,l}_k, \quad  b^{j,l}=\sum_{k=1}^{c_j} S^{j,l}_k.
		\end{align}
		We use a hybrid splitting with $c_j$ parallel splittings to treat all operators, where operators $A^{j,l}_k,b^{j,l}_k$ and $S^{j,l}_k$ are used in the $k$-th parallel splitting. We do the same decomposition for the right branch.
		\item At grid level $j$, the $l$-th sequential step and the $k$-th parallel splitting, we take the $c_j$ outputs from the $(l-1)$-th sequential step as inputs, and use kernels from $A^{j,l}_k$ to convolve with them. Therefore, we decompose $A^{j,l}_k$ into $c_j$ kernels:
		\begin{align}
			A^{j,l}_k=\sum_{s=1}^{c_j} A^{j,l}_{k,s} \quad \mbox{ with } \quad 
			c_{j,l}=\begin{cases}
				c_{j-1} & \mbox{ if } l=1,\\
				c_j & \mbox{ if } l>1.
			\end{cases}
			\label{eq.c}
		\end{align} 
		We do the same decomposition for the right branch:
		\begin{align}
			\widetilde{A}^{j,l}_k=\sum_{s=1}^{\widetilde{c}_j} \widetilde{A}^{j,l}_{k,s} \quad \mbox{ with } \quad 
			\tilde{c}_{j,l}=\begin{cases}
				c_{j+1} & \mbox{ if } l=1,\\
				c_j & \mbox{ if } l>1.
			\end{cases}
			\label{eq.ctilde}
		\end{align}

		\item We decompose $A^*$ in the same way as the decomposition step (v) in Section \ref{sec.BasicAlg}.
	\end{enumerate}
	
	After the decompositions, the control variables and operations are decomposed as:
	\begin{align}
		&A(\xb,t)=\sum_{j=1}^J \sum_{l=1}^{L_j}\sum_{k=1}^{c_j} \sum_{s=1}^{c_{j,l}}  A_{k,s}^{j,l}(\xb,t), \quad &&\tilde{A}(\xb,t)=\sum_{j=1}^J \sum_{l=1}^{L_j}\sum_{k=1}^{c_j} \sum_{s=1}^{\tilde{c}_{j,l}}  A_{k,s}^{j,l}(\xb,t)+ \sum_{s=1}^{c_1} A^*_s(\xb,t),  \label{eq.full.A}\\
		&b(\xb,t)=\sum_{j=1}^J \sum_{l=1}^{L_j}\sum_{k=1}^{c_j}   b_{k}^{j,l}(\xb,t), \quad &&\tilde{b}(\xb,t)=\sum_{j=1}^J \sum_{l=1}^{L_j}\sum_{k=1}^{c_j}  b_{k,s}^{j,l}(\xb,t)+  b^*(\xb,t),
		\label{eq.full.b}\\
		&S(u)=\sum_{j=1}^J  \sum_{l=1}^{L_j} \sum_{k=1}^{c_j} S_{k}^j (u), \quad && \tilde{S}(u)=\sum_{j=1}^{J-1}  \sum_{l=1}^{L_j} \sum_{k=1}^{c_j} \tilde{S}_{k}^j (u) +S^*(u).
		\label{eq.full.S}
	\end{align}
	
	The decomposed variables $A(\xb,t),\widetilde{A}(\xb,t), b(\xb,t),\widetilde{b}(\xb,t)$ and the operations $S$ and $\widetilde{S}$ have similar forms as those variables and operations in Appendix \ref{sec.hybrid}. Thus we can use the hybrid splitting method Algorithm \ref{alg.hybrid.general} to solve the control problem.  The resulting scheme for updating $U^{n}$ to $U^{n+1}$ is summarized in Algorithm \ref{alg.V.full}, where the dependency of the control variables on $\xb$ is omitted. Note that Algorithm \ref{alg.V.full} is a special case of Algorithm \ref{alg.hybrid.general} by setting $M_j=M_{2J-j}=L_j$ for $1\leq j\leq J$, $d_{j,m}=c_{j,m-1}, c_{j,m}=c_{2J-j,m}=c_j $ for $1\leq j\leq J, 1\leq m\leq L_j$.  The architecture of Algorithm \ref{alg.V.full} is illustrated in Figure \ref{fig.alg}. The explanations of all indices for operators and variables of the left branch are summarized in Table \ref{tab.V.full}. 
	
	\begin{algorithm}
		\caption{A general V-Cycle  multigrid control  algorithm}\label{alg.V.full}
		\begin{algorithmic}
			\STATE {\bf Data:} The solution $U^n$ at time $t^n$.
			\STATE {\bf Result:} The computed solution $U^{n+1}$ at time step $t^{n+1}$.
			
			{\bf Set} $c_{0} = 1, L_0=1, v^{0} = U^n, v_1^{0,1} = U^n$. \\
			\FOR {$j = 1,  \cdots, J$}
			\STATE Set $v^{j,0}=v^{j-1,L_{j-1}}$ and $v_k^{j,0}=v_k^{j-1,L_{j-1}}, \ k=1,...,c_{j-1}$, \\
			\FOR {$l=1,...,L_{j}$}
			\FOR {$k=1,...,c_{j}$}
			\STATE Compute $v_k^{j,l}$ by solving
			\begin{align}
				\frac {v_k^{j,l} - v^{j,l-1}} {2^{j-1}c_j \Delta t } =   
				\sum_{s=1}^{c_{j,l}} 
				A_{k,s}^{j,l}(t^n)  * v_s^{j,l-1}   + b_{k}^{j,l}(t^n)   + S_{k}^{j,l}(v_k^{j,l} ),
				\label{eq.full.v}
			\end{align}
			where $c_{j,l}$ is defined in (\ref{eq.c}) and (\ref{eq.ctilde}).
			\ENDFOR 
			
			Compute $v^{j+1,l}$ as
			\begin{equation*}
				v^{j,l} = \frac 1 {c_{j}} \sum_{k=1}^{c_{j}} v_k ^ {j,l} . 
			\end{equation*}
			\ENDFOR 
			
			\ENDFOR
			
			{\bf Set} $\bar{u}^{J,L_J}= v^{J,L_J}$ and $ u_k^{J,L_J} = v_k^{J,L_J}$ for $k =1, 2, \cdots c_J$.\\
			\FOR {$j = J-1,  \cdots, 1$}
			\STATE Set $u^{j,0}=\bar{u}^{j+1,L_{j+1}}$ and $u_k^{j,0}=u_k^{j+1,L_{j+1}}$, $k=1,...,c_{j+1}$, \\
			\FOR {$l=1,...,L_j$}
			\FOR {$k=1,2,\cdots c_j$}
			\STATE Compute $u_k^{j,l}$ by solving 
			\begin{align}
				&\frac {u_k^{j,l} - u^{j,l-1}} {2^{j} c_{j} \Delta t } =   
				\sum_{s=1}^{\tilde{c}_{j,l}} 
				\tilde A_{k,s}^{j}(t^n)  * u_s^{j,l-1}   + \tilde b_{k}^{j}(t^n)   + \tilde S_{k}^j(u_k^{j} ),
				\label{eq.full.u}
			\end{align}					
			where $\tilde{c}_{j,l}$ is defined in (\ref{eq.c}) and (\ref{eq.ctilde}).
			\ENDFOR
			
			Compute $u^{j,l}$ as
			\[
			u^{j,l} = \frac 1 {c_{j}} \sum_{k=1}^{c_j} u_k ^ {j,l}.
			\]
			
			\ENDFOR
			
			Compute $\bar{u}^{j,L_j}_k, \bar{u}^{j,L_j}$ as 
			\begin{align} 
				\bar{u}^{j,L_j}_k=\frac{1}{2} u^{j,L_j}_k+ \frac{1}{2} v^{j,L_j}_k, \quad
				\bar{u}^{j,L_j}=\frac{1}{c_j} \sum_{k=1}^{c_j} \bar{u}_k^{j,L_j}
				\label{eq.full.relax}
			\end{align}

			\ENDFOR
			
			Compute $U^{n+1}$ by solving
			\begin{align}
				\frac {U^{n+1} - \bar{u}^{1,L_1}} {\Delta t } =   
				\sum_{s=1}^{c_{1}} 
				A_{s}^{*}(t^n)  * u_s^{1,L_1}   + b^*(t^n)   +  S^*(U^{n+1} ).
				\label{eq.full.final}
			\end{align}
		\end{algorithmic}
	\end{algorithm}
	
	\begin{figure}
		\centering
		\includegraphics[width=\textwidth]{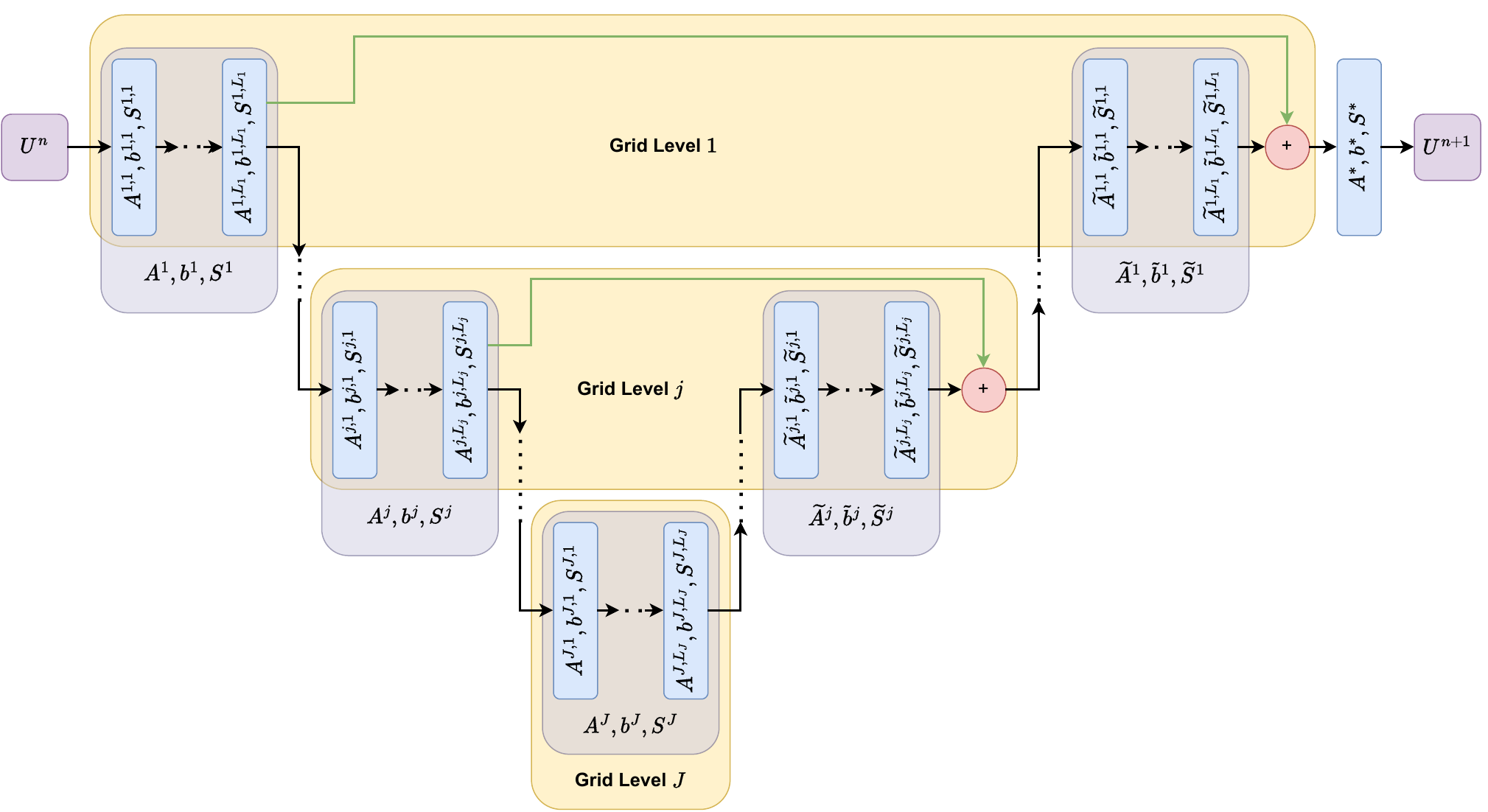}
		\caption{An illustration of Algorithm \ref{alg.V.full}.}
		\label{fig.alg}
	\end{figure}
	Similar to (\ref{eq.theta2}), denote the collection of all parameters discussed in this subsection by $\theta_3$, and denote
	$$
	\cN_3: f\rightarrow H(f)\rightarrow U^1 \rightarrow \cdots \rightarrow U^N
	$$
	as the mapping that maps $f$ to $U^N$ by applying Algorithm \ref{alg.V.full} for $N$ times. Parameters $\theta_3$ are learned by solving
	\begin{equation}
		\min_{\theta_3} \sum_{i=1}^I \mathcal{L}(\mathcal{N}_3(f_i,\theta_3) , v_i ).
		\label{eq.min.general}
	\end{equation}

	\paragraph{Connections to neural networks:} Algorithm \ref{alg.V.full} gives a mathematical explanation of many popular encoder-decoder-based network architectures. In addition to the connections between Algorithm \ref{alg:alg3} with neural networks, Algorithm \ref{alg.V.full} has one more connection: the number of sequential splitting for each grid level in Algorithm \ref{alg.V.full} corresponds to the number of layers for each data resolution in an encoder-decoder-based network. Compared to Algorithm \ref{alg:alg3}, this additional connection makes Algorithm \ref{alg.V.full}  a mathematical explanation of a more general class of encoder-decoder-based networks.
	
	Note that Algorithm \ref{alg:alg3} is a special case of Algorithm \ref{alg.V.full} by setting $L_j=1$ for $j=1,...,J$. In the rest of this paper, we focus on Algorithm \ref{alg.V.full}.
	
	\section{Numerical discretization}
	\label{sec.dis}
	
	Algorithm \ref{alg.V.full} is semi-constructive as we still need to discretize $u,v$ in space and solve the semi-implicit problems (\ref{eq.full.v}), (\ref{eq.full.u}) and (\ref{eq.full.final}). We discuss spatial discretization in this Section.

	\subsection{On the choices of $S,\tilde{S}$}
	\label{sec.select}
	
	In Algorithm \ref{alg.V.full}, one needs to solve (\ref{eq.full.v}), (\ref{eq.full.u}) and (\ref{eq.full.final}), which includes components of $S,\tilde{S}$. We discuss the choices of $S,\tilde{S}$ and present how to solve (\ref{eq.full.v}), (\ref{eq.full.u}) and (\ref{eq.full.final}) in the next subsection.
	
	According to (\ref{eq.S.sum}), $S+\tilde{S}$ consists of two terms: (i) The first term is $-\varepsilon\ln \frac{u}{1-u}$ which is resulted from the regularized softmax. This term enforces $u$ to be between 0 and 1. (ii) The second term $-\eta G*(1-2u)$ results from the length penalty, which promotes smooth boundaries in the segmentation $u$. As all substeps except the last substep in Algorithm \ref{alg.V.full} are used to extract and reconstruct important features of the input image, we set 
	\begin{align}
		S^{j,l}_k= -\frac{(2^{j-1}c_j)^{-1}}{\kappa} \varepsilon\ln \frac{u}{1-u}, \quad \tilde{S}^{j,l}_k= -\frac{(2^{j}c_j)^{-1}}{\kappa}\varepsilon\ln \frac{u}{1-u}
		\label{eq.S.choice}
	\end{align}
	with the normalization term
	\begin{align}
		\kappa=\sum_{j=1}^J  \sum_{l=1}^{L_j} \sum_{k=1}^{c_j} (2^{j-1}c_j)^{-1}+ \sum_{j=1}^{J-1}  \sum_{l=1}^{L_j} \sum_{k=1}^{c_j}(2^{j}c_j)^{-1} +1.
	\end{align}
	
	At the last substep, the segmentation is reconstructed, for which we will use the length penalty term to regularize the segmentation. We set
	\begin{align} \label{eq:sstar}
		S^*= -\frac{1}{\kappa} \varepsilon\ln \frac{u}{1-u}-\eta G_{\sigma}*(1-2u).
	\end{align}
	In (\ref{eq.S.choice}), the factor on the right-hand side is chosen based on two considerations: (i) the functions $S$ and $\tilde{S}$ should satisfy (\ref{eq.S.sum}), and (ii) the effect of  $-\varepsilon\ln \frac{u}{1-u}$ are expected to be the same for every intermediate variable $v^{j,l}_k$ and $u^{j,l}_{k}$, i.e., it is not affected by the factors on the left-hand side of (\ref{eq.full.v}) and (\ref{eq.full.u}) which are resulted from parallel splitting and the relaxation step.

	\subsection{On the solution to (\ref{eq.full.v}), (\ref{eq.full.u}) and (\ref{eq.full.final})}
	\label{sec.solver}
	In (\ref{eq.full.v}) and (\ref{eq.full.u}) when $l=1$, the problems involve computations between functions discretized on different grid levels. We will first convert all functions to $\cT^j$ and then solve the problems.
	
	For (\ref{eq.full.v}) when $l=1$, $v^{j,0}=v^{j-1,L_j}$ and $v^{j,0}_k=v^{j-1,L_j}_k$ are defined on grid level $j-1$. We use the average pooling (\ref{eq.downsampling.ave}) or the max pooling (\ref{eq.downsampling.max}) to downsample these functions to grid level $j$ and then assign them to $v^{j,0}$ and $v^{j,0}_k$. In our experiment, (\ref{eq.downsampling.max}) is used.
	
	For (\ref{eq.full.u}) when $l=1$, $u^{j,0}=\bar{u}^{j+1,L_j}$ and $u^{j,0}_k=\bar{u}^{j+1,L_j}_k$ are defined on grid level $j+1$. We use the piecewise constant upsampling (\ref{eq.upsampling}) to upsample these functions to grid level $j$ and then assign them to $u^{j,0}$ and $u^{j,0}_k$.
	
	Observe that (\ref{eq.full.v}), (\ref{eq.full.u}) and (\ref{eq.full.final}) are in the form of
	\begin{align}
		\frac{u-u^*}{\gamma \Delta t}=\sum_{s=1}^c \hat{A}_s*u^*_s +\hat{b} +\hat{S}(u),
		\label{eq.basic.form}
	\end{align}
	where $\gamma$ is some constant, $c$ is some integer, $u^*=\frac{1}{c} \sum_{s=1}^c u^*_s$ for some functions $u^*_s$'s, $\hat{A}_s$'s are some convolution kernels, $\hat{b}$ is some bias function and $\hat{S}$ is some nonlinear function. The solution to (\ref{eq.basic.form}) can be computed using two substeps:
	\begin{align}
		\begin{cases}
			\bar{u}=u^*+\gamma\Delta t \left(\sum_{s=1}^c \hat{A}_s*u^*_s +\hat{b}\right),\\
			u=(I_{\rm id}-\gamma\Delta t \hat{S})^{-1} (\bar{u}),
		\end{cases}
		\label{eq.basic.form.sub}
	\end{align}
	where $I_{\rm id}$ denotes the identity operator, $(I_{\rm id}-\gamma\Delta t \hat{S})^{-1}$ is the resolvent operator of $(I_{\rm id}-\gamma\Delta t \hat{S})$.  When $\hat{S}$ is a nonlinear operator, (\ref{eq.basic.form.sub}) is the building block of neural networks: the first substep is a linear layer, and the second substep corresponds to some activation function. Therefore, Algorithm \ref{alg.V.full} is a convolutional network with $(I_{\rm id}-\Delta t S)^{-1}$ being the activation function.

	In (\ref{eq.basic.form.sub}), there is no difficulty in solving for $\bar{u}$ as it is an explicit step. For $u$ in (\ref{eq.basic.form.sub}), when $\hat{S}=S^{j,l}_k, \tilde{S}^{j,l}_k$ or $S^*$  as in (\ref{eq.S.choice})  and (\ref{eq:sstar}) , we need to solve a problem of the following form:
	\begin{align}
		\frac{u-\bar{u}}{C_1\Delta t} + C_2G_\sigma*(1-2u)=-\varepsilon\ln \frac{u}{1-u}
		\label{eq.activate}
	\end{align}
	where 
	\begin{align*}
		&C_1=1/\kappa,\ C_2= 0 \quad \mbox{ for } \hat{S}=S^{j,l}_k,\tilde{S}^{j,l}_k,\\
		&C_1=1/\kappa,\ C_2=\eta \quad \mbox{ for } \hat{S}=S^*.
	\end{align*}
	We use a fixed point method to solve (\ref{eq.activate}).
	
	First initialize $p^0=\bar{u}$. From $p^k$, we update $p^{k+1}$ by solving
	\begin{align}
		\frac{p^k-\bar{u}}{C_1\Delta t} + C_2\eta G_\sigma*(1-2p^k)=-\varepsilon\ln \frac{p^{k+1}}{1-p^{k+1}},
	\end{align}
	for which we have the closed-form solution
	\begin{align}
		p^{k+1}=\Sig\left(-\frac{1}{\varepsilon}\left(\frac{p^k-\bar{u}}{C_1\Delta t} + C_2 \eta G*(1-2p^k)\right)\right),
		\label{eq.p.update}
	\end{align}
	where $\Sig(x)=\frac{1}{1+e^{-x}}$ is the sigmoid function. By repeating (\ref{eq.p.update}) so that $p^{k+1}$ converges to some function $p^*$, we set $u=p^*$.
	
	As $U^n$ evolves during iterations, it is not necessary to repeat (\ref{eq.p.update}) until convergence for every intermediate variable. Instead, one may only use a few steps of (\ref{eq.p.update}). In particular, if only two steps of (\ref{eq.p.update}) are used, the resulting formula is a relaxed and regularized version of the sigmoid activation function.
	
	\begin{remark}
		In (\ref{eq.activate}), if $C_2=0$, i.e., there is no length penalty term, at least two steps of (\ref{eq.p.update}) should be used. In fact, when $C_2=0$ and we initialize $p^0=u^*$, the first step is  trivial as one always has $p^1=0.5$.
	\end{remark}
	
	\section{Relations of Algorithm \ref{alg.V.full} to existing networks}
	\label{sec.relation}
	
	Encoder-decoder-based neural networks have been widely used in image segmentation, such as UNet \cite{ronneberger2015u}, UNet++  \cite{zhou2019unetpp} and SegNet  \cite{badrinarayanan2017segnet}. In Sections \ref{sec.BasicAlg} and \ref{sec.FullAlg}, we showed that Algorithm \ref{alg.V.full} has an encoder-decoder architecture and its analogy to neural networks. In this section, we show that with minor modifications, Algorithm \ref{alg.V.full} can recover the architectures of most encoder-decoder-based neural networks, just with different activation functions. 
	
	We order the image resolution (corresponding to the grid levels in Algorithm \ref{alg.V.full}) from the finest to coarsest by 1 to $J$, where $J$ is the total levels of image resolution. With this ordering, level 1 corresponds to the finest resolution and level $J$ corresponds to the coarsest resolution.
	
	\subsection{Relations to general convolutional neural networks}
	\label{sec.buildingblock}
	For a convolutional neural network, the building block is
	\begin{align}
		\begin{cases}
			\bar{v}_k=\sum_{s=1}^c W_{k,s}*v^*_s+b_k,\\
			v_k=\chi(\bar{v}_k),
		\end{cases}
		\label{eq.network.block}
	\end{align}
	where $v_s^*$'s are the outputs from the previous layer, $v_k$ is the output of the $k$-th channels of the current layer, $W_{k,s}$'s are convolutional kernels and $\chi$ is an activation function. For Algorithm \ref{alg.V.full}, the building block is (\ref{eq.basic.form}), which is solved by (\ref{eq.basic.form.sub}).  In fact, (\ref{eq.network.block}) and (\ref{eq.basic.form.sub}) have the same form. With the proper choice of $\sigma$, they are equivalent to each other.
	
	In the following, we show that if we choose $\chi=(I_{\rm id}-\gamma\Delta t \hat{S})^{-1}$, $v^*_s=u^*_s$ and the proper convolution kernels $W_{k,s}$'s, we have $v_k=u$. 
	
	In the first equation of (\ref{eq.basic.form.sub}), substitute the expression of $u^*$, and we have
	\begin{align}
		\bar{u}=&\frac{1}{c}\sum_{s=1}^c u^*_s + \gamma\Delta t\left(\sum_{s=1}^c \hat{A}_s*u_s^* +\hat{b}\right)=\sum_{s=1}^c \left(\frac{1}{c} \mathds{1}+\gamma\Delta t \hat{A}_s\right)*u_s^* + \gamma\Delta t\hat{b},
	\end{align}
	where $\mathds{1}$ denotes the identity kernel satisfying $\mathds{1}*g=g$ for any function $g$. Set 
	\begin{align}
		W_{k,s}=\frac{1}{c} \mathds{1}+\gamma\Delta t \hat{A}_s, \quad b_k=\hat{b}.
	\end{align}
	We have $\bar{v}_k=\bar{u}$. By choosing $\chi=(I_{\rm id}-\gamma\Delta t \hat{S})^{-1}$, we have $v_k=u$. Essentially, Algorithm \ref{alg.V.full} and convolutional neural networks have the same building block. Most encoder-decoder-based neural networks are instances of Algorithm \ref{alg.V.full}, i.e., an operator-splitting scheme for some control problems.
	
	\subsection{Relations to networks with skip pathways}
	Some popular encoder-decoder networks have skip pathways to improve performance, such as UNet \cite{ronneberger2015u} and UNet++ \cite{zhou2019unetpp}.  In Algorithm \ref{alg.V.full}, skip pathways are realized using the relaxation step (\ref{eq.full.relax}). The implementation of skip pathways in Algorithm \ref{alg.V.full} may be slightly different from that in UNet and other networks. While one can make minor modifications to Algorithm \ref{alg.V.full} to exactly recover the architecture of these networks, here for simplicity, we only address the modification needed to recover UNet's architecture. For other architectures, one can revise Algorithm \ref{alg.V.full} in a similar manner.
	
	In UNet at resolution $j$, the skip pathway copies features from the encoder part to the upsampled features from resolution $j+1$ in the decoder part, which is just before computations at resolution $j$ in the decoder part are conducted. In our Algorithm \ref{alg.V.full}, the intermediate variables $v_k^{j,L_j}$'s at grid level $j$ in the encoder part are added to the output of grid level $j$ in the decoder part, which is after all computations at grid level $j$ in the decoder part are finished.
	
	To recover the architecture of UNet, we replace the initialization of $u^{j,0}$ and $u^{j,0}_k$ by
	\begin{align}
		u^{j,0}_k= \begin{cases}
			\frac{1}{2} u_k^{j+1,L_{j+1}} + \frac{1}{2} v_k^{j,L_j} & \mbox{ for } 1\leq k \leq \min\{c_j,c_{j+1}\},\\
			\frac{1}{2} u_k^{j+1,L_{j+1}} + \frac{1}{2} v^{j,L_j} & \mbox{ for } c_j<k\leq c_{j+1} \mbox{ if } c_{j+1}>c_j,\\
			\frac{1}{2} u^{j+1,L_{j+1}} + \frac{1}{2} v_k^{j,L_j} & \mbox{ for } c_{j+1}<k\leq c_{j} \mbox{ if } c_{j}>c_{j+1},
		\end{cases},\quad
		u^{j,0}=\frac{1}{c_j} \sum_{j=1}^{c_j} u^{j,0}_k,
	\end{align}
	replace (\ref{eq.full.u}) by
	\begin{align}
		&\frac {u_k^{j,l} - u^{j,l-1}} {2^{j-1} c_{j} \tau } =   
		\sum_{s=1}^{\tilde{c}_{j,l}} 
		\tilde A_{k,s}^{j}(t^n)  * u_k^{j,l-1}   + \tilde b_{k}^{j}(t^n)   + \tilde S_{k}^j(u_k^{j} ),
	\end{align}	
	and remove (\ref{eq.full.relax}).
	
	One can show that the modified algorithm is an operator splitting scheme for the control problem (\ref{eq.control.full}) with proper choices of convolution kernels, biases and nonlinear operators, and it has the same architecture as UNet. Putting this property with the one discussed in Section \ref{sec.buildingblock}, we see that the modified algorithm is nothing else but UNet with different activation functions.
	
	Compared to UNet, UNet++ has additional nested, dense skip pathways. For UNet, the skip pathways are realized by using relaxation steps. To realize the additional skip pathways in UNet++, one needs to add more relaxation steps correspondingly. With the proper decomposition of $A,\tilde{A},b,\tilde{b},S,\tilde{S}$ and the introduction of new relaxations, one can easily modify Algorithm \ref{alg.V.full} so that it has the same architecture with UNet++.

	\subsection{Relations to networks without skip pathways}
	SegNet \cite{badrinarayanan2017segnet} is another popular encoder-decoder network for image segmentation. Compared to UNet and UNet++, SegNet does not have any skip pathways and only contains an encoder and a decoder. To recover the architecture of SegNet, we only need to make two changes to Algorithm \ref{alg.V.full}: (i) in the relaxation step (\ref{eq.full.relax}), directly set $\bar{u}_{k}^{j,L_j}=u_k^{j,L_j}$; (ii) remove the factors $2^{j-1}$ in (\ref{eq.full.v}) and $2^j$ in (\ref{eq.full.u}). The revised algorithm has the same architecture as SegNet and one can show that it is an operator splitting scheme for a control problem.
	
	\section{Experiments}
	\label{sec.experiment}
	We compare Algorithm \ref{alg.V.full} with popular networks. We show that on using a single network to segment images with various noise levels, Algorithm \ref{alg.V.full} has better performance.
	
	\subsection{Implementation details}
	\label{sec.experiments.implementation}
	The building block of Algorithm \ref{alg.V.full} is (\ref{eq.basic.form.sub}). We call the first and second steps as convolution step and activation step, respectively. Our implementation details are as follows:
	\begin{itemize}
		\item Choices of $b$: In the control problem for Potts model (\ref{eq:potts.control.k=2}), there are two control variables: $W$ and $d$. Variable $W$ is a kernel function used to compute the convolution with $u$. Variable $d$ is independent to $u$ nor directly applied to $u$. While in the Euler Lagrange equation of Potts model (\ref{eq.Potts.EL}), there is a term $g$ which is independent of $u$ and only depends on the input image $f=\{f_1,f_2,f_3\}$, it is natural to set $d$, and thus $b$ and $\tilde{b}$ (see (\ref{eq.decompose.2})), as a function of $f$. In our implementation, we set 
		\begin{align}
			b^{j,l}_{k}=\begin{cases}
				\sum_{s=1}^3B^{j,l}_{k,s}*f_s & \mbox{ if } l=1,\\
				\beta_{k}^{j,l} & \mbox{ if } l>1,
			\end{cases}
			\quad
			\tilde{b}^{j,l}_{k}=\begin{cases}
				\sum_{s=1}^3\tilde{B}^{j,l}_{k,s}*f_s & \mbox{ if } l=1,\\
				\tilde{\beta}_{k}^{j,l} & \mbox{ if } l>1,
			\end{cases}
			\label{eq.bias}
		\end{align}
		for some learnable kernels $B^{j,l}_{k,s},\tilde{B}^{j,l}_k$ defined on grid level $j$, and $\beta^{j,l}_k,\tilde{\beta}^{j,l}_k$ being some bias constants to be learned.
		
		\item Convolution step in (\ref{eq.basic.form.sub}): The convolution step is an explicit step that can be easily implemented using discretized convolution.
		\item Activation step in (\ref{eq.basic.form.sub}): The activation step is implemented as (\ref{eq.p.update}) with two iterations. For simplicity, we set $C_1=1,C_2=0$ when computing $u_{s}^{k,l}$ (or $u_{s}^{k,l}$). At the last substep (of Algorithm \ref{alg.V.full}), we set $C_1=1,C_2=\lambda$. 
		\item Batch normalization: To improve the stability of the training process, a batch normalization step is added before each activation step except for the last substep. 
		\item Downsampling and upsampling: For downsampling and upsampling, maxpooling (\ref{eq.downsampling.max}) and piecewise constant upsampling (\ref{eq.upsampling}) are used, respectively.
		
		\item Number of time steps: One iteration of Algorithm \ref{alg.V.full} is one-time step. In our implementation, four steps are used.
		
		\item Skip pathways: We replace the weight $1/2$ in (\ref{eq.full.relax}) by parameters that will be learned in training. 
		\item Initial condition: We set $U^0=\sum_{s=1}^3 W_s^0 *f_s$ for some learnable kernel $W_s^0$'s. 
		\item Network hyper parameters: Without specification, we test Algorithm \ref{alg.V.full} with five grid levels ($J=5$) and the following hyper parameters
		\begin{align*}
			\{L_1,L_2,L_3,L_3,L_5\}=\{3,3,3,5,5\}, \quad \{c_1,c_2,c_3,c_4,c_5\}=\{32, 32, 64, 128, 256\}.
		\end{align*}
		In the activation function step (\ref{eq.p.update}), except for the parameters $C_1,C_2$ arising from the skip pathways, this step only depends on $\varepsilon\Delta t$ and $\eta/\varepsilon$. In our experiments, we set $\Delta t=0.5, \varepsilon\Delta t=1, \eta/\varepsilon=40$ and $N=4$ i.e., 4 time steps. We set $\sigma=0.5$ in $G_\sigma$. 
		For the kernel size in the convolution layers, we use $3\times 3$ kernels in the first layer (initialization) and the coarsest grid level. In other layers, we use $5\times5$ kernels. 
	\end{itemize}
	
	\subsection{Models and datasets}
	Algorithm \ref{alg.V.full} is a mathematical explanation of encoder-decoder-based neural networks. We compare the performance of Algorithm \ref{alg.V.full}  with popular encoder-decoder-based networks, such as Unet \cite{ronneberger2015u}, UNet++ \cite{zhou2019unetpp}, DeepLabV3+ \cite{chen2018encoder} and SegNet \cite{badrinarayanan2017segnet}. The implementation of UNet++ and DeepLabV3+ used the segmentation models PyTorch package \cite{Iakubovskii2019}. 
	
	We test all networks on two data sets: the cell segmentation dataset (CSD) \cite{dsb2018}, and the MSRA10K dataset \cite{cheng2014global}. CSD is used in the 2018 data science bowl competition. The stage 1 dataset contains 536 training images and 134 test images. We resize all images to a size of $96\times 96$. MSRA10K contains 10,000 salient object images with manually annotated masks. We choose 800 images for training and 200 images for testing. We resize all images to a size of $128\times 192$.

	\subsection{Training strategy}
	\label{sec.experiment.train}
	We aim to train our networks that are robust to several noise levels. If a network is trained on data with a certain noise level and tested on a data set with higher noise, the performance can be poor; but if the network is tested on a data set with lower noise, it will still provide good results. So we will train our network with a high noise level. To make the training more stable, we use a progressive training strategy: we first train the network on clean data set and then gradually increase the noise level. We add Gaussian noise with standard deviation (SD) $ \sigma \in \{0, 0.3, 0.5, 0.8, 1\}$. For each noise level, we use 500 epochs to train our network and use the trained parameters as initial values to train the network for the next noise level. In our experiments, we normalize all images so that all pixel values are between 0 and 1. 
	
	In our training, we tried two settings. In the first setting, for each noise level with SD=$a$, each pixel has a noise SD that is randomly generated from $[0,a]$. In the second setting, for each noise level with SD=$a$, all pixels have noise with SD=$a$. The comparisons of both settings are visualized in Figure \ref{fig.CSD.sd}, in which various models are tested on images with Gaussian noise varying from 0 to 1. In the figure, PottsMGNet is trained under the first setting, PottsMGNetSD is trained under the second setting, SD=$a$ denotes the model that is progressively trained with the largest SD=$a$, i.e., the training process stops at SD=$a$. For the second setting, if the model stops training at SD=0.5, then the model has the highest test accuracy on noise images with SD around 0.5. Further training the model on higher noise only makes the test accuracy better when tested on higher noise, but makes the test accuracy worse on low noise levels. While the first setting always gives robust results, similar phenomena are observed for other models.
	
	In the following comparisons, the first setting is always used with the highest SD being 1.
	
	\begin{figure}[t!]
		\centering
		\begin{tabular}{cc}
			accuracy & dice\\
			\includegraphics[width=0.4\textwidth]{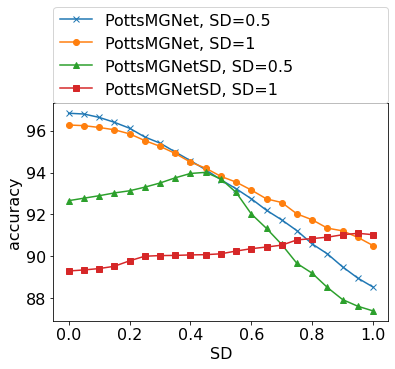} &
			\includegraphics[width=0.4\textwidth]{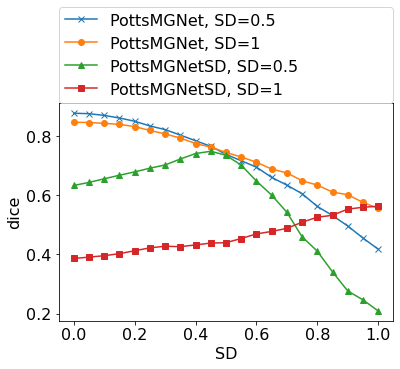} 
		\end{tabular}
		\caption{Results on CSD. Comparisons of different training strategies. PottsMGNet is trained under the first setting, PottsMGNetSD is trained under the second setting, SD=$a$ denotes the model that are progressively trained with largest SD=$a$.}
		\label{fig.CSD.sd}
	\end{figure}

	\subsection{Robustness to hyper parameters}
	\label{sec.experiment.para}
	In this section, we show that PottsMGNet is not sensitive to hyper parameters. We compare the performances of PottsMGNet with hyper parameters specified in Table \ref{tab.CSD.para}.
	\begin{table}[ht!]
		\centering
		\begin{tabular}{c|c|c|c}
			\hline \hline
			& $N$ & $\{L_1,L_2,L_3,L_3,L_5\}$ & $\{c_1,c_2,c_3,c_4,c_5\}$\\
			\hline
			Model 1 & 4 & $\{3,3,3,5,5\}$ & $\{32, 32, 64, 128, 256\}$\\
			\hline
			Model 2 &3 & $\{3,3,3,5,5\}$ & $\{32, 32, 64, 128, 256\}$\\
			\hline
			Model 3 &4 & $\{3,3,3,3,3\}$ & $\{32, 32, 64, 128, 256\}$\\
			\hline
			Model 4 &4 & $\{5,5,5,5,5\}$ & $\{32, 32, 64, 128, 256\}$\\
			\hline
			Model 5 &4 & $\{3,3,3,5,5\}$ & $\{32, 64, 128, 256, 256\}$\\
			\hline \hline
		\end{tabular}
		\caption{PottsMGNet with different hyper parameters considered in Section \ref{sec.experiment.para}.}
		\label{tab.CSD.para}
	\end{table}
	
	In Table \ref{tab.CSD.para}, Model 1 is the default model specified in Section \ref{sec.experiments.implementation}. With the training strategy described in Section \ref{sec.experiment.train} and for CSD, the testing accuracy and dice score of all models are presented in Figure \ref{fig.CSD.para}. All models have similar performances, implying that PottsMGNet is not sensitive to hyper parameters.
	
	
	\begin{figure}[t!]
		\centering
		\begin{tabular}{cc}
			accuracy & dice\\
			\includegraphics[width=0.4\textwidth]{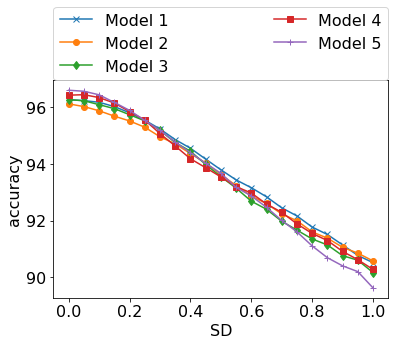} &
			\includegraphics[width=0.4\textwidth]{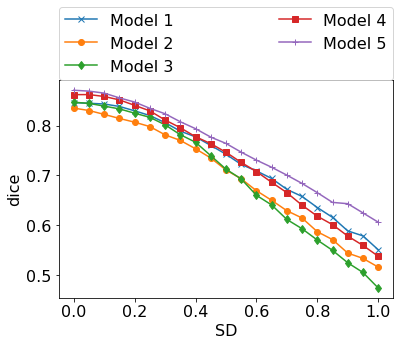} 
		\end{tabular}
		\caption{Results on CSD. Comparisons of different PottsMGNet with hyper parameters specified in Table \ref{tab.CSD.para}.}
		\label{fig.CSD.para}
	\end{figure}
	
	\subsection{CSD results}
	On CSD, we test the trained models on several noise levels. The comparisons with SD=0, 0.5, 0.8 and 1 are shown in Table \ref{tab.CSD}. Comparisons on more noise levels are visualized in Figure \ref{fig.CSD.comparison}. We observe that for all noise levels, PottsMGNet always provides the highest accuracy and dice score. Some sample comparisons are presented in Figure \ref{fig.CSD.reuslts}. The columns from left to right correspond to noise SD= 0, 0.3, 0.5 and 0.7. Even with very high noise (SD=0.7), PottsMGNet still identifies a large portion of cells. 
	\begin{table}[t!]
		\centering		
		\begin{tabular}{c|c|c|c|c|c}
			\hline\hline
			& PottsMGNet & UNet & UNet++ & DeepLab V3+ & SegNet\\
			\hline
			\multicolumn{6}{c}{SD=0}\\
			\hline
			accuracy& 96.27\% & 94.24\% & 94.10\% & 93.88\% & 94.48\\
			\hline
			dice & 0.8460 & 0.7401 & 0.7332 & 0.7231 & 0.7474\\
			\hline
			\hline
			\multicolumn{6}{c}{SD=0.5}\\
			\hline
			accuracy& 93.85\% & 91.11\% & 91.12\% & 91.87\% & 91.66\\
			\hline
			dice & 0.7445 & 0.5659 & 0.5571 & 0.5996 & 0.5591\\
			\hline
			\hline
			\multicolumn{6}{c}{SD=0.8}\\
			\hline
			accuracy& 91.78\% & 89.73\% & 89.08\% & 90.50\% & 90.31\\
			\hline
			dice & 0.6353 & 0.4938 & 0.4652 & 0.4988 & 0.4622\\
			\hline
			\hline
			\multicolumn{6}{c}{SD=1}\\
			\hline
			accuracy& 90.50\% & 89.10\% & 87.48\% & 89.79\% & 89.57\\
			\hline
			dice & 0.5554 & 0.4625 & 0.4224 & 0.4582 & 0.4087\\
			\hline
			\hline
		\end{tabular}
		\caption{Results on CSD. Comparisons of accuracy and dice score of different models for noise SD=0, 0.5, 0.8, 1. }
		\label{tab.CSD}
	\end{table}
	
	\begin{figure}[t!]
		\centering
		\begin{tabular}{cc}
			accuracy & dice\\
			\includegraphics[width=0.4\textwidth]{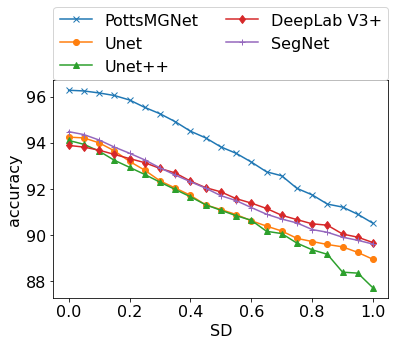} &
			\includegraphics[width=0.4\textwidth]{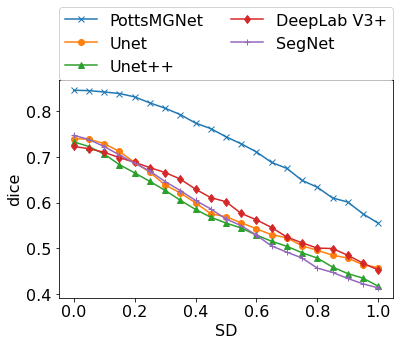} 
		\end{tabular}
		\caption{Results on CSD. Comparisons of accuracy and dice score of different models on for noise SD varying from 0 to 1. }
		\label{fig.CSD.comparison}
	\end{figure}
	
	\begin{figure}
		\begin{tabular}{ccccccc}
			Clean Image & 		\includegraphics[align=c,width=0.18\textwidth]{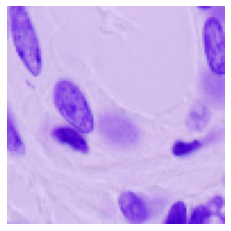}&
			\includegraphics[align=c,width=0.18\textwidth]{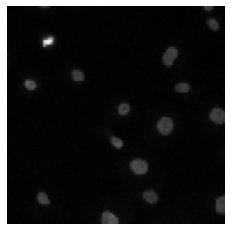}&
			\includegraphics[align=c,width=0.18\textwidth]{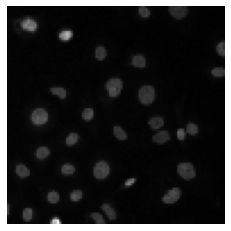}&
			\includegraphics[align=c,width=0.18\textwidth]{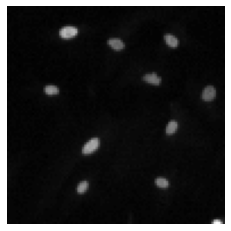}\\
			Noisy Image & 
			\includegraphics[align=c,width=0.18\textwidth]{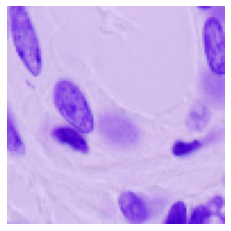}&
			\includegraphics[align=c,width=0.18\textwidth]{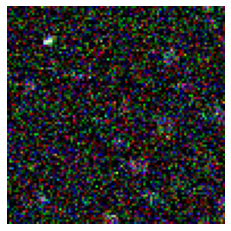}&
			\includegraphics[align=c,width=0.18\textwidth]{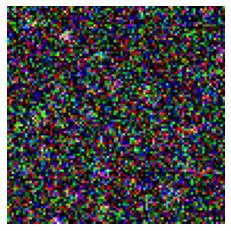}&
			\includegraphics[align=c,width=0.18\textwidth]{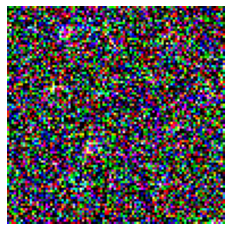}\\
			Mask & 
			\includegraphics[align=c,width=0.18\textwidth]{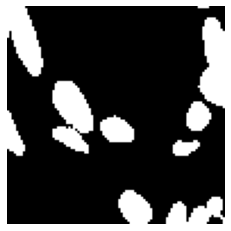}&
			\includegraphics[align=c,width=0.18\textwidth]{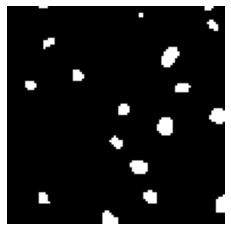}&
			\includegraphics[align=c,width=0.18\textwidth]{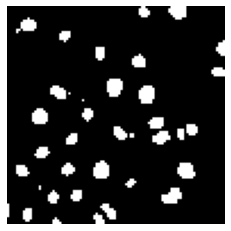}&
			\includegraphics[align=c,width=0.18\textwidth]{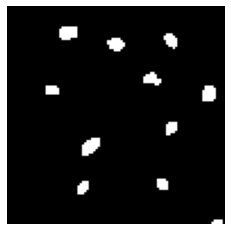}\\
			PottsMGNet & 
			\includegraphics[align=c,width=0.18\textwidth]{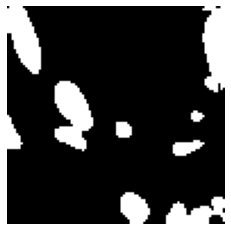}&
			\includegraphics[align=c,width=0.18\textwidth]{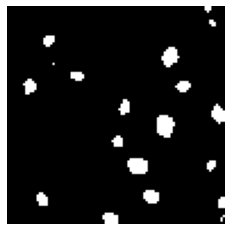}&
			\includegraphics[align=c,width=0.18\textwidth]{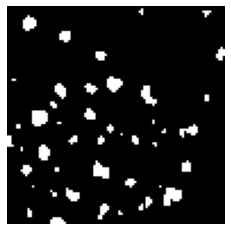}&
			\includegraphics[align=c,width=0.18\textwidth]{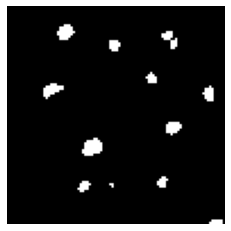}\\
			UNet & 
			\includegraphics[align=c,width=0.18\textwidth]{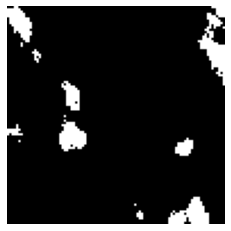}&
			\includegraphics[align=c,width=0.18\textwidth]{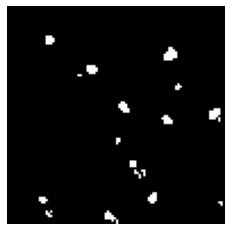}&
			\includegraphics[align=c,width=0.18\textwidth]{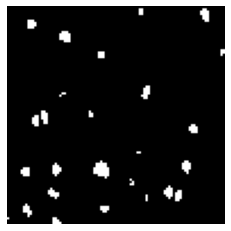}&
			\includegraphics[align=c,width=0.18\textwidth]{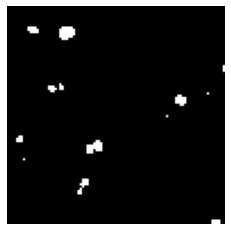}\\
			UNet++ & 
			\includegraphics[align=c,width=0.18\textwidth]{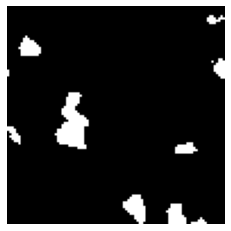}&
			\includegraphics[align=c,width=0.18\textwidth]{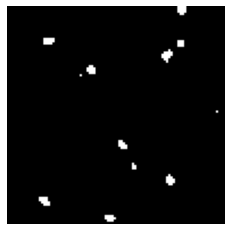}&
			\includegraphics[align=c,width=0.18\textwidth]{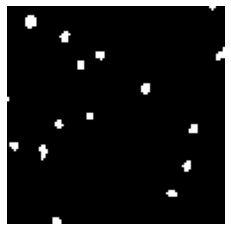}&
			\includegraphics[align=c,width=0.18\textwidth]{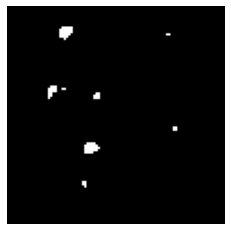}\\
			DeepLabV3+ & 
			\includegraphics[align=c,width=0.18\textwidth]{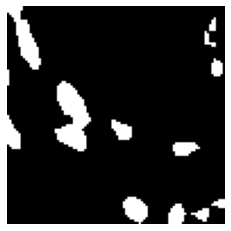}&
			\includegraphics[align=c,width=0.18\textwidth]{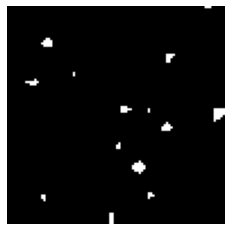}&
			\includegraphics[align=c,width=0.18\textwidth]{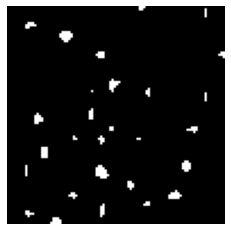}&
			\includegraphics[align=c,width=0.18\textwidth]{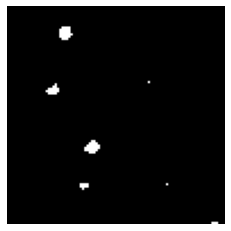}\\
			SegNet & 
			\includegraphics[align=c,width=0.18\textwidth]{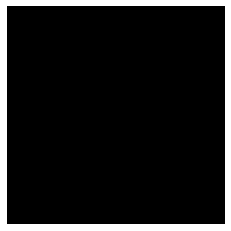}&
			\includegraphics[align=c,width=0.18\textwidth]{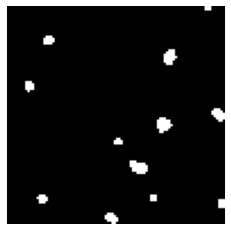}&
			\includegraphics[align=c,width=0.18\textwidth]{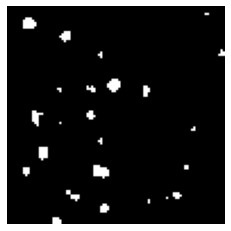}&
			\includegraphics[align=c,width=0.18\textwidth]{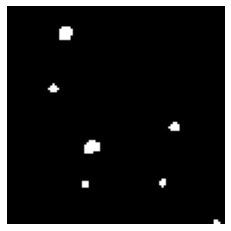}\\
		\end{tabular}
		\caption{Results on CSD. Examples of results by different models. Columns 1--4 correspond to noise SD=0, 0.3, 0.5, 0.7.}
		\label{fig.CSD.reuslts}
	\end{figure}
	
	\subsection{MSRA10K}
	For MSRA10K, we train all models with noise standard deviation (SD) $\{0, 0.3, 0.5, 0.8\}$. For each noise level, 500 epochs are used. we test the trained models on several noise levels. The comparisons with SD=0, 0.5 and 0.8 are shown in Table \ref{tab.MA}. Comparisons on more noise levels are visualized in Figure \ref{fig.MA.comparison}. Similar to our observation for CSD, for most noise levels, PottsMGNet always provides the highest accuracy and dice score. Some sample comparisons are presented in Figure \ref{fig.MA.reuslts}. The columns from left to right correspond to noise SD= 0, 0.3, 0.5 and 0.8. In all examples, PottsMGNet can better segment the target object. The segmented images are presented in Figure \ref{fig.MA.reuslts.seg}.
	\begin{table}[t!]
		\centering		
		\begin{tabular}{c|c|c|c|c|c}
			\hline\hline
			& PottsMGNet & UNet & UNet++ & DeepLab V3+ & SegNet\\
			\hline
			\multicolumn{6}{c}{SD=0}\\
			\hline
			accuracy& 93.28\% & 92.14\% & 92.49\% & 92.47\% & 92.20\\
			\hline
			dice & 0.8417 & 0.8012 & 0.8190 & 0.8170 & 0.8108\\
			\hline
			\hline
			\multicolumn{6}{c}{SD=0.5}\\
			\hline
			accuracy& 92.69\% & 91.83\% & 91.77\% & 91.96\% & 91.99\\
			\hline
			dice & 0.8278 & 0.7976 & 0.8016 & 0.8074 & 0.8071\\
			\hline
			\hline
			\multicolumn{6}{c}{SD=0.8}\\
			\hline
			accuracy& 91.46\% & 90.94\% & 90.28\% & 90.89\% & 91.23\\
			\hline
			dice & 0.7969 & 0.7726 & 0.7728 & 0.7802 & 0.7907\\
			\hline
			\hline
		\end{tabular}
		\caption{Results on MSRA10K. Comparisons of accuracy and dice score of different models for noise SD=0, 0.5, 0.8.}
		\label{tab.MA}
	\end{table}
	
	\begin{figure}[t!]
		\centering
		\begin{tabular}{cc}
			accuracy & dice\\
			\includegraphics[width=0.4\textwidth]{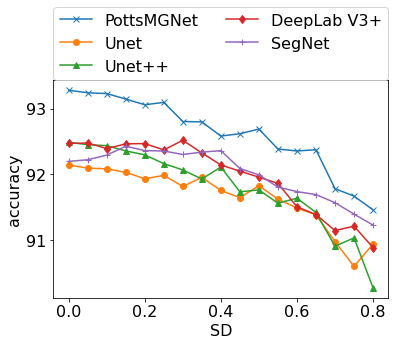} &
			\includegraphics[width=0.4\textwidth]{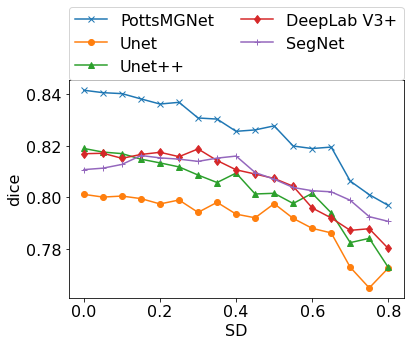} \\
		\end{tabular}
		\caption{Results on MSRA10K. Comparisons of accuracy and dice score of different models on for noise SD varying from 0 to 1.}
		\label{fig.MA.comparison}
	\end{figure}
	
	\begin{figure}
		\begin{tabular}{ccccccc}
			Clean Image & 
			\includegraphics[align=c,width=0.18\textwidth]{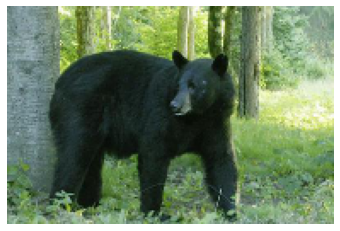}&
			\includegraphics[align=c,width=0.18\textwidth]{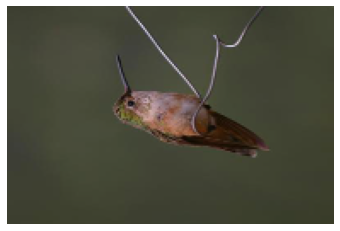}&
			\includegraphics[align=c,width=0.18\textwidth]{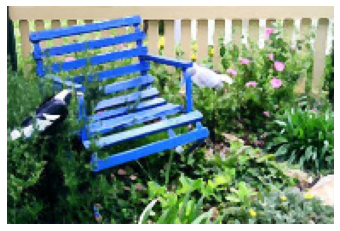}&
			\includegraphics[align=c,width=0.18\textwidth]{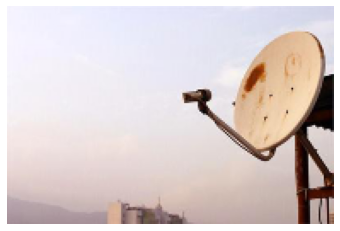}\\
			Noisy Image & 
			\includegraphics[align=c,width=0.18\textwidth]{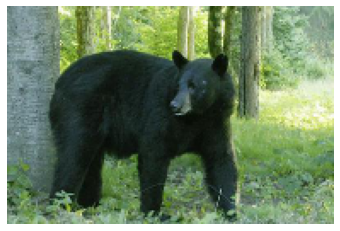}&
			\includegraphics[align=c,width=0.18\textwidth]{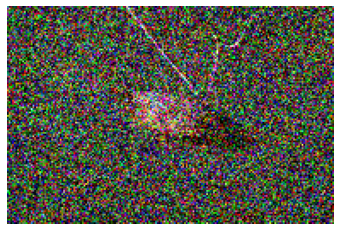}&
			\includegraphics[align=c,width=0.18\textwidth]{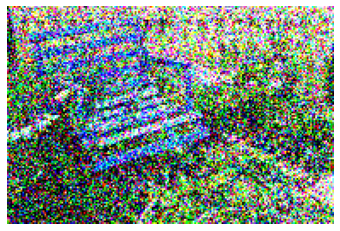}&
			\includegraphics[align=c,width=0.18\textwidth]{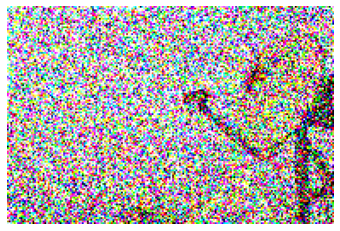}\\
			Mask & 
			\includegraphics[align=c,width=0.18\textwidth]{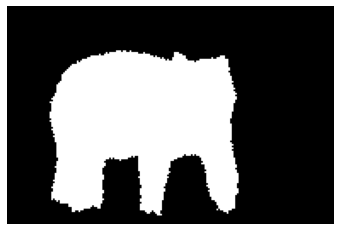}&
			\includegraphics[align=c,width=0.18\textwidth]{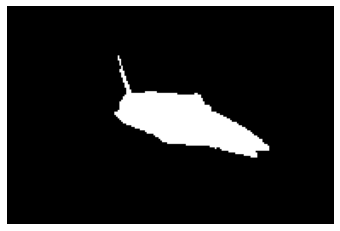}&
			\includegraphics[align=c,width=0.18\textwidth]{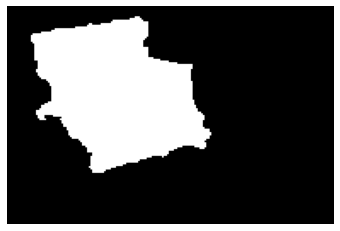}&
			\includegraphics[align=c,width=0.18\textwidth]{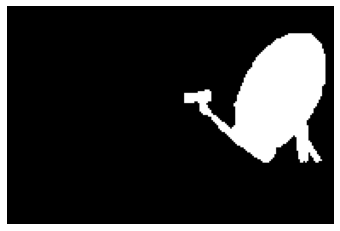}\\
			PottsMGNet & 
			\includegraphics[align=c,width=0.18\textwidth]{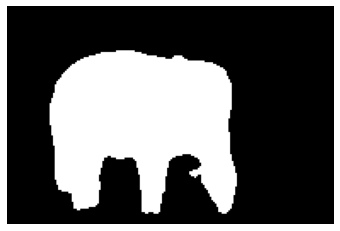}&
			\includegraphics[align=c,width=0.18\textwidth]{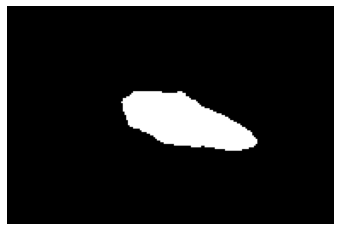}&
			\includegraphics[align=c,width=0.18\textwidth]{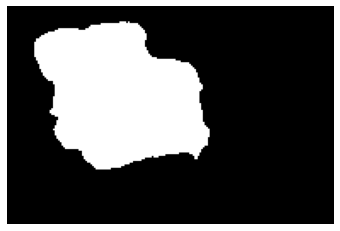}&
			\includegraphics[align=c,width=0.18\textwidth]{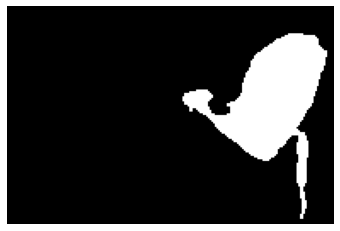}\\
			UNet & 
			\includegraphics[align=c,width=0.18\textwidth]{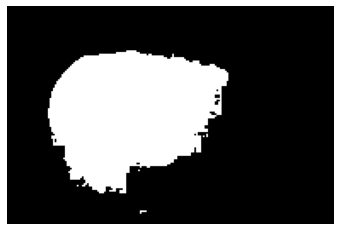}&
			\includegraphics[align=c,width=0.18\textwidth]{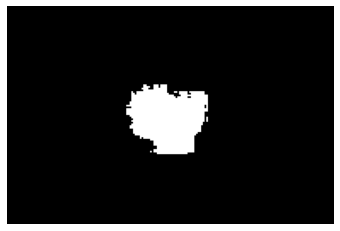}&
			\includegraphics[align=c,width=0.18\textwidth]{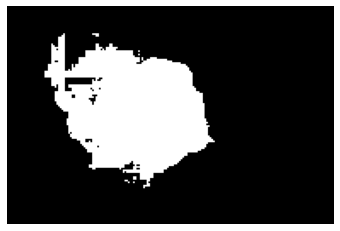}&
			\includegraphics[align=c,width=0.18\textwidth]{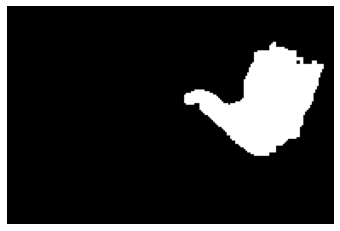}\\
			UNet++ & 
			\includegraphics[align=c,width=0.18\textwidth]{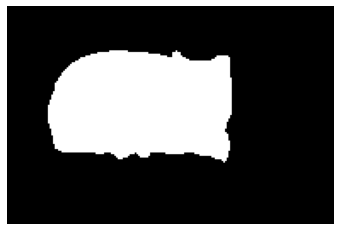}&
			\includegraphics[align=c,width=0.18\textwidth]{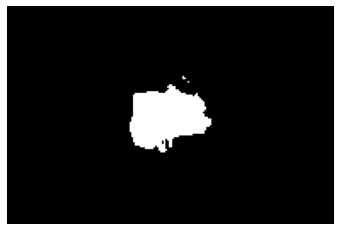}&
			\includegraphics[align=c,width=0.18\textwidth]{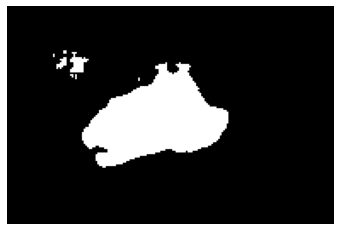}&
			\includegraphics[align=c,width=0.18\textwidth]{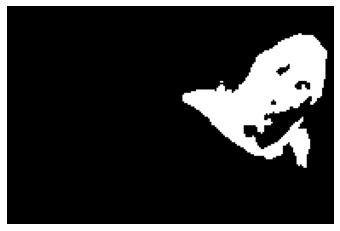}\\
			DeepLabV3+ & 
			\includegraphics[align=c,width=0.18\textwidth]{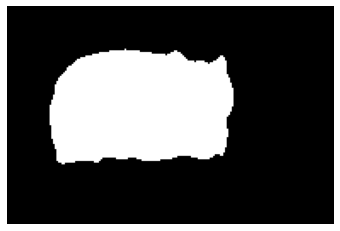}&
			\includegraphics[align=c,width=0.18\textwidth]{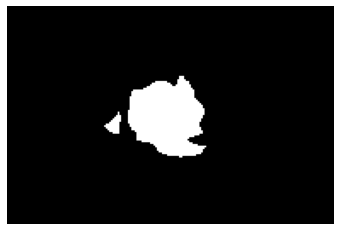}&
			\includegraphics[align=c,width=0.18\textwidth]{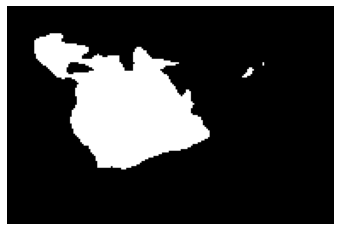}&
			\includegraphics[align=c,width=0.18\textwidth]{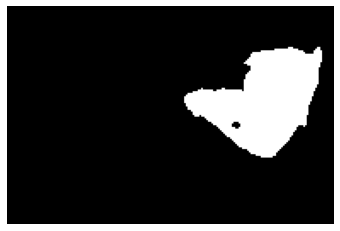}\\
			SegNet & 
			\includegraphics[align=c,width=0.18\textwidth]{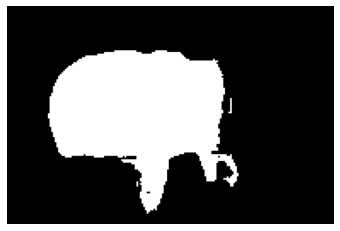}&
			\includegraphics[align=c,width=0.18\textwidth]{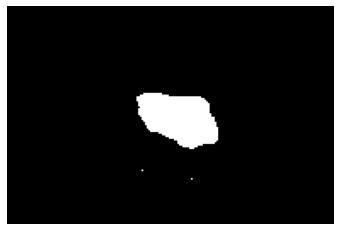}&
			\includegraphics[align=c,width=0.18\textwidth]{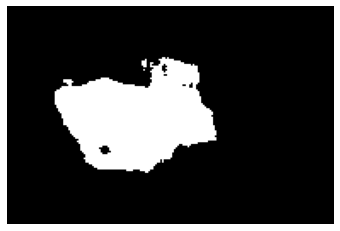}&
			\includegraphics[align=c,width=0.18\textwidth]{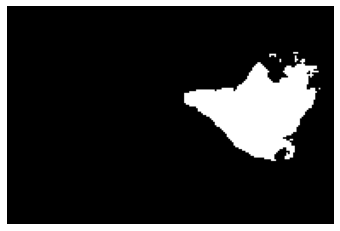}\\
		\end{tabular}
		\caption{Results on MSRA10K. Examples of results by different models. Columns 1--4 correspond to noise SD=0, 0.3, 0.5, 0.8.}
		\label{fig.MA.reuslts}
	\end{figure}
	
	\begin{figure}
		\begin{tabular}{ccccccc}
			Clean Image & 
			\includegraphics[align=c,width=0.18\textwidth]{figures/MA-all-SD00-1-7-img0}&
			\includegraphics[align=c,width=0.18\textwidth]{figures/MA-all-SD03-249-7-img0}&
			\includegraphics[align=c,width=0.18\textwidth]{figures/MA-all-SD05-81-0-img0}&
			\includegraphics[align=c,width=0.18\textwidth]{figures/MA-all-SD08-99-5-img0}\\
			Noisy Image & 
			\includegraphics[align=c,width=0.18\textwidth]{figures/MA-all-SD00-1-7-img}&
			\includegraphics[align=c,width=0.18\textwidth]{figures/MA-all-SD03-249-7-img}&
			\includegraphics[align=c,width=0.18\textwidth]{figures/MA-all-SD05-81-0-img}&
			\includegraphics[align=c,width=0.18\textwidth]{figures/MA-all-SD08-99-5-img}\\
			Mask & 
			\includegraphics[align=c,width=0.18\textwidth]{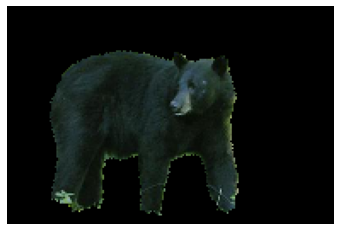}&
			\includegraphics[align=c,width=0.18\textwidth]{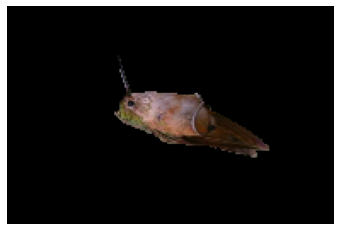}&
			\includegraphics[align=c,width=0.18\textwidth]{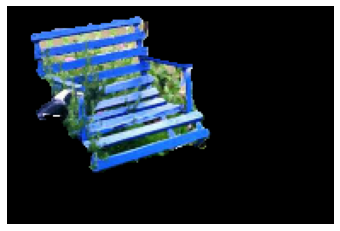}&
			\includegraphics[align=c,width=0.18\textwidth]{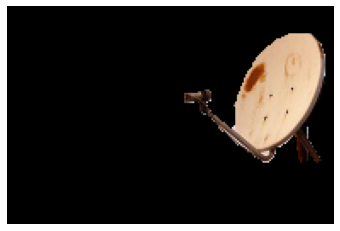}\\
			PottsMGNet & 
			\includegraphics[align=c,width=0.18\textwidth]{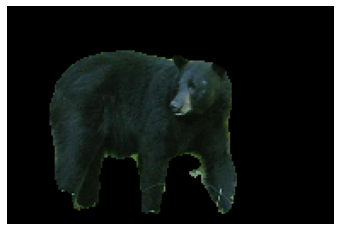}&
			\includegraphics[align=c,width=0.18\textwidth]{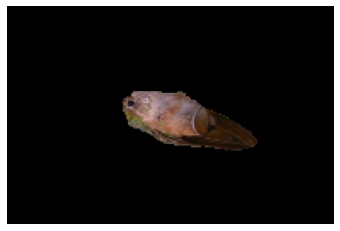}&
			\includegraphics[align=c,width=0.18\textwidth]{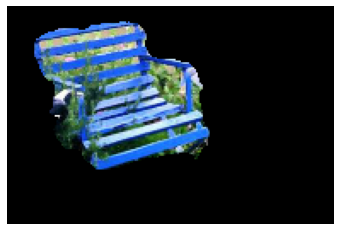}&
			\includegraphics[align=c,width=0.18\textwidth]{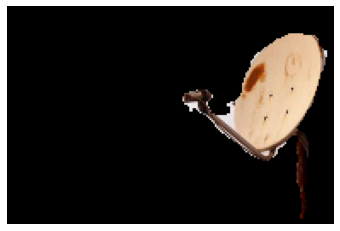}\\
			UNet & 
			\includegraphics[align=c,width=0.18\textwidth]{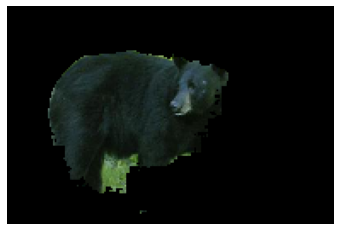}&
			\includegraphics[align=c,width=0.18\textwidth]{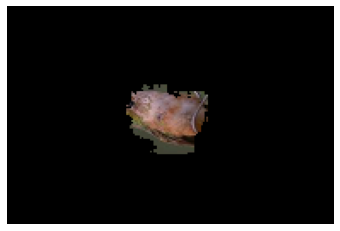}&
			\includegraphics[align=c,width=0.18\textwidth]{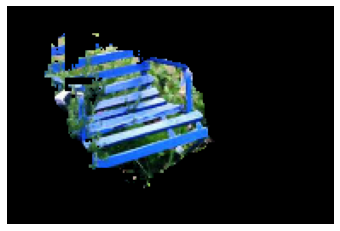}&
			\includegraphics[align=c,width=0.18\textwidth]{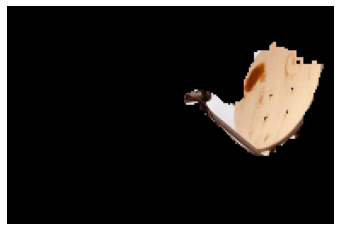}\\
			UNet++ & 
			\includegraphics[align=c,width=0.18\textwidth]{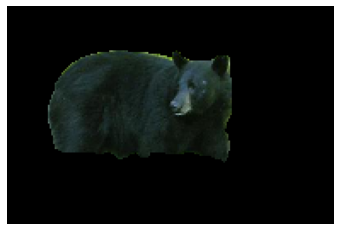}&
			\includegraphics[align=c,width=0.18\textwidth]{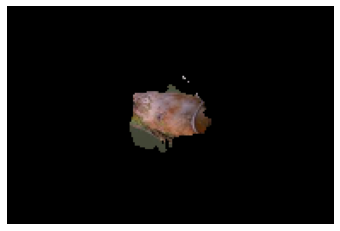}&
			\includegraphics[align=c,width=0.18\textwidth]{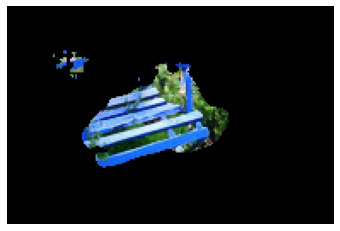}&
			\includegraphics[align=c,width=0.18\textwidth]{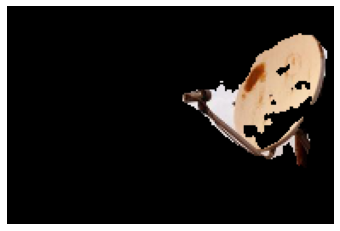}\\
			DeepLabV3+ & 
			\includegraphics[align=c,width=0.18\textwidth]{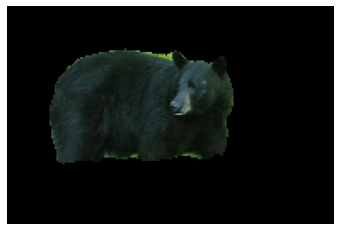}&
			\includegraphics[align=c,width=0.18\textwidth]{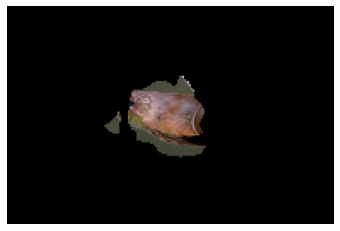}&
			\includegraphics[align=c,width=0.18\textwidth]{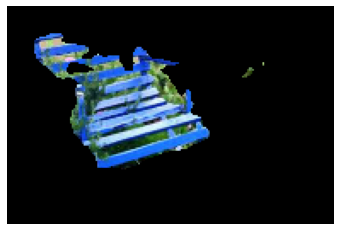}&
			\includegraphics[align=c,width=0.18\textwidth]{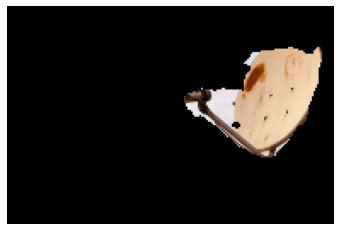}\\
			SegNet & 
			\includegraphics[align=c,width=0.18\textwidth]{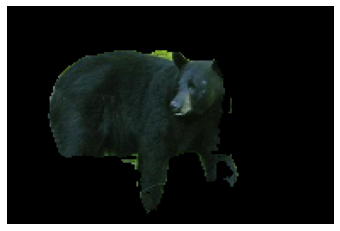}&
			\includegraphics[align=c,width=0.18\textwidth]{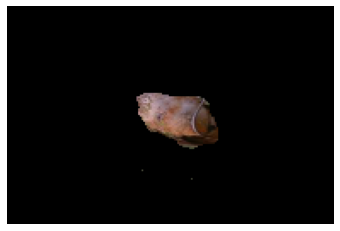}&
			\includegraphics[align=c,width=0.18\textwidth]{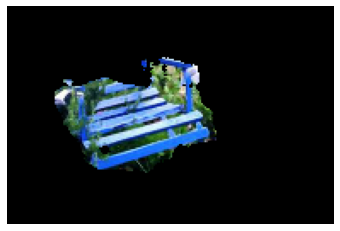}&
			\includegraphics[align=c,width=0.18\textwidth]{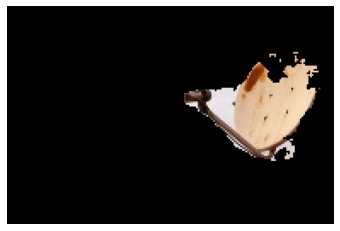}\\
		\end{tabular}
		\caption{Results on MSRA10K. Segmented images in Figure \ref{fig.MA.reuslts}.}
		\label{fig.MA.reuslts.seg}
	\end{figure}

	\section{Conclusion}
	\label{sec.conclusion}
	In this paper we propose the PottsMGNet algorithm for image segmentation. Starting from the two-phase Potts model, a control problem is considered. 
	The proposed Potts net is a first-order operator splitting algorithm that solves the control problem. The Potts net has two ingredients: (i) it is based on a novel hybrid splitting strategy for initial value problems, and (ii) it uses a multigrid idea to split all computations into several substeps according to the grid level. The PottsMGNet essentially is a neural network. From the algorithmic perspective, it provides mathematical explanations for most encoder-decoder type neural networks for image segmentation: these networks are operator-splitting schemes for some control problems. Our numerical experiments show that PottsMGNet is more robust to different noise levels compared to existing networks.

	\appendix
	
	\section*{Appendix}
	
	\section{Gradient flow for the Potts model}
	\label{sec.potts.derivation}
	The Potts model is a popular approach used in image segmentation and classification tasks with ideas originating from \cite{potts1952some}.  It extends the Ising model to multiple states, allowing it to be used for classification tasks \cite{geman1984stochastic,boykov2001fast}.
	We discuss the derivation of the gradient flow of the Potts model in a multi-phase setting. The flow for the two-phase model used in (\ref{sec.potts}) is a special case.
	
	Let $\Omega$ be the image domain. The continuous $K$--phase Potts model is in the form of 
	\cite{yuan2010study,tai2021potts}:
	\begin{align}
		\begin{cases}
			\min\limits_{\Omega_k,\  k=1,2\cdots K}  
			\displaystyle \sum_{k=1}^K \left ( \displaystyle\int_{\Omega_k} f_k(\xb) d\xb   + \lambda |\partial \Omega_k| \right),\\
			\cup_{k=1}^K  \Omega_k =\Omega,\ 
			\Omega_{k_1}\cap \Omega_{k_2} =\emptyset, \forall k_1 \not = k_2,  
		\end{cases}
		\label{eq.potts.0}
	\end{align}
	where $|\partial\Omega_k|$ is the perimeter of $\partial \Omega_k$, $\lambda\ge 0$ is a given constant and $f_k$'s are some nonnegative given functions. A popular choice of $f_k$ is  $f_k(\xb)= (f(\xb)-c_k)^2/\alpha$ 	in which $\alpha$ is a scaling parameter and $c_k$ is the estimated mean density of $f(\xb)$ on $\Omega_k$. With such a choice, the Potts model (\ref{eq.potts.0}) reduces to the piecewise Mumford--Shah functional, which is also known as Chan--Vese model \cite{chan2001active}. 
	
	For any set $\Omega_k\in \Omega$, define the indicator function as
	\begin{align}
		\bone_{\Omega_k}(\xb)=\begin{cases}
			1 & \mbox{ if } \xb\in \Omega_k,\\
			0 & \mbox{ if } \xb\notin \Omega_k.
		\end{cases}
	\end{align}
	If we set 
	$\vb  =(v_1,v_2,\cdots, v_K)$,  
	$v_k(\xb)=\bone_{\Omega_k}(\xb)$, 
	$\fb  =(f_1,f_2,\cdots, f_K)$, then solving the Potts model (\ref{eq.potts.0}) is equivalent to solving 
	\begin{align}
		\min_{\vb (\xb)\in \hat{{\cal S}}} \left[\int_{\Omega} \vb (\xb)  \cdot \fb(\xb) d\xb +\sum_{k=1}^K | \partial \Omega_k | \right],
		\label{eq.potts.1}
	\end{align}
	where
	\begin{equation}
		\hat{{\cal S}} = \bigg \{ \vb : \vb  =(v_1,v_2,\cdots, v_K), \sum_{k=1}^K v_k(\xb) =1, 
		v_k (\xb) \in \{ 0, 1 \} \bigg   \} .
	\end{equation}
	The minimizer of (\ref{eq.potts.1}), denoted by $\ub$, segments the image into $K$ regions: region $\Omega_k$ is represented by  $\{\xb: u_k(\xb)=1\}$. 	The term for the perimeter regularization  can be approximated by a smoothed version using threshold dynamics  \cite{ alberti1998non,merriman1992diffusion,esedog2015threshold,liu2022deep}:
	\begin{align}
		\sum_{k=1}^K 	|\partial \Omega_k|  \equiv  \sum_{k=1}^K  	\sqrt{\frac{\pi}{\sigma}} \int_{\Omega} u_k (\xb)(G_\sigma*(1-u_k ))(\xb)d\xb,
		\label{eq.per.dynamic}
	\end{align}
	where $G_\sigma$ is the Gaussian kernel 	
 $
	G_\sigma(\xb)=\frac{1}{2\pi\sigma^2} \exp\left(-\frac{\|\xb\|^2}{2\sigma^2}\right).
	$ 	
 One can show that the right-hand term in (\ref{eq.per.dynamic}) converges to $\sum_{k=1}^K | \partial \Omega_k | $ as $\sigma\rightarrow 0$ \cite{alberti1998non,merriman1992diffusion,miranda2007short}. Then the functional we are minimizing is 
	\begin{align}
		\min_{\vb \in \hat{{\cal S}}} \left[\int_{\Omega} 
		\vb(\xb)  \cdot f(\xb)  d \xb + \frac \lambda 2 \int_{\Omega} \vb (\xb)\cdot (G_\sigma *(1-\vb  ))(\xb) d\xb\right],
		\label{eq.potts.3}
	\end{align}
	where 
	\begin{align}
		G_{\sigma}*(1-\vb)=\begin{bmatrix}
			G_{\sigma}*(1-v_1) & \cdots & G_{\sigma}*(1-v_K)
		\end{bmatrix}^{\top}.
	\end{align}
	
	We then relax the constraint $\vb(\xb)\in\{0,1\}$ to $\vb(\xb)\in [0,1]$  and consider the following minmization problem with a smoothing parameter $\varepsilon$:
	\begin{align}
		\min_{\vb \in \cS } \left[ 
		\int_{\Omega} \vb\cdot  \fb d\xb + \varepsilon\int_{\Omega} \vb \cdot \ln \vb  d\xb +\frac \lambda 2 \int_{\Omega} \vb (\xb)\cdot  (G_\sigma *(1-\vb ))(\xb) d\xb\right],
		\label{eq.potts.4}
	\end{align}
	with 
	\begin{equation}
		\cS = \bigg \{ \vb : \vb  =(v_1,v_2,\cdots, v_K), \sum_{k=1}^K v_k(\xb) =1, 
		v_k (\xb) \ge 0  \bigg   \} .
	\end{equation}
	If $\ub \in \cS$ is a minimizer of the above energy functional, we have the following lemma 
	\begin{lemma} \label{lem.limit}
		When $\varepsilon \mapsto 0$, we have $\ub\rightarrow \{0, 1\}^K$ for any $\xb\in\Omega$. 
		
	\end{lemma}
	
	\begin{proof}
		Denote 
		\begin{align}
			F_{\varepsilon}(\vb)=\int_{\Omega} \vb\cdot  \fb d\xb + \varepsilon\int_{\Omega} \vb \cdot \ln \vb  d\xb +\frac \lambda 2 \int_{\Omega} \vb (\xb)\cdot  (G_\sigma *(1-\vb ))(\xb) d\xb.
		\end{align}
		Let $\{F_{\varepsilon(n)}\}_{n=1}^{\infty}$ be any sequence so that $\lim_{n\rightarrow \infty} \varepsilon(n)=0$. Let $\vb^*$ be any elements in $\cS$ and $\{\vb_n\}_{n=1}^{\infty}\subset \cS$ be a sequence satisfying $\lim_{n\rightarrow \infty} \vb_n=\vb^*$. Since $F_{\varepsilon}$ is continuous in $\varepsilon$ and $\vb$, we have
		\begin{align}
			\lim_{n\rightarrow \infty} F_{\varepsilon(n)}(\vb_n)= F(\vb).
		\end{align}
		Therefore, $\{F_{\varepsilon(n)}\}_{n=1}^{\infty}$ $\Gamma$-converges to $\cF$. Proving Lemma \ref{lem.limit} reduces to proving $\ub\in\{0,1\}^K$ for the limiting case, i.e., when $\varepsilon=0$. 
		
		Our proof borrows techniques from \cite{wang2021efficient}. When $\varepsilon=0$, we have
		\begin{align}
			\ub = &\min_{\vb \in \cS} \int_{\Omega} \vb\cdot \fb d\xb + \frac{\lambda}{2}\int_{\Omega} \vb\cdot G_{\sigma} * (1-\vb)d\xb \nonumber\\
			=& \min_{\vb \in \cS} \sum_{k=1}^{K} \left[\int_{\Omega} v_kf_k d\xb + \frac{\lambda}{2}\int_{\Omega} v_k  G_{\sigma} * (1-v_k)d\xb\right]
			\label{eq.potts.tu}
		\end{align}
		which is concave in each $v_k$. In (\ref{eq.potts.4}), we can represent $v_K=1-\sum_{k=1}^{K-1}v_k$. Denote $\widetilde{\vb}=[v_1,...,v_{K-1}]^{\top}$ and $\widetilde{\fb}=[f_1-f_K,...,f_{K-1}-f_K]^{\top}$. Denote the set $\widetilde{\cal S}=\{\widetilde{\vb}: 0\leq \widetilde{v}_k\leq 1 \mbox{ for } k=1,...,K-1, \mbox{ and }  0\leq \sum_{k=1}^{K-1} \widetilde{v}_k\leq 1\}.$ If $\ub$ is a minimizer of (\ref{eq.potts.tu}), then $\widetilde{\ub}$ is a minimizer of
		\begin{align}
			\min_{\widetilde{\vb} \in \widetilde{S}} \sum_{k=1}^{K-1} \left[\int_{\Omega} \widetilde{v}_k(f_k-f_K) d\xb + \frac{\lambda}{2}\int_{\Omega} \widetilde{v}_k  G_{\sigma}* (1-\widetilde{v}_k)-\widetilde{v}_kG_{\sigma}*\left(\sum_{k=1}^{K-1} \widetilde{v}_k\right) d\xb\right].
		\end{align}
		We prove the lemma by contradiction. 
		Assume $\widetilde{\ub}\in \widetilde{\cal S}$ is not binary. There are two cases. 
		
		\noindent {\bf Case 1:} There exist a constant $c>0$  and a set $\cK$ on which 
		\begin{align}
			c< \widetilde{u}_{k^*}< 1-c.
		\end{align}
		for some $1\leq k^*\leq K-1$ and $\widetilde{u}_k=0$ for $k\neq k^*$. Let $\delta$ be the indicator function of $\cK$ and $\bchi=[\chi_1,...,\chi_{K-1}]^{\top}$ such that $\chi_{k^*}=\delta$ and $\chi_k=0$ for $k\neq k^*$. Then $\widetilde{\ub}+t\bchi\in \widetilde{\cal S}$ for $|t|\leq c$.
		We deduce that
		\begin{align*}
			\frac{d (\widetilde{\ub}+t\bchi)}{d t}=\int_{\Omega} \delta(f_{k^*}-f_{K-1})d\xb &+ \frac{\lambda}{2}\int_{\Omega} \delta G_{\sigma}*(1-\widetilde{u}_{k^*}-t\delta)-\delta G_{\sigma}*(\widetilde{u}_{k^*}+t\delta)d\xb\\
			&-\frac{\lambda}{2}\int_{\Omega} \delta G_{\sigma}*\left(t\delta+\sum_{k=1}^{K-1} \widetilde{u}_k\right) + \delta G_{\sigma}*(\widetilde{u}_{k^*}+t\delta)d\xb,
		\end{align*}
		and
		\begin{align*}
			\frac{d^2 (\widetilde{\ub}+t\bchi)}{d t^2}=-2\lambda\int_{\Omega} \delta G_{\sigma} *\delta d\xb<0,
		\end{align*}
		which contradicts that $\widetilde{\ub}$ is a minimizer.
		
		\noindent {\bf Case 2:} There exist a constant $c>0$ and a set $\cK$ on which 
		\begin{align}
			c< \widetilde{u}_{k_1^*}< 1-c, \quad c< \widetilde{u}_{k_2^*}< 1-c
		\end{align}
		for some $1\leq k_1^*,k_2^*\leq K-1$ and $\widetilde{u}_k=0$ for $k\neq k_1^*,k_2^*$.
		Let $\delta$ be the indicator function of $\cK$ and $\bchi=[\chi_1,...,\chi_{K-1}]^{\top}$ such that $\chi_{k_1^*}=\delta$, $\chi_{k_2^*}=-\delta$, and $\chi_k=0$ for $k\neq k_1^*,k_2^*$. Then $\widetilde{\ub}+t\bchi\in \widetilde{\cal S}$ for $|t|\leq c$.
		
		We deduce that
		\begin{align*}
			\frac{d (\widetilde{\ub}+t\bchi)}{d t}=&\int_{\Omega} \delta(f_{k_1^*}-f_{K-1})-\delta(f_{k_2^*}-f_{K-1}) d\xb \\
			&+ \frac{\lambda}{2}\int_{\Omega} \delta G_{\sigma}*(1-\widetilde{u}_{k_1^*}-t\delta)-\delta G_{\sigma}*(\widetilde{u}_{k_1^*}+t\delta)d\xb\\
			&+\frac{\lambda}{2}\int_{\Omega} -\delta G_{\sigma}*(1-\widetilde{u}_{k_2^*}+t\delta)+\delta G_{\sigma}*(\widetilde{u}_{k_2^*}-t\delta)d\xb
		\end{align*}
		and
		\begin{align*}
			\frac{d^2 (\widetilde{\ub}+t\bchi)}{d t^2}=-2\lambda\int_{\Omega} \delta G_{\sigma} *\delta d\xb<0,
		\end{align*}
		which contradicts that $\widetilde{\ub}$ is a minimizer.
		
		In conclusion, $\widetilde{\ub}$ must be binary and thus $\ub$ must be binary.
		
		
	\end{proof} 
	The Euler-Lagrangian equation of (\ref{eq.potts.4}) reads as
	$$ \varepsilon (1 +  \ln \ub)   + \lambda G_\sigma *(1-2\ub ) + f + \partial I_{\cS} (\ub )  \ni  \mathbf{0}, \quad \forall \xb \in \Omega, $$
	where $I_{\cS}(\cdot )$ is the indicator function for set $\cS$ and $\partial I_{\cS}(\ub)$ is its sub-differential at $\ub$. The gradient flow for this problem is
	$$
	\begin{cases}
		\frac {\partial \ub}{\partial t}   + \varepsilon( 1 +  \ln \ub)  + \lambda G_\sigma *(1-2\ub) + \fb +  \partial I_{\cS} (\ub ) \ni {\bf 0}, \ (\xb,t)\in \Omega\times (0,T],\\  
		\ub(\xb,0) = \ub_0(\xb),   \xb \in \Omega . 
	\end{cases}
	$$
	For the two-phase Potts model, i.e. $K=2$, we use $v=v_1, v_2=1-v$, then the minimization problem (\ref{eq.potts.4})  becomes (\ref{eq.potts2.4}).
	
	We want to mention that (\ref{eq.potts.4}) is a regularized softmax in the sense that it reduces to the softmax function when $\lambda=0$ and $\varepsilon=1$. It has been used in \cite{liu2022deep} to embed variational models into traditional neural networks. If we replace the Gaussian kernel $G_\sigma(\xb)$ by $B(\xb/\sigma)$ where $B(x)$ is the indicator function for the unit ball, then the length approximation is still correct, see \cite{wang2009edge} for a proof. It has been used in \cite{Liu2011} in the same way as threshold dynamics for image segmentation. 
	
	\section{Multigrids and image functions over the grids}
	\label{app.multigrid}
	
	In the discrete setting, an image can be viewed as a piecewise constant function on a grid. Images with different resolutions can then be viewed as functions on grids of different sizes. 
	Without loss of generality, we assume the original discrete image has a grid (resolution) $\cT$ of size $m\times n$, and grid step size $h$, with 
	$$m=2^{s_1},\quad n=2^{s_2}$$
	for some $h>0$ and integers $s_1,s_2>0$. The image $f$ has a constant value on each small patch $[\alpha_1h,(\alpha_1+1)h)\times[\alpha_2h,(\alpha_2+1)h)$ for $\alpha_1=1,...,m$ and $\alpha_2=1,...,n$. 
	
	Starting with $\cT^1=\cT$, we consider a sequence of coarse grids $\{\cT^j\}_{j=1}^J$ so that $\cT^j$ has grid size $m_j\times n_j$ and grid step size $h_j$ with
	\begin{align*}
		&	m_j=2^{s_1-j+1},\quad n_j=2^{s_2-j+1}, \quad h_k=2^{j-1}h.
	\end{align*}
	A sequence of grids with $m=n=16,h=1$ are illustrated in the first row of Figure \ref{fig.grids}.
	\begin{figure}
		\centering
		\begin{tabular}{cccc}
			$\cT^1$ & $\cT^2$ & $\cT^3$ & $\cT^4$\\
			\includegraphics[width=0.22\textwidth]{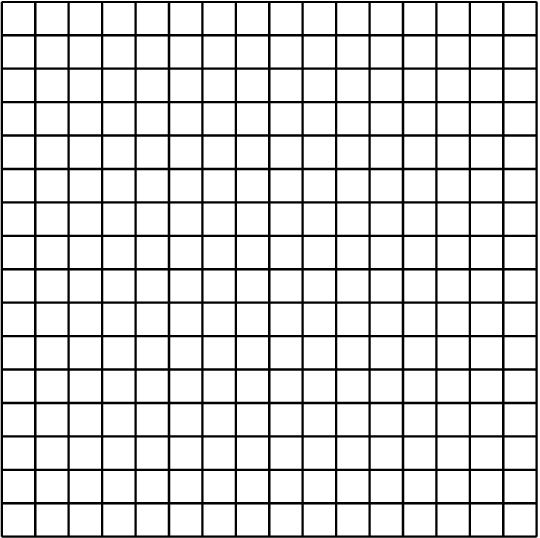} &
			\includegraphics[width=0.22\textwidth]{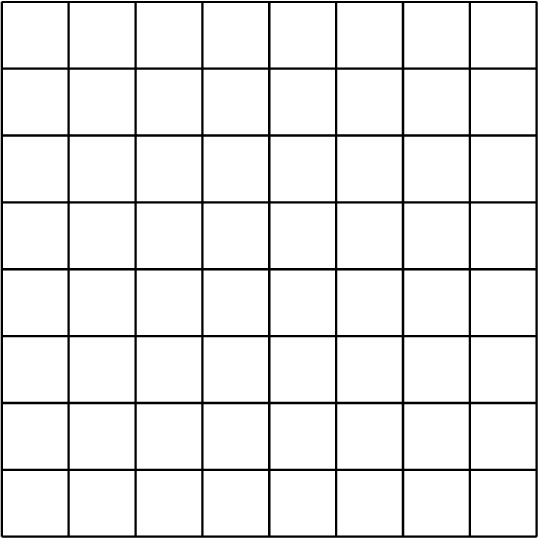} &
			\includegraphics[width=0.22\textwidth]{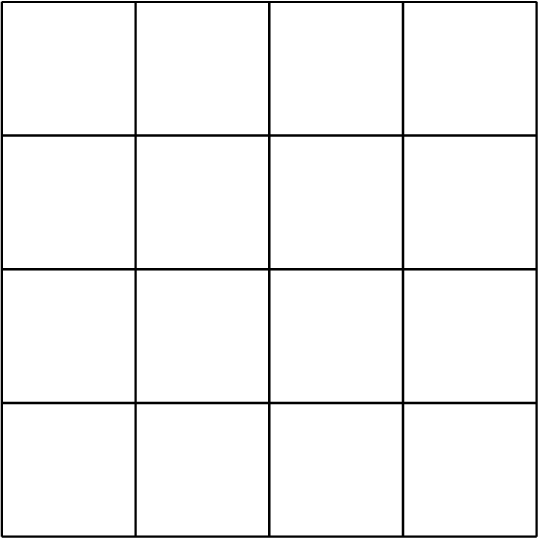} &
			\includegraphics[width=0.22\textwidth]{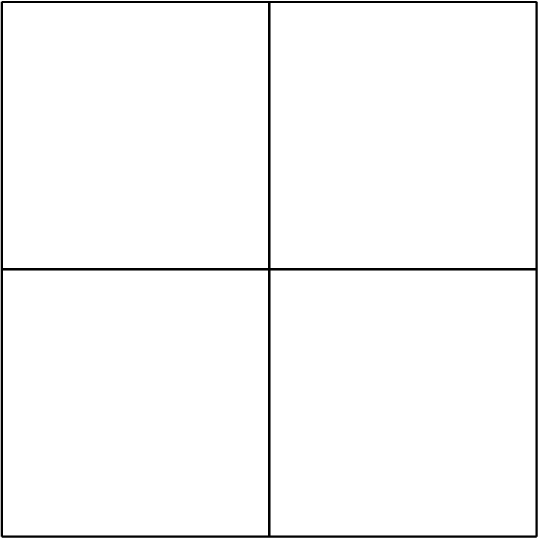}\\
			\includegraphics[trim={2cm 0 1.5cm 0},clip,width=0.22\textwidth]{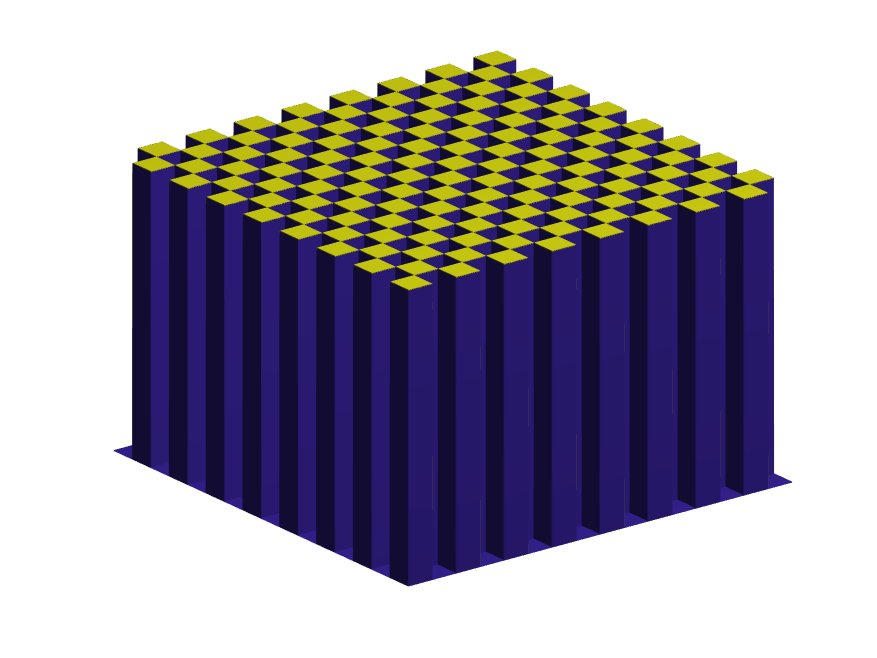} &
			\includegraphics[trim={2cm 0 1.5cm 0},clip,width=0.22\textwidth]{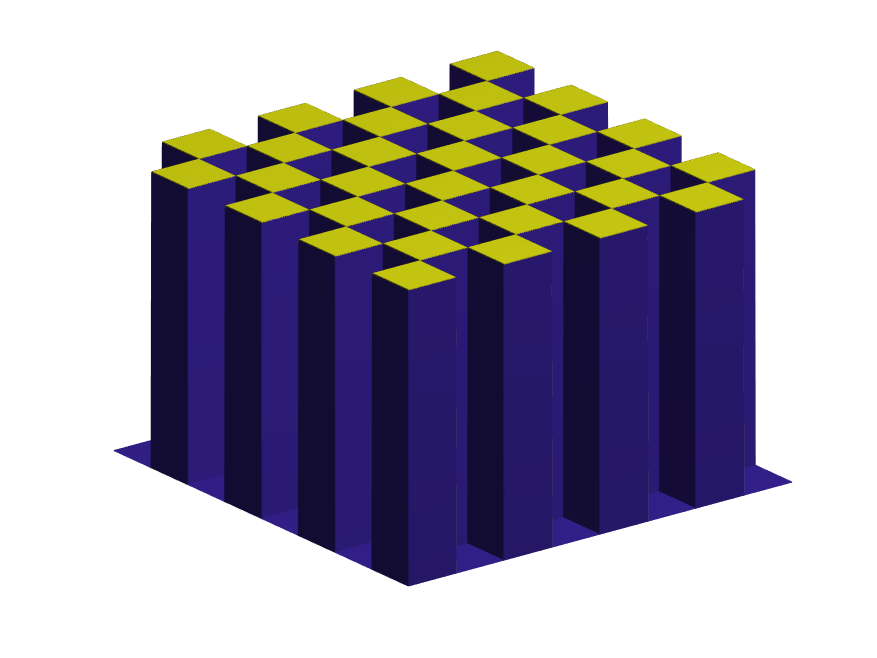} &
			\includegraphics[trim={2cm 0 1.5cm 0},clip,width=0.22\textwidth]{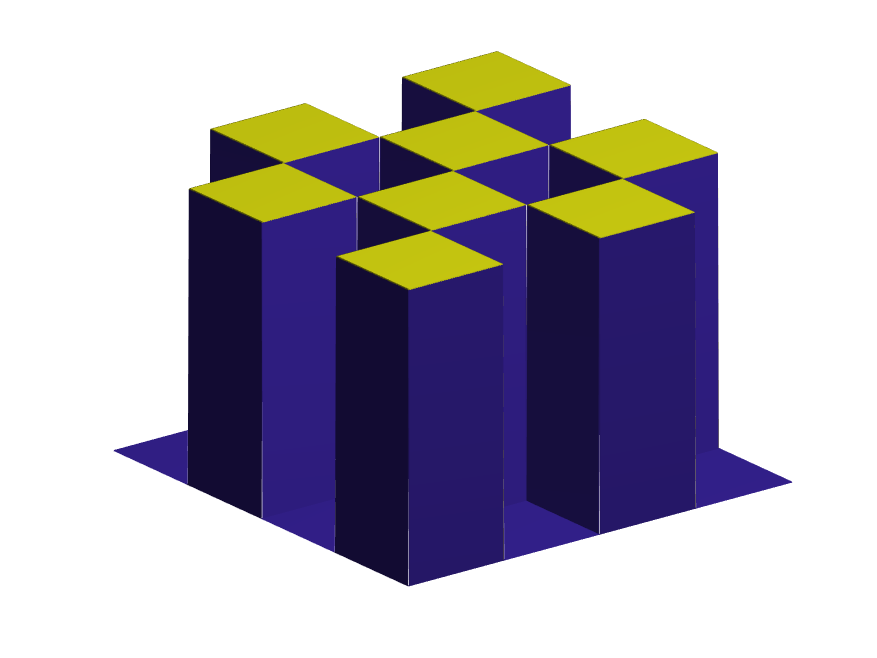} &
			\includegraphics[trim={2cm 0 1.5cm 0},clip,width=0.22\textwidth]{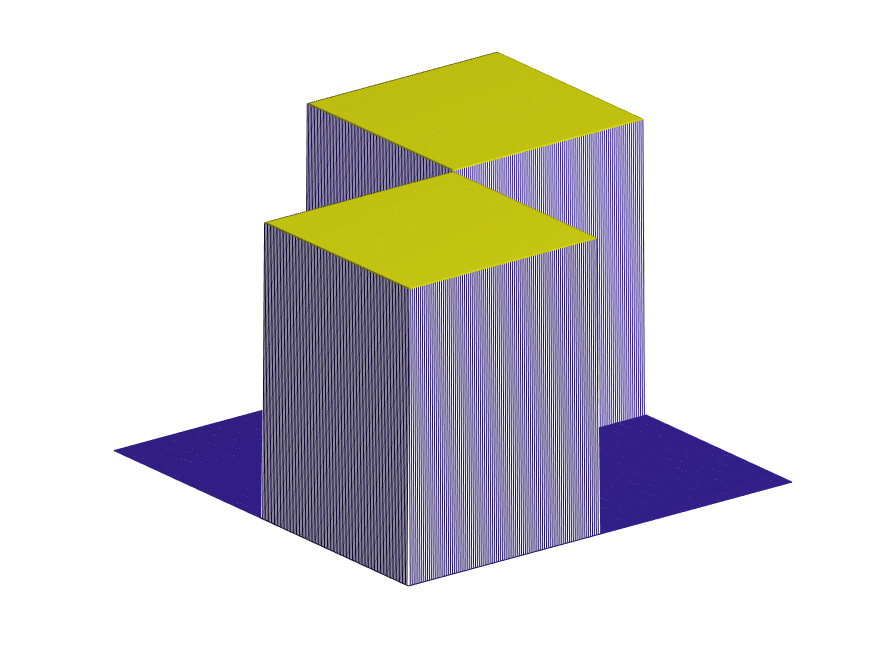}
		\end{tabular}
		\caption{An illustration of a sequence of grids with $J=4$. First row: visualization of grids. Second row: visualization of half of the basis functions for each grid.}
		\label{fig.grids}
	\end{figure}
	Denote 
	$
	\cI^j=\{ \balpha  : \balpha= (\alpha_1,\alpha_2), \alpha_1=1,...,m_j,\ \alpha_2=1,...,n_j\}. 
	$
	For a given grid $\cT^j$, we can define a set of piecewise-constant basis functions 
	$\{\phi^j_{\alpha} \}_{\alpha \in I^j}$ so that
	\begin{align}
		\phi^j_{\balpha}(x,y)=\begin{cases}
			1 & \mbox{ if } (x,y)\in [\alpha_1h_j,(\alpha_1+1)h_j)\times[\alpha_2h_j,(\alpha_2+1)h_j),\\
			0 & \mbox{ otherwise}.
		\end{cases}
		\label{eq.basis}
	\end{align}
	In the second row of Figure \ref{fig.grids}, we illustrate the grids with $m=m=16, J=4$ and half of the basis functions for each grid. Let $\cV^j ={\rm span}(\{\phi^j_{\alpha} \}_{\alpha \in I^j})$ be the linear space containing all the piecewise constant functions over grid $\cT^j$, then 
	we have
	\begin{align}
		\cV^1\supset \cV^2 \supset \cdots \supset \cV^J. 
		\label{eq.fspace.relation}
	\end{align}
	For each $f\in \cV^j$, it can be expressed as
	$f(x,y)=\sum_{\balpha\in \cI^j} f_{\balpha}^j\phi_{\balpha}^j(x,y)$
	with $f_{\balpha}^j=f(\alpha_1h_j,\alpha_2h_j)$.
	
	Let $\cT^j$ and $\cT^{j+1}$ be two grids. We next discuss the downsampling and upsampling operations. Consider $f^{j+1}\in \cV^{j+1}$.  According to (\ref{eq.fspace.relation}), there exists a function $f^{j}\in \cV^{j}$ satisfying $f^{j}=f^{j+1}$. We denote the upsampling operator $\cU^{j+1}:\cV^{j+1} \rightarrow \cV^{j} $ so that
	\begin{align}
		f^j=\cU^{j+1}(f^{j+1}).
	\end{align}
	It is easy to see that for $\balpha\in \cI^j$, we have
	\begin{align}
		(\cU^j(f^j))_{\balpha}=f^{j+1}_{\balpha'} \mbox{ with } \alpha_1',\alpha_2' \mbox{ satisfying } 2\alpha_1'-1\leq \alpha_1 \leq 2\alpha_1', \  2\alpha_2'-1\leq \alpha_2 \leq 2\alpha_2'.
		\label{eq.upsampling}
	\end{align}

	Given a function $f^{j}\in \cV^j$, there are many ways to define a downsampling operator $\cD^j: \cV^{j} \rightarrow \cV^{j+1}$. For example, we can define $\cD^j$ as an averaging downsampling operator:
	\begin{align}
		f^{j+1}=(\cU^j(f^{j}))_{\balpha}= \frac{1}{4}\sum_{\alpha_1'=2\alpha_1-1}^{2\alpha_1} \sum_{\alpha_2'=2\alpha_2-1}^{2\alpha_2} f^j_{\alpha_1',\alpha_2'}.
		\label{eq.downsampling.ave}
	\end{align} 
	Another choice is the max pooling operator which is widely used in deep learning:
	\begin{align}
		f^{j+1}=(\cU^k(f^{j}))_{\balpha}= \max_{\substack{\alpha_1'=2\alpha_1-1,2\alpha_1\\ \alpha_2'=2\alpha_2-1,2\alpha_2}} f^j_{\alpha_1',\alpha_2'}.
		\label{eq.downsampling.max}
	\end{align}
	
	\section{Parallel and sequential splitting scheme}
	\label{sec.splitting}
	Operator splitting methods have been widely used for scientific computing and many other applications. Among the rich literature on it, we could point to the following recent publications and surveys  
	\cite{glowinski1989augmented,glowinski2017splitting,glowinski2016some}. 
	Consider an initial value problem 
	\begin{align}
		\begin{cases}
			u_t+\sum_{m=1}^M (A_m(\xb,t;u) +S_m(\xb,t;u)+f_m(\xb,t))=0 \mbox{ on } \Omega\times [0,T],\\
			u(0)=u_0,
		\end{cases}
		\label{eq.general}
	\end{align}
	where $A_m(\xb,t;u)$ and $S_m(u, \xb, t)$ are operators (linear or nonlinear), and $f_m(\xb,t)$ are functions independent of $u$.  $\Omega$ is a given domain and $T$ is a given time.  In the following, we omit $\xb$ for the simplicity of notation.  We discuss in this section parallel and sequential splitting schemes to solve (\ref{eq.general}). As usual, we denote $\Delta t = T/N, t^n = n \Delta t$ for $n=0,1,\cdots N$ with a given step number $N$.

	\subsection{Parallel splitting schemes}
	
	The authors in  \cite{lu1992parallel} proposed a parallel splitting scheme to solve (\ref{eq.general}):
	\begin{align}
		\begin{cases}
			u^0=u_0,\\
			\mbox{for } n\geq0, \mbox{we compute } u^{n+1} \mbox{ from } u^n \mbox{ by solving first}\\
			\displaystyle \frac{u^{n}_m-u^n}{M\Delta t} + A_m(t^{n+1};u^n)+S_m(t^{n+1};u^{n}_m)+f_m(t^{n+1})=0, \mbox{ for } m=1,...,M,\\
			\displaystyle u^{n+1} \mbox{ is computed by averaging }
			u^{n+1}=\frac{1}{M} \sum_{j=1}^J u^{n}_m.
		\end{cases}
		\label{eq.general.parallel}
	\end{align}
	See also \cite[Sec.2.8]{glowinski2016some} for some more explanations about this scheme. 
	It has been proved that (\ref{eq.general.parallel}) is $O(\Delta t)$ accurate when $A$ and $B$ are linear operators. In (\ref{eq.general.parallel}), all operators are treated implicitly. If the operators $A_m$  are treated explicitly, the parallel splitting scheme becomes
	\begin{align}
		\begin{cases}
			u^0=u_0,\\
			\mbox{for } n\geq0, \mbox{we compute } u^{n+1} \mbox{ from } u^n \mbox{ by solving first}\\
			\displaystyle \frac{u^{n}_m-u^n}{M\Delta t} + A_m(t^{n};u^{n})+S_m(t^{n+1};u^{n}_m)+f_m(t^{n+1})=0, \mbox{ for } m=1,...,M,\\
			\displaystyle u^{n+1} \mbox{ is computed by averaging }
			u^{n+1}=\frac{1}{M} \sum_{m=1}^M u^{n}_m.
		\end{cases}
		\label{eq.general.parallel.1}
	\end{align}
	The structure of (\ref{eq.general.parallel.1}) is illustrated in Figure \ref{fig.parasequen}(a).
	
	\begin{figure}
		\centering
		(a)\\
		\includegraphics[width=0.8\textwidth]{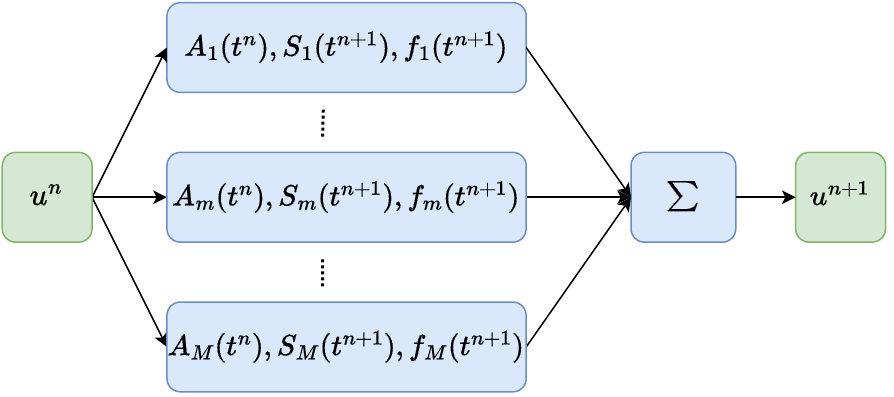}\\
		(b)\\
		\includegraphics[width=0.9\textwidth]{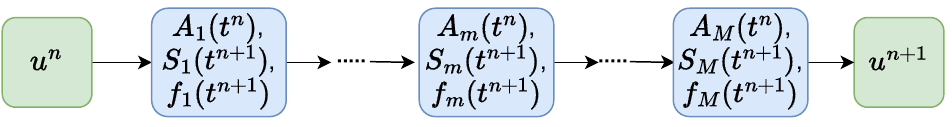}
		\caption{Illustrations of the corresponding network architecture for (a) parallel splitting and (b) sequential splitting.}
		\label{fig.parasequen}
	\end{figure}
	\subsection{Sequential splitting schemes}
	
	As a sequential splitting scheme, the Marchuk-Yanenko scheme \cite{marchuk1990splitting, glowinski2003finite} solving (\ref{eq.general}) reads as:
	\begin{align}
		\begin{cases}
			u^0=u_0,\\
			\mbox{for } n\geq0, \mbox{we compute } u^{n+1} \mbox{ from } u^n \mbox{ by solving }\\
			\displaystyle \frac{u^{n+m/M}-u^{n+(m-1)/M}}{\Delta t} + A_m(t^{n+1};u^{n+m/M})+S_m(t^{n+1};u^{n+m/M})+f_m(t^{n+1})=0, \\
			\mbox{for } m=1,...,M. 
		\end{cases}
		\label{eq.general.sequential}
	\end{align}
	See also \cite[Sec.2.2]{glowinski2016some} for some more explanations about this scheme. 
	One can show that this scheme is also $O(\Delta t)$ accurate. In (\ref{eq.general.sequential}), all operators are treated implicitly. If the operators $A_m$ are treated explicitly, the sequential splitting scheme becomes
	\begin{align}
		\begin{cases}
			u^0=u_0,\\
			\mbox{for } n\geq0, \mbox{we compute } u^{n+1} \mbox{ from } u^n \mbox{ by solving }\\
			\displaystyle \frac{u^{n+m/M}-u^{n+(m-1)/M}}{\Delta t} + A_m(t^{n};u^{n+(m-1)/M})+S_m(t^{n+1};u^{n+m/M})+f_m(t^{n+1})=0, \\
			\mbox{for } m=1,...,M.
		\end{cases}
		\label{eq.general.sequential.1}
	\end{align}
	The structure of (\ref{eq.general.sequential.1}) is illustrated in Figure \ref{fig.parasequen}(b).
	
	\section{A  hybrid splitting scheme to solve initial value problems}
	\label{sec.hybrid}
	Algorithm \ref{alg.V.full} uses a hybrid splitting scheme to split the control problem (\ref{eq.control.full}) into several subproblems. In this section, we introduce the hybrid splitting scheme and give an error analysis.
	
	\subsection{A hybrid splitting scheme}
	Consider the following initial value problem
	\begin{align}
		\begin{cases}
			u_t+\displaystyle\sum_{m=1}^M \left(\sum_{k=1}^{c_m} \sum_{s=1}^{d_{m}} A_{k,s}^m(\xb,t;u) +\sum_{k=1}^{c_m} S_{k}^m(\xb,t;u) +\sum_{k=1}^{c_m}	f_{k}^m(\xb,t)\right)=0 \mbox{ on } \Omega\times [0,T],\\
			u(0)=u_0.
		\end{cases}
		\label{eq.general.hybrid}
	\end{align}
	Above $\{c_m\}_{m=1}^M, \{d_{m}\}_{m=1}^{M}$ are some known positive integers and $A_{k,s}^m, S_k^m$ are operators. As usual, $f_k^m$ is used to denote functions independent of $u$. 
	We propose a hybrid splitting scheme, which is a mixture of the parallel and the sequential splitting schemes, to solve (\ref{eq.general.hybrid}). In our splitting, we split all operators into $M$ sequential substeps, each of which consists $c_m$ parallel splittings. The computation of each parallel splitting uses $d_m$ intermediate results from the previous substep, which requires the condition $d_{m}\leq c_{m-1}$. The algorithm is summarized in Algorithm \ref{alg.hybrid}.  The operators can be treated implicitly or explicitly. In (\ref{eq.hybrid.k}), we purposely treat $A_{k,s}^m$ explicitly and $S_k^m$
	implicitly.
	
	\begin{algorithm}[th!]
		\caption{A hybrid splitting scheme}\label{alg.hybrid}
		\begin{algorithmic}
			\STATE {\bf Data:} The solution $u^n$ at time step $t^n$.
			\STATE {\bf Result:} The computed solution $u^{n+1}$ at time step $t^{n+1}$.
			
			{\bf Set} $d_{1} = 1,u_1^{n} = u^n$. \\
			\FOR {$m=1,...,M$}
			\FOR {$k=1,...,c_{j}$}
			\STATE Compute $u^{n+m/M}_k$ by solving
			\begin{align}
				\frac {u^{n+m/M}_k - u^{n+(m-1)/M}} {c_m \Delta t } =   
				-\sum_{s=1}^{d_m} 
				A_{k,s}^{m}( t^n;u_s^{n+(m-1)/M})  - S_k^m(t^{n+1};u^{n+m/M})-f_k^m(t^n).
				\label{eq.hybrid.k}
			\end{align}
			
			\ENDFOR
			
			Compute $u^{n+m/M}$ as
			\begin{equation}
				u^{n+m/M} = \frac 1 {c_m} \sum_{k=1}^{c_m} u_k^{n+m/M}. 
				\label{eq.hybrid.ave}
			\end{equation}
			\ENDFOR
		\end{algorithmic}
	\end{algorithm}
	
	The following theorem shows that Algorithm \ref{alg.hybrid}  converges with first order at least when the split operators are all linear. 
	\begin{theorem}\label{thm.hybrid}
		For a fixed $T>0$ and a positive integer $N$, set $\Delta t=T/N$. Let $u^{n+1}$ be the numerical solution by Algorithm \ref{alg.hybrid}. Assume $A_{k,s}^m$'s and $S_k^m$'s are Lipschitz with respect to $t,\xb$, and are linear symmetric positive definite operators with respect to $u$. Assume $\Delta t$ is small enough (i.e., $N$ is large enough). We have 
		\begin{align}
			\|u^{n+1}-u(t^{n+1})\|_{\infty}=O(\Delta t)
		\end{align}
		for any $0\leq n\leq N$.
	\end{theorem}
	Theorem \ref{thm.hybrid} is proved in Appendix \ref{sec.thm.hybrid.proof}.
	\subsection{A more general hybrid splitting scheme with relaxation}
	
	When applying operator splitting schemes to neural network constructions, we need to handle a more general dynamical system in the following form:
	\begin{align}
		\begin{cases}
			u_t+\displaystyle \sum_{j=1}^{2J-1} \sum_{m=1}^{M_j} \left(\sum_{k=1}^{c_{j,m}} \sum_{s=1}^{d_{j,m}} A_{k,s}^{j,m}(\xb,t;u) +\sum_{k=1}^{c_{j,m}} S_{k}^{j,m}(\xb,t;u) +\sum_{k=1}^{c_{j,m}}	f_{k}^{j,m}(\xb,t)\right)\\
			\hspace{5cm} +\displaystyle\sum_{s=1}^{d_{2J}} A^*_s(\xb,t;u)+S^*(\xb,t;u)+f^*(\xb,t)=0 \mbox{ on } \Omega\times [0,T], \\
			u(0)=u_0,
		\end{cases}
		\label{eq.general.hybrid.relax}
	\end{align}
	
	The positive integers $M, J_m, c_{j,m}, d_{j,m}$, the operators $A_{k,s}^{j,m}(t), S_k^{j,m}$ and functions $f_k^{j,m}(t)$ are  supposed to be given. Suppose $A_{k,s}^{j,m}(t), S_k^{j,m}(t)$ and $f_k^{j,m}(t)$ are time dependent. We assume $d_{j,m}\leq c_{j,m-1}$ for $1\leq m\leq M_j$, and $d_{2J}\leq c_{2J-1,M_{2J-1}}$, where $c_{j,0}=c_{j-1,M_{j-1}}$ is used. We also assume $c_{j,M_j}=c_{2J-j,M_{2J-j}}$. These conditions are needed to make the algorithm meaningful.  In (\ref{eq.general.hybrid.relax}), the operators $A_s^*,S^*$ and $f^*$ can be absorbed by the second term. Here we write them down explicitly so that the operators in (\ref{eq.general.hybrid.relax}) have the same form as the decomposed control variables in Section \ref{sec.FullAlg}. We will show that Algorithm \ref{alg.V.full} is a special case of the general hybrid splitting scheme discussed in this section.
	
	The number of operations in (\ref{eq.general.hybrid.relax}) is about $2J-1$ times of that in (\ref{eq.general.hybrid}).
	A simple way to solve (\ref{eq.general.hybrid.relax}) is to apply Algorithm \ref{alg.hybrid} for $(2J-1)$ times. In many numerical methods for PDEs, relaxation steps may make the method more stable. We will propose a new hybrid splitting scheme that incorporates relaxation steps. Our general idea is to split all operators into $2J$ sequential parts according to the index $j$. For the first $2J-1$ part, the $j$-th part consists of $M_j$ sequential substeps. Algorithm \ref{alg.hybrid} will be used $2J-1$ times for these parts. The $2J$-th part only has one substep. There are infinitely many ways to conduct the relaxations. In this section, we discuss one way as an example. Other ways can be implemented similarly. In our algorithm, for $j=J+1,...,2J-1$, we will pass the intermediate variable from part $2J-j$ to part $j$. We simply use averaging for these relaxations, which uses the condition  $c_{j,M_j}=c_{2J-j,m_{2J-j}}$. The new scheme is summarized in Algorithm \ref{alg.hybrid.general}.
	
	\begin{algorithm}
		\caption{A general hybrid splitting algorithm}\label{alg.hybrid.general}
		\begin{algorithmic}
			\STATE {\bf Data:} The solution $u^n$ at time $t^n$.
			\STATE {\bf Result:} The computed solution $u^{n+1}$ at time step $t^{n+1}$.
			
			{\bf Set} $d_{1,1} = 1, u_1^n=u^n$. \\
			\FOR {$j = 1,..., J$}
			\STATE Set $u^{n,j,0}=u^{n,j-1,M_{j-1}}$, $u^{n,j,0}_k=u^{n,j-1,M_{j-1}}_k$ for $k=1,...,c_{j-1,M_{j-1}}$.\\
			\FOR {$m=1,...,M_j$}
			\FOR {$k=1,...,c_{j,m}$}
			\STATE Compute $u^{n,j,m}_k$ by solving
			\begin{align}
				\frac {u^{n,j,m}_k - u^{n,j,m-1}} {2^{j-1}c_{j,m} \Delta t } =   
				-\sum_{s=1}^{d_{j,m}} 
				A_{k,s}^{j,m}( t^n;u_s^{n,j,m-1})  - S_k^m(t^{n+1};u^{n,j,m})-f_k^{j,m}(t^n).
				\label{eq.hybrid.k1}
			\end{align}
			
			\ENDFOR
			
			Compute $u^{n,j,m}$ as
			\begin{equation}
				u^{n,j,m} = \frac 1 {c_{j,m}} \sum_{k=1}^{c_{j,m}} u_k^{n,j,m}. 
				\label{eq.hybrid.ave1}
			\end{equation}
			\ENDFOR
			
			\ENDFOR
			
			{\bf Set} $\bar{u}^{n,J,M_J}= u^{n,J,M_J}$ and $ \bar{u}_k^{n,J,M_J} = u_k^{n,J,M_J}$ for $k =1, 2, ..., c_{J,M_J}$.\\
			\FOR {$j = J+1,  \cdots, 2J-1$}
			\STATE Set $u^{n,j,0}=\bar{u}^{n,j-1,M_{j-1}}$ and $u_k^{n,j,0}=\bar{u}_k^{n,j-1,M_{j-1}}$, $k=1,...,c_{j-1,M_{j-1}}$. \\
			\FOR {$m=1,...,M_j$}
			\FOR {$k=1,2,\cdots c_{j,m}$}
			\STATE Compute $u_k^{n,j,m}$ by solving 
			\begin{align}
				\frac {u^{n,j,m}_k - u^{n,j,m-1}} {2^{j}c_{j,m} \Delta t } =   
				-\sum_{s=1}^{d_{j,m}} 
				A_{k,s}^{j,m}(t^n; u_s^{n,j,m-1})  - S_k^m(t^{n+1};u^{n,j,m})-f_k^{j,m}(t^n).
				\label{eq.hybrid.k2}
			\end{align}			
			\ENDFOR
			
			Compute $u^{n,j,m}$ as
			\begin{equation}
				u^{n,j,m} = \frac 1 {c_{j,m}} \sum_{k=1}^{c_{j,m}} u_k^{n,j,m}. 
				\label{eq.hybrid.ave2}
			\end{equation}
			
			\ENDFOR
			
			Compute $\bar{u}^{n,j,M_j}_k, \bar{u}^{n,j,M_j}$ as 
			\begin{align} 
				\bar{u}^{n,j,M_j}_k=\frac{1}{2} u^{n,j,M_j}_k+ \frac{1}{2} u^{n,2J-j,M_{2J-j}}_k, \quad
				\bar{u}^{n,j,M_j}=\frac{1}{c_{j,M_j}} \sum_{k=1}^{c_{j,M_j}} \bar{u}_k^{n,j,L_j}
				\label{eq.hybrid.relax}
			\end{align}
			

			\ENDFOR
			
			Compute $u^{n+1}$ by solving
			\begin{align}
				\frac {u^{n+1} - \bar{u}^{n,2J-1,M_{2J-1}}} {\Delta t } =   
				- \sum_{s=1}^{d_{2J}} 
				A_{s}^{*}(t^n;u_s^{n,2J-1,M_{2J-1}})   -   S^*(t^{n+1};u^{n+1})-f^*(t^n) .
				\label{eq.hybrid.final}
			\end{align}
		\end{algorithmic}
	\end{algorithm}
	
	The following theorem gives an error estimation of Algorithm \ref{alg.hybrid.general} when  $S_k^{j,m}$'s and $S^*$ are linear operators.
	
	\begin{theorem}\label{thm.full}
		Assume $S_{k}^{j,m}$'s and $S^*$ are linear operators, $A_{k,s}^{j,m}, A_s^*, S_k^{j,m},  S^*$ are Lipschitz in $\xb, t$, and are symmetric positive definite with respect to $u$. 
		From an initial condition $u(0)=u_0$, we use Algorithm \ref{alg.hybrid.general} to solve (\ref{eq.general.hybrid.relax}) until time $t=T$ with time step $\Delta t$ so that $T=N\Delta t$ for some integer $N>0$. Denote the numerical solution at $t^n$ by $u^n$. We have
		\begin{align}
			\|u^{n+1}-u(t^{n+1})\|_{\infty}=O(\Delta t)
		\end{align}
		for any $0\leq n\leq N$.
	\end{theorem}
	
	Theorem \ref{thm.full} is proved in Appendix \ref{sec.thm.full.proof} which shows that under certain conditions, Algorithm \ref{alg.hybrid.general} is a first order numerical scheme in solving the control problem (\ref{eq.general.hybrid.relax}).
	
	\section{Proof of theorems}
	\subsection{Proof of Theorem \ref{thm.hybrid}}
	\label{sec.thm.hybrid.proof}
	\begin{proof}[Proof of Theorem \ref{thm.hybrid}]
		In this proof, we focus on the case $d_m=c_{m-1}$. Other cases can be proved similarly. Due to linearity, 
		we denote $A_{k,s}^m(t^n;u), S_k^m(t^n;u)$ by $A_{k,s}^m(t^n)u, S_k^m(t^n) u$.
		Denote $\widetilde{u}^{n+1}$  the solution of (\ref{eq.general.hybrid}) at $t=t^{n+1}$ using $u^n$ as initial condition at $t=t^n$. Denote 
		\begin{align}
			W(t)u=\sum_{m=1}^M \sum_{k=1}^{c_m}\left( \sum_{s=1}^{c_{m-1}} A_{k,s}^m(t) u + S_{k}^m(t) u \right)
		\end{align}
		and $\cI$ as the identity operator (matrix).
        We further denote 
        \begin{align}
            f(t)=\sum_{m=1}^M\sum_{k=1}^{c_m} f_k^m(\xb,t).
            \label{eq.hybrid.f}
        \end{align}
  
		We have
		\begin{align}
			\widetilde{u}^{n+1}=&\exp\left(\int_{t^n}^{t^{n+1}}- W(\tau)d\tau \right)u^n \nonumber \\
			&-\exp\left(\int_{t^n}^{t^{n+1}} -W(\tau)d\tau \right)\int_{t^n}^{t^{n+1}} \exp\left(\int_{t^n}^{\tau} W(\theta)d\theta \right)f(\tau) d\tau 
			\label{eq.hybrid.tildeu}
		\end{align}
		Using Taylor expansion, we have
		\begin{align}
			&\int_{t^n}^{t^{n+1}} \exp\left(\int_{t^n}^{\tau} W(\theta)d\theta \right)f(\tau) d\tau \nonumber\\
			=& \Delta t \exp\left(\int_{t^n}^{t^{n}} W(\theta)d\theta \right)f(t^n)+O(\Delta t^2) \nonumber\\
			=& \Delta t f(t^n)+O(\Delta t^2),
			\label{eq.hybrid.1}
		\end{align}
		and 
		\begin{align}
			\exp\left(\int_{t^n}^{t^{n+1}} -W(t)d\tau \right)= &1-\int_{t^n}^{t^{n+1}} W(t)d\tau +O(\Delta t^2) \nonumber\\
			=&1-\Delta t \widetilde{W}(t^n)+O(\Delta t^2)
			\label{eq.hybrid.2}
		\end{align}
		with 
		\begin{align}
			\widetilde{W}(t^n)=\sum_{m=1}^M \sum_{k=1}^{c_m}\left( \sum_{s=1}^{c_{m-1}}A_{k,s}^m(t^n) + S_{k}^m(t^{n+1})  \right).
   \label{eq.hybrid.wtilde}
		\end{align}
		Substituting (\ref{eq.hybrid.1}) and (\ref{eq.hybrid.2}) into (\ref{eq.hybrid.tildeu}) gives rise to
		\begin{align}
			\widetilde{u}^{n+1}=&u^n-\Delta t \widetilde{W}(t^n)u^n-\Delta t f(t^n)+O(\Delta t^2).
			\label{eq.hybrid.tilde}
		\end{align}

		
		We next focus on $u^{n+1}$. We show that for any $1\leq m^*\leq M$, we have
		\begin{align}
			u^{n+m^*/M}_{k^*}=& u^n-\Delta t \Bigg( \sum_{m=1}^{m^*-1} \sum_{k=1}^{c_{m}}\left( \left( \sum_{s=1}^{c_{m-1}} A^m_{k,s}(t^n)+S_k^m(t^{n+1})\right) u^n+f_k^m (t^{n})\right) \nonumber\\
			&\quad + c_{m^*}\left(\sum_{s=1}^{c_{m^*-1}} \left(A^m_{k^*,s}(t^n)+S_{k^*}^m(t^{n+1}) \right)u^n+f_{k^*}^m(t^{n})\right) \Bigg)+ O(\Delta t^2),
			\label{eq.hybrid.proof.induction1}\\
			u^{n+m^*/M}=& u^n-\Delta t \sum_{m=1}^{m^*} \sum_{k=1}^{c_{m}}\left( \left( \sum_{s=1}^{c_{m-1}} A^m_{k,s}(t^n)+S_k^m(t^{n+1})\right) u^n+f_k^m(t^{n}) \right) + O(\Delta t^2),
			\label{eq.hybrid.proof.induction2}
		\end{align}
		for any $1\leq k^* \leq c_{m^*}$. We will prove (\ref{eq.hybrid.proof.induction1})--(\ref{eq.hybrid.proof.induction2}) by mathematical induction. For $m^*=1$, we have $c_0=1$.
		According to (\ref{eq.hybrid.k}) and (\ref{eq.hybrid.ave}), we have
		\begin{align}
			u^{n+1/M}_k=&(\cI+\Delta t c_1S_k^m(t^{n+1}))^{-1} \left(u^n- \Delta t c_1  A^1_{k,1}(t^n) u^n-\Delta t c_1f_k^1(t^{n})\right) \nonumber\\
			=& (\cI-\Delta t c_1 S_k^m(t^{n+1}))\left(u^n- \Delta t c_1 A^1_{k,1}(t^{n}) u^n-\Delta t c_1 f_k^1(t^{n})\right) +O(\Delta t^2) \nonumber\\
			=&u^n- \Delta t c_1 \left( A^1_{k,1}(t^{n})+S_k^m(t^{n+1})\right)  u^n-\Delta t c_1 f_k^1(t^{n}) +O(\Delta t^2),
			\label{eq.hybrid.proof.induction.1.1}
		\end{align}
		and
		\begin{align}
			u^{n+1/M}=&\frac{1}{c_1} \sum_{k=1}^{c_1} \left( u^n- \Delta t c_1 \left( A^1_{k,1}(t^{n})+S_k^m(t^{n+1})\right) u^n-\Delta t c_1 f_k^1(t^{n}) \right) +O(\Delta t^2) \nonumber\\
			=& u^n-\Delta t \sum_{k=1}^{c_1} \left( \left( A^1_{k,1}(t^{n})+S_k^1(t^{n+1})\right) u^n+f_k^1(t^{n}) \right) + O(\Delta t^2).
			\label{eq.hybrid.proof.induction.1.2}
		\end{align}
		Thus equation (\ref{eq.hybrid.proof.induction1})--(\ref{eq.hybrid.proof.induction2}) hold. 
		
		Assume (\ref{eq.hybrid.proof.induction1})--(\ref{eq.hybrid.proof.induction2}) hold for $m^*=m_1<M$, i.e.,
		\begin{align}
			u^{n+m_1/M}_{k^*}=& u^n-\Delta t \Bigg( \sum_{m=1}^{m_1-1} \sum_{k=1}^{c_{m}}\left( \left( \sum_{s=1}^{c_{m-1}} A^m_{k,s}(t^{n})+S_k^m(t^{n+1})\right) u^n+f_k^m(t^{n}) \right) \nonumber\\
			&\quad + c_{m_1}\left( \sum_{s=1}^{c_{m_1-1}} \left(A^m_{k^*,s}(t^{n})+S_{k^*}^m(t^{n+1}) \right)u^n+f_{k^*}^m (t^{n})\right)\Bigg)+ O(\Delta t^2),
			\label{eq.hybrid.proof.induction.2.1}\\
			u^{n+m_1/M}=& u^n-\Delta t \sum_{m=1}^{m_1} \sum_{k=1}^{c_{m}}\left( \left( \sum_{s=1}^{c_{m-1}} A^m_{k,s}(t^{n})+S_k^m(t^{n+1})\right) u^n+f_k^m(t^{n}) \right) + O(\Delta t^2).
			\label{eq.hybrid.proof.induction.2.2}
		\end{align}
		for any  $1\leq k^* \leq c_{m_1}$. 	When $m^*=m_1+1$, we have
		\begin{align}
			&u^{n+(m_1+1)/M}_{k^*} \nonumber\\
			=&\left(\cI+\Delta t c_{m_1+1}S_{k^*}^{m_1+1}(t^{n+1})\right)^{-1} \Bigg(u^{n+m_1/M}- \Delta t c_{m_1+1}  \sum_{s=1}^{m_1} A^{m_1+1}_{k^*,s} (t^{n})u^{n+m_1/M}_s \nonumber\\
			&\hspace{8cm} -\Delta t c_{m_1+1} f_{k^*}^{m_1+1}(t^{n})\Bigg) \nonumber\\
			=& \left(\cI-\Delta t c_{m_1+1} S_{k^*}^{m_1+1}(t^{n+1})\right)\Bigg(u^{n+m_1/M}- \Delta t c_{m_1+1} \sum_{s=1}^{m_1} A^{m_1+1}_{k^*,s}(t^{n})u^{n+m_1/M}_s \nonumber\\
			&\hspace{8cm}-\Delta t c_{m_1+1} f_{k^*}^{m_1+1}(t^{n})\Bigg) +O(\Delta t^2) \nonumber\\
			=&u^{n+m_1/M}- \Delta t c_{m_1+1} \sum_{s=1}^{c_{m_1}} A^{m_1+1}_{k^*,s} (t^{n})u^{n+m_1/M}_s- \Delta t c_{m_1+1} S_{k^*}^{m_1+1}(t^{n+1})u^{n+m_1/M} \nonumber\\
			&\hspace{8cm}-\Delta t c_{m_1+1} f_{k^*}^{m_1+1}(t^{n}) +O(\Delta t^2).
			\label{eq.hybrid.proof.induction.3}
		\end{align}
		Substituting (\ref{eq.hybrid.proof.induction.2.1})--(\ref{eq.hybrid.proof.induction.2.2}) into (\ref{eq.hybrid.proof.induction.3}) gives rise to
		\begin{align}
			u^{n+(m_1+1)/M}_{k^*}=&u^n-\Delta t \sum_{m=1}^{m_1} \sum_{k=1}^{c_{m}}\left( \left( \sum_{s=1}^{c_{m-1}} A^m_{k,s}(t^{n})+S_k^m(t^{n+1})\right)(t^{n+1}) u^n+f_k^m(t^n) \right) \nonumber \\
			& - \Delta t c_{m_1+1} \sum_{s=1}^{c_{m_1}} A^{m_1+1}_{k^*,s} (t^{n})u^{n+m_1/M}_s - \Delta t c_{m_1+1} S_{k^*}^{m_1+1}(t^{n+1})u^n-\Delta t c_{m_1+1} f_{k^*}^{m_1+1}(t^{n}).
			\label{eq.hybrid.proof.induction.4.1}
		\end{align}
		According to (\ref{eq.hybrid.ave}), we have
		\begin{align}
			u^{n+(m_1+1)/M}=&u^n-\Delta t \sum_{m=1}^{m_1} \sum_{k=1}^{c_{m}}\left( \left( \sum_{s=1}^{c_{m-1}} A^m_{k,s}(t^{n})+S_k^m\right)(t^{n+1}) u^n+f_k^m (t^{n})\right) \nonumber\\
			& - \Delta t \sum_{k=1}^{c_{m_1+1}} \Bigg(\sum_{s=1}^{c_{m_1}} A^{m_1+1}_{k,s}(t^{n}) u^{n+m_1/M}_s -   S_k^{m_1+1} (t^{n+1})u^n \nonumber\\
			&\hspace{5cm} -  f_k^{m_1+1}(t^{n+1})\Bigg) + O(\Delta t^2) \nonumber\\
			=& u^n-\Delta t \sum_{m=1}^{m_1+1} \sum_{k=1}^{c_{m+1}}\left( \left( \sum_{s=1}^{c_{m}} A^m_{k,s}(t^{n})+S_k^m(t^{n+1})\right) u^n+f_k^m (t^{n})\right) + O(\Delta t^2).
			\label{eq.hybrid.proof.induction.4.2}
		\end{align}
		Therefore, equation (\ref{eq.hybrid.proof.induction1})--(\ref{eq.hybrid.proof.induction2}) holds for $m=m_1+1$. Combining (\ref{eq.hybrid.proof.induction.1.1})--(\ref{eq.hybrid.proof.induction.1.2}) and (\ref{eq.hybrid.proof.induction.4.1})--(\ref{eq.hybrid.proof.induction.4.2}), we have that (\ref{eq.hybrid.proof.induction1})--(\ref{eq.hybrid.proof.induction2}) holds for any $1\leq m\leq M$.
		
		Setting $m=M$ gives rise to
		\begin{align}
			u^{n+1}=& u^n-\Delta t \sum_{m=1}^{M} \sum_{k=1}^{c_{m}}\left( \left( \sum_{s=1}^{c_{m-1}} A^m_{k,s}(t^{n})+S_k^m(t^{n+1})\right) u^n+f_k^m (t^{n})\right) + O(\Delta t^2) \nonumber\\
			=&u^n-\Delta t \widetilde{W}(t^n)u^n -\Delta t\sum_{m=1}^{M} \sum_{k=1}^{c_{m}}f_k^m(t^{n})+ O(\Delta t^2),
			\label{eq.hybrid.proof.induction.5}
		\end{align}
  where $\widetilde{W}(t^n)$ is defined in (\ref{eq.hybrid.wtilde}).
  
		Comparing (\ref{eq.hybrid.proof.induction.5}) with (\ref{eq.hybrid.tilde}), we have the local error
		\begin{align}
			u^{n+1}-\tilde{u}^{n+1}=O(\Delta t^2).
		\end{align}
		Therefore, the global error is $O(\Delta t)$.
	\end{proof}
	
	\subsection{Proof of Theorem \ref{thm.full}}
	\label{sec.thm.full.proof}
	\begin{proof}[Proof of Theorem \ref{thm.full}]
		We consider the case $d_{j,m}=c_{j,m-1}$ and $d_{2J}=c_{2J-1,M_{2j-1}}$. Other cases can be proved similarly. We denote $A_{k,s}^{j,m}(t;u), S_k^{j,m}(u)$ by $A_{k,s}^{j,m}(t)u, S_{k}^{j,m}u$.
		
		To simplify the notation, we denote 
		\begin{align*}
			&Z_{j}(t)=\sum_{m=1}^{M_j}\sum_{k=1}^{c_{j,m}} \left(\sum_{s=1}^{c_{j,m-1}}  A_{k,s}^{j,m}(t)+S_{k}^{j,m}(t)\right),\ Z_*(t)=\sum_{s=1}^{c_{2J-1,M_{2J-1}}} A^*_s(t)+S^*(t),\\
			&f_{j}(t)=\sum_{m=1}^{M_j}\sum_{k=1}^{c_{j,m}} f_k^{j,m}(t),\\
			&Z(t)=\sum_{j=1}^{2J-1} Z_{j}(t) +Z^*(t), \ f(t)=\sum_{j=1}^{2J-1} f_{j}(t)+f^*(t).
		\end{align*}
		and
		\begin{align}
			&\widetilde{Z}_{j}(t^n)=\sum_{m=1}^{M_j}\sum_{k=1}^{c_{j,m}} \left(\sum_{s=1}^{c_{j,m-1}}  A_{k,s}^{j,m}(t^n)+S_{k}^{j,m}(t^{n+1})\right),\ \widetilde{Z}^*(t^n)=\sum_{s=1}^{c_{2J-1,M_{2J-1}}} A^*_s(t^n)+S^*(t^{n+1}),\nonumber\\
			&\widetilde{Z}(t^n)=\sum_{j=1}^{2J-1} \widetilde{Z}_{j}(t^n)+\widetilde{Z}^*(t^n), \ f(t^n)=\sum_{j=1}^{2J-1} f_{j}(t^n)+f^*(t^n).
   \label{eq.full.f}
		\end{align}
		
		Denote $\widetilde{u}^{n+1}$ the solution of (\ref{eq.general.hybrid.relax}) at $t=t^{n+1}$ using $u^n$ as initial condition at $t=t^n$.  We have
		\begin{align}
			\widetilde{u}^{n+1}=&\exp\left(\int_{t^n}^{t^{n+1}}- Z(\tau)d\tau \right)u^n \nonumber \\
			&-\exp\left(\int_{t^n}^{t^{n+1}} -Z(\tau)d\tau \right)\int_{t^n}^{t^{n+1}} \exp\left(\int_{t^n}^{\tau} Z(\theta)d\theta \right)f(\tau) d\tau 
			\label{eq.full.tilde}
		\end{align}
		Using Taylor expansion, we have
		\begin{align}
			&\int_{t^n}^{t^{n+1}} \exp\left(\int_{t^n}^{\tau} Z(\theta)d\theta \right)f(\tau) d\tau \nonumber\\
			=& \Delta t \exp\left(\int_{t^n}^{t^{n}} Z(\theta)d\theta \right)f(t^n)+O(\Delta t^2) \nonumber\\
			=& \Delta t f(t^n)+O(\Delta t^2),
			\label{eq.full.1}
		\end{align}
		and 
		\begin{align}
			\exp\left(\int_{t^n}^{t^{n+1}} -Z(t)d\tau \right)= &1-\int_{t^n}^{t^{n+1}} Z(t)d\tau +O(\Delta t^2) \nonumber\\
			=&1-\Delta t \widetilde{Z}(t^n)+O(\Delta t^2).
			\label{eq.full.2}
		\end{align}
		Substituting (\ref{eq.full.1}) and (\ref{eq.full.2}) into (\ref{eq.full.tilde}) gives rise to
		\begin{align}
			\widetilde{u}^{n+1}=&u^n-\Delta t \widetilde{Z}(t^n)u^n-\Delta t f(t^n)+O(\Delta t^2).
			\label{eq.tildeu}
		\end{align}
		
		We then focus on $u^{n+1}$. For any $j\leq J$ in Algorithm \ref{alg.hybrid.general}, it is a standard hybrid splitting. From the proof of Theorem \ref{thm.hybrid}, in particular equation (\ref{eq.hybrid.proof.induction1})--(\ref{eq.hybrid.proof.induction2}), for any $1\leq j^* \leq J, 1\leq k^*\leq c_{j^*,M_{j^*}}$, we have
		\begin{align}
			u^{n,j^*,M_{j^*}}_{k^*}=& u^n-\Delta t  \Bigg( \sum_{j=1}^{j^*-1} 2^{j-1}\left(\sum_{m=1}^{M_{j}} \sum_{k=1}^{c_{j,m}} \left(\sum_{s=1}^{c_{j,m-1}} A_{k,s}^{j,m}(t^n)  + S_k^{j,m}(t^{n+1})\right)u^n+ f_k^{j,m}(t^n)\right)  \nonumber \\
			&\quad +2^{j^*-1}\left(\sum_{m=1}^{M_{j^*}-1} \sum_{k=1}^{c_{j^*,m}} \left(\sum_{s=1}^{c_{j^*,m-1}} A_{k,s}^{j^*,m}(t^n)  + S_k^{j^*,m}(t^{n+1})\right)u^n + f_k^{j^*,m}(t^n)\right) \nonumber \\
			& \quad + 2^{j^*-1}c_{j^*,M_{j^*}}\left(\sum_{s=1}^{c_{j^*,M_{j^*}-1}} A_{k^*,s}^{j^*,M_{j^*}}(t^n)  + S_{k^*}^{j^*,M_{j^*}}(t^{n+1}) \right)u^n+f_{k^*}^{j^*,M_{j^*}}(t^n) \Bigg)+ O(\Delta t^2) \nonumber\\
			=&u^n-\Delta t  \Bigg( \sum_{j=1}^{j^*-1} 2^{j-1}\left(\widetilde{Z}_{1,j}u^n+ f_{j}(t^n)\right)  \nonumber \\
			&\quad -2^{j^*-1}\left(\sum_{m=1}^{M_{j^*}-1} \sum_{k=1}^{c_{j^*,m}} \left(\sum_{s=1}^{c_{j^*,m}} A_{k,s}^{j^*,m}(t^n)  + S_k^{j^*,m}(t^{n+1})\right)u^n + f_k^{j^*,m}(t^n)\right) \nonumber \\
			& \quad - 2^{j^*-1}c_{j^*,M_{j^*}}\left(\sum_{s=1}^{c_{j^*,M_{j^*}-1}} A_{k^*,s}^{j^*,M_{j^*}}(t^n)  + S_{k^*}^{j^*,M_{j^*}}(t^{n+1}) \right)u^n+f_{k^*}^{j^*,M_{j^*}}(t^n) \Bigg)+ O(\Delta t^2)
			\label{eq.full.proof.induction1} \\
			u^{n,j^*,M_{j^*}}=&u^n-\Delta t  \left( \sum_{j=1}^{j^*} 2^{j-1}\left( \sum_{l=1}^{M_{j}} \sum_{k=1}^{c_{j,m}} \left(\sum_{s=1}^{c_{j,m-1}} A_{k,s}^{j,l}(t^n)  + S_k^{j,l}(t^{n+1})\right)u^n \right)\right)  \nonumber \\
			& - \Delta t  \left( \sum_{j=1}^{j^*} 2^{j-1}\left(\sum_{m=1}^{M_{j}} \sum_{k=1}^{c_{j,m}} f_k^{j,m}(t^n) \right)\right) + O(\Delta t^2) \nonumber\\
			=& u^n-\Delta t  \left( \sum_{j=1}^{j^*} 2^{j-1}\left(\widetilde{Z}_{j}(t^n)u^n+ f_j(t^n)\right) \right)+ O(\Delta t^2).
		\end{align}
		
		Setting $j^*=J$ and from Algorithm \ref{alg.hybrid.general}, we have
		\begin{align}
			\bar{u}^{n,J,M_J}_{k^*}=u^{n,J,M_{J}}_{k^*}=& u^n-\Delta t  \Bigg( \sum_{j=1}^{J-1} 2^{j-1}\left(\widetilde{Z}_{j}(t^{n})u^n+ f_j(t^n)\right)  \nonumber \\
			&\quad -2^{J-1}\left(\sum_{m=1}^{M_{J}-1} \sum_{k=1}^{c_{J,m}} \left(\sum_{s=1}^{c_{J,m-1}} A_{k,s}^{J,m}(t^n)  + S_k^{J,m}(t^{n+1})\right)u^n + f_k^{J,m}(t^n)\right) \nonumber \\
			& \quad - 2^{J-1}c_{J,M_J}\left(\sum_{s=1}^{c_{J,M_j-1}} A_{k^*,s}^{J,M_{J}}(t^n)  + S_{k^*}^{J,M_{J}}(t^{n+1}) \right)u^n+f_{k}^{J,M_{J}}(t^n) \Bigg)+ O(\Delta t^2)
			\label{eq.full.induction2.J.1}\\
			\bar{u}^{n,J,M_J}=u^{n,J,M_J}=&u^n-\Delta t  \left( \sum_{j=1}^{J} 2^{j-1}\left(\widetilde{Z}_{j}(t^{n})u^n+ f_j(t^n)\right) \right)+ O(\Delta t^2).
			\label{eq.full.induction2.J.2}
		\end{align}
		For the simplicity of notation, we denote $\widetilde{j}^*=2J-j^*$.
		We next use mathematical induction to show that for any $J+1\leq j^* \leq  2J-1$ and $1\leq k^* \leq c_{j^*,M_{j^*}}$, we have
		\begin{align}
			\bar{u}^{n,j^*,M_{j^*}}_{k^*}=& u^n-\Delta t  \Bigg( \sum_{j=1}^{\widetilde{j}^*-1} 2^{j-1}\left(\widetilde{Z}_{j}(t^n)u^n+ f_{j}(t^n)\right) + \sum_{j=\widetilde{j}^*+1}^{J} 2^{\widetilde{j}^*-1}\left(\widetilde{Z}_{j}(t^n)u^n+ f_j(t^n)\right)\Bigg) \nonumber \\
			&- \Delta t 2^{\widetilde{j}^*-2}\Bigg(\widetilde{Z}_{\widetilde{j}^*}(t^n)u^n+f_{\widetilde{j}^*}(t^n)\Bigg) \nonumber\\
			&-\Delta t 2^{\widetilde{j}^*-2}\sum_{m=1}^{M_{\widetilde{j}^*}-1} \sum_{k=1}^{c_{\widetilde{j}^*,m}} \left(\left(\sum_{s=1}^{c_{\widetilde{j}^*,m-1}} A_{k,s}^{\widetilde{j}^*,m}(t^n)  + S_k^{\widetilde{j}^*,m}(t^{n+1})\right)u^n + f_k^{\widetilde{j}^*,m}(t^n)\right) \nonumber\\
			&- \Delta t 2^{\widetilde{j}^*-2}c_{\widetilde{j}^*,M_{\widetilde{j}^*}}\left(\left(\sum_{s=1}^{c_{\widetilde{j}^*,M_{\widetilde{j}^*}-1}} A_{k^*,s}^{\widetilde{j}^*,M_{\widetilde{j}^*}}(t^n)+ S_{k^*}^{\widetilde{j}^*,M_{J}}(t^{n+1}) \right)u^n+f_{k^*}^{\widetilde{j}^*,M_{\widetilde{j}^*}}(t^n)\right) \nonumber\\
			&-\Delta t 2^{\widetilde{j}^*-1} \Bigg(\left(\sum_{j=J+1}^{j^*-1} \left(\widetilde{Z}_{j}(t^{n})u^n+ f_j(t^n)\right) \right) \nonumber \\
			& +\left(\sum_{m=1}^{M_{j^*}-1} \sum_{k=1}^{c_{j^*,m}} \left(\sum_{s=1}^{c_{j^*,m-1}} A_{k,s}^{j^*,m}(t^n)  + S_k^{j^*,m}(t^{n+1})\right)u^n + f_k^{j^*,m}(t^n)\right) \nonumber \\
			&  + c_{j^*,M_{j^*}}\left(\left(\sum_{s=1}^{c_{j^*,M_{j^*}-1}} A_{k^*,s}^{J,M_{j^*}}(t^n)  + S_{k^*}^{j^*,M_{j^*}}(t^{n+1}) \right)u^n+f_{k^*}^{j^*,M_{j^*}}(t^n)\right) \Bigg)+ O(\Delta t^2),
			\label{eq.full.induction.1}\\
			\bar{u}^{n,j^*,M_{j^*}}=&u^n-\Delta t  \left( \sum_{j=1}^{\widetilde{j}^*-1} 2^{j-1}\left(\widetilde{Z}_{j}(t^{n})u^n+ f_j(t^n)\right) + \sum_{j=\widetilde{j}^*}^{J} 2^{\widetilde{j}^*-1}\left(\widetilde{Z}_{j}(t^{n})u^n+ f_j(t^n)\right) \right) \nonumber \\
			&-\Delta t 2^{\widetilde{j}^*-1} \Bigg(\sum_{j=J+1}^{j^*} \left(\widetilde{Z}_{j}(t^{n})u^n+ f_j(t^n)\right)  \Bigg)+ O(\Delta t^2).
			\label{eq.full.induction.2}
		\end{align}
		We first show that (\ref{eq.full.induction.1})--(\ref{eq.full.induction.2}) holds when $j^*=J+1$.
		Set $u^{n,J+1,0}= \bar{u}^{n,J,M_J}$ and $ u_k^{n,J+1,0} = \bar{u}_k^{n,J,M_J}$ for $k =1, 2, ..., c_{J,M_J}$. 
		For any $1\leq k^*\leq c_{J+1,1}$, we have
		\begin{align}
			&u^{n,J+1,1}_{k^*} \nonumber\\
			=&\left(\cI+\Delta t2^{J-1}c_{J+1,1}S_{k^*}^{J+1,1}(t^{n+1})\right)^{-1} \Bigg(u^{n,J+1,0} \nonumber\\
			&\hspace{3cm}-\Delta t 2^{J-1}c_{J+1,1} \left( \sum_{s=1}^{c_{J+1,0}} A_{k^*,s}^{J+1,1}(t^n)u^{n,J+1,0}_s +f_{k^*}^{J+1,1}(t^{n})\right)\Bigg) \nonumber\\
			=& \left(\cI-\Delta t2^{J-1}c_{J+1,1}  S_{k^*}^{J+1,1}(t^{n+1})\right) \Bigg(u^{n,J+1,0} \nonumber\\
			&\hspace{3cm} -\Delta t 2^{J-1}c_{J+1,1} \left( \sum_{s=1}^{c_{J+1,0}} A_{k^*,s}^{J+1,1}(t^n)u^{n,J+1,0}_s +f_{k^*}^{J+1,1}(t^{n})\right) \Bigg) +(\Delta t^2) \nonumber\\
			=&u^{n,J+1,0} -\Delta t 2^{J-1}c_{J+1,1} \Bigg(\sum_{s=1}^{c_{J+1,0}} A_{k^*,s}^{J+1,1}(t^n)u^{n,J+1,0}_s + S_{k^*}^{J+1,1}(t^{n+1})u^{n,J+1,0} \nonumber\\
			&\hspace{8cm}+ f_{k^*}^{J+1,1}(t^{n})\Bigg) + O(\Delta t^2) \nonumber\\
			=&u^n-\Delta t  \left( \sum_{j=1}^{J} 2^{j-1}\left(\widetilde{Z}_{j}(t^{n})u^n+ f_j(t^n)\right) \right) \nonumber\\
			&-\Delta t 2^{J-1}c_{J+1,1} \left( \left(\sum_{s=1}^{c_{J+1,0}} A_{k^*,s}^{J+1,1}(t^n) + S_{k^*}^{J+1,1}(t^{n+1})\right)u^n + f_{k^*}^{J+1,1}(t^{n})\right) + O(\Delta t^2)
		\end{align}
		and
		\begin{align}
			u^{n,J+1,1}=&\frac{1}{c_{J+1,1}} \sum_{k=1}^{c_{J+1,1}} u^{n,J+1,1}_k \nonumber\\
			=&u^n-\Delta t  \left( \sum_{j=1}^{J} 2^{j-1}\left(\widetilde{Z}_{j}(t^{n})u^n+ f_j(t^n)\right) \right) \nonumber\\
			&-\Delta t 2^{J-1}\sum_{k=1}^{c_{J+1,1}} \left( \left(\sum_{s=1}^{c_{J+1,0}} A_{k,s}^{J+1,1}(t^n) + S_k^{J+1,1}(t^{n+1})\right)u^n + f_k^{J+1,1}(t^{n})\right) + O(\Delta t^2).
		\end{align}
		Repeating the process, we can show that for any $1\leq m^* \leq M_{J+1}$, we have
		\begin{align}
			&u^{n,J+1,m^*}_{k^*} \nonumber\\
			=&u^n-\Delta t  \left( \sum_{j=1}^{J} 2^{j-1}\left(\widetilde{Z}_{j}(t^{n})u^n+ f_j(t^n)\right) \right) \nonumber\\
			&-\Delta t 2^{J-1} \sum_{m=1}^{m^*-1}\sum_{k=1}^{c_{J+1,m}}\left( \left(\sum_{s=1}^{c_{J+1,m-1}} A_{k,s}^{J+1,m}(t^n) + S_k^{J+1,m}(t^{n+1})\right)u^n + f_k^{J+1,m}(t^{n})\right) \nonumber\\
			&-\Delta t 2^{J-1}c_{J+1,m^*} \left( \left(\sum_{s=1}^{c_{J+1,m^*}} A_{k^*,s}^{J+1,m^*}(t^n) + S_{k^*}^{J+1,m^*}(t^{n+1})\right)u^n + f_{k^*}^{J+1,m^*}(t^{n})\right) + O(\Delta t^2).
			\label{eq.full.induction.L0.1}
		\end{align}
		and
		\begin{align}
			u^{n,J+1,m^*}=&\frac{1}{c_{J+1,m^*}} \sum_{k=1}^{c_{J+1,m^*}} u^{n,J+1,m^*} \nonumber\\
			=&u^n-\Delta t  \left( \sum_{j=1}^{J} 2^{j-1}\left(\widetilde{Z}_{j}(t^{n})u^n+ f_j(t^n)\right) \right) \nonumber\\
			&-\Delta t 2^{J-1} \sum_{m=1}^{m^*}\sum_{k=1}^{c_{J+1,m}}\Bigg( \left(\sum_{s=1}^{c_{J+1,m-1}} A_{k,s}^{J+1,m}(t^n) + S_k^{J+1,m}(t^{n+1})\right)u^n \nonumber\\
			& \hspace{5cm}+ f_k^{J+1,l}(t^{n})\Bigg) + O(\Delta t^2).
			\label{eq.full.induction.L0.2}
		\end{align}
		Set $m^*=M_{J+1}$ in (\ref{eq.full.induction.L0.1}) and (\ref{eq.full.induction.L0.2}). We get
		\begin{align}
			&u^{n,J+1,M_{J+1}}_{k^*} \nonumber\\
			=&u^n-\Delta t  \left( \sum_{j=1}^{J} 2^{j-1}\left(\widetilde{Z}_{j}(t^{n})u^n+ f_j(t^n)\right) \right) \nonumber\\
			&-\Delta t 2^{J-1} \sum_{m=1}^{M_{J+1}-1}\sum_{k=1}^{c_{J+1,m}}\left( \left(\sum_{s=1}^{c_{J+1,m-1}} A_{k,s}^{J+1,m}(t^n) + S_k^{J+1,m}(t^{n+1})\right)u^n + f_k^{J+1,m}(t^{n})\right) \nonumber\\
			&-\Delta t 2^{J-1}c_{J+1,M_{J+1}} \Bigg( \left(\sum_{s=1}^{c_{J+1,M_{J+1}-1}} A_{k^*,s}^{J+1,M_{J+1}}(t^n) + S_{k^*}^{J+1,M_{J+1}}(t^{n+1})\right)u^n \nonumber\\
			& \hspace{5cm}+ f_{k^*}^{J+1,M_{J+1}}(t^{n})\Bigg) + O(\Delta t^2).
			\label{eq.full.induction.j0.L.1}
		\end{align}
		and
		\begin{align}
			u^{n,J+1,M_{J+1}}	=&u^n-\Delta t  \left( \sum_{j=1}^{J} 2^{j-1}\left(\widetilde{Z}_{j}(t^{n})u^n+ f_j(t^n)\right) \right) \nonumber\\
			&-\Delta t 2^{J-1} \left(\widetilde{Z}_{J+1}(t^n)+ f_{J+1}(t^n)\right) + O(\Delta t^2).
			\label{eq.full.induction.j0.L.2}
		\end{align}
		For any $1\leq k^* \leq M_{J+1}$, we compute
		\begin{align}
			&\bar{u}^{n,J+1,M_{J+1}}_{k^*} \nonumber\\
			=&\frac{1}{2} u_{k^*}^{n,J+1,M_{J+1}}+ \frac{1}{2} u_{k^*}^{n,J-1,M_{J-1}} \nonumber\\
			=&u^n-\Delta t  \left( \sum_{j=1}^{J-2} 2^{j-1}\left(\widetilde{Z}_{j}(t^{n})u^n+ f_j(t^n)\right) +  2^{J-2}\left(\widetilde{Z}_{J}(t^{n})u^n+ f_J(t^n)\right)  \right) \nonumber\\
			&-\Delta t2^{J-3}\left(\widetilde{Z}_{J-1}(t^n)u^n+f_{J-1}(t^n)\right) \nonumber\\
			& -\Delta t 2^{J-3}\sum_{m=1}^{M_{J-1}-1} \sum_{k=1}^{c_{J-1,m}} \left(\left(\sum_{s=1}^{c_{J-1,m-1}} A_{k,s}^{J-1,m}(t^n)  + S_k^{J-1,m}(t^{n+1})\right)u^n + f_k^{J-1,m}(t^n)\right) \nonumber\\
			&-\Delta t 2^{J-3} c_{J-1,M_{J-1}}\left(\left(\sum_{s=1}^{c_{J-1,M_{J-1}-1}} A_{k^*,s}^{J-1,M_{J-1}}(t^n)+ S_{k^*}^{J-1,M_{J-1}} (t^{n+1})\right)u^n+f_{k^*}^{J-1,M_{J-1}}(t^n)\right) \nonumber \\
			&-\Delta t 2^{J-2} \sum_{m=1}^{M_{J+1}-1}\sum_{k=1}^{c_{J+1,m}}\left( \left(\sum_{s=1}^{c_{J+1,m-1}} A_{k,s}^{J+1,m}(t^n) + S_k^{J+1,m}(t^{n+1})\right)u^n + f_k^{J+1,m}(t^{n})\right) \nonumber\\
			&-\Delta t 2^{J-2}c_{J+1,M_{J+1}} \Bigg( \left(\sum_{s=1}^{c_{J+1,M_{J+1}-1}} A_{k^*,s}^{J+1,M_{J+1}}(t^n) + S_{k^*}^{J+1,M_{J+1}}(t^{n+1})\right)u^n \nonumber\\
			&\hspace{5cm} + f_{k^*}^{J+1,M_{J+1}}(t^{n})\Bigg) + O(\Delta t^2).
		\end{align}
		and
		\begin{align}
			\bar{u}^{n,J+1,M_{J+1}}=&\frac{1}{c_{J+1,M_{J+1}}} \sum_{k=1}^{c_{J+1,M_{J+1}}} \bar{u}_k^{n,J+1,M_{J+1}}  \nonumber\\
			=& u^n-\Delta t  \left( \sum_{j=1}^{J-2} 2^{j-1}\left(\widetilde{Z}_{j}(t^{n})u^n+ f_j(t^n)\right) + \sum_{j=J-1}^{J} 2^{J-2}\left(\widetilde{Z}_{j}(t^{n})u^n+ f_j(t^n)\right) \right) \nonumber \\
			&-\Delta t 2^{J-2} \left(\widetilde{Z}_{J+1}(t^{n})u^n+ f_{J+1}(t^n)\right) + O(\Delta t^2).
		\end{align}
		Therefore, (\ref{eq.full.induction.1})--(\ref{eq.full.induction.2}) hold for $j^*=J-1$. 
		
		Assume (\ref{eq.full.induction.1})--(\ref{eq.full.induction.2}) hold for $j^*=j_1-1\geq J+1$. Set $u^{n,j_1,0}= \bar{u}^{n,j_1-1,M_{j_1-1}}$ and $ u_k^{n,j_1,0} = \bar{u}_k^{n,j_1-1,M_{j_1-1}}$ for $k =1, 2, ..., c_{j_1-1,M_{j_1-1}}$. Denote $\widetilde{j}_1=2J-j_1$. For any $1\leq k^* \leq M_{j_1}$, we compute
		\begin{align}
			&u^{n,j_1,1}_{k^*} \nonumber\\
			=&\left(\cI+\Delta t2^{\widetilde{j}_1}c_{j_1,1}S_{k^*}^{j_1,1}(t^{n+1})\right)^{-1} \left(u^{n,j_1,0} +\Delta t 2^{\widetilde{j}_1}c_{j_1,1} \left( \sum_{s=1}^{c_{j_1,0}} A_{k^*,s}^{j_1,1}(t^n)u^{n,j_1,0}_s +f_{k^*}^{j_1,1}(t^{n})\right)\right) \nonumber\\
			=& \left(\cI-\Delta t2^{\widetilde{j}_1}c_{j_1,1} S_k^{j_1,1}(t^{n+1})\right) \left(u^{n,j_1,0} +\Delta t 2^{\widetilde{j}_1}c_{j_1,1} \left( \sum_{s=1}^{c_{j_1,0}} A_{k^*,s}^{j_1,1}(t^n)u^{n,j_1,0}_s +f_{k^*}^{j_1,1}(t^{n})\right) \right) +(\Delta t^2) \nonumber\\
			=&u^{n,j_1,0} -\Delta t 2^{\widetilde{j}_1}c_{j_1,1} \left( \left(\sum_{s=1}^{c_{j_1,0}} A_{k^*,s}^{j_1,1}(t^n) + S_{k^*}^{j_1,1}(t^{n+1})\right)u^{n,j_1,0}_s + f_{k^*}^{j_1,1}(t^{n})\right) + O(\Delta t^2) \nonumber\\
			=&u^n-\Delta t  \left( \sum_{j=1}^{\widetilde{j}_1} 2^{j-1}\left(\widetilde{Z}_{j}(t^{n})u^n+ f_j(t^n)\right) + \sum_{j=\widetilde{j}_1+1}^{J} 2^{\widetilde{j}_1}\left(\widetilde{Z}_{j}(t^{n})u^n+ f_j(t^n)\right) \right) \nonumber\\
			&-\Delta t 2^{\widetilde{j}_1} \Bigg(\sum_{j=J+1}^{j_1-1} \left(\widetilde{Z}_{j}(t^{n})u^n+ f_j(t^n)\right)  \Bigg) \nonumber\\
			&-\Delta t 2^{\widetilde{j}_1}c_{j_1,1} \left( \left(\sum_{s=1}^{c_{j_1,0}} A_{k^*,s}^{j_1,1}(t^n) + S_{k^*}^{j_1,1}(t^{n+1})\right)u^n + f_{k^*}^{j_1,1}(t^{n})\right) + O(\Delta t^2).
		\end{align}
		and
		\begin{align}
			u^{n,j_1,1}=&\frac{1}{c_{j_1,1}} \sum_{k=1}^{c_{j_1,1}} u_k^{n,j_1,1} \nonumber\\
			=&u^n-\Delta t  \left( \sum_{j=1}^{\widetilde{j}_1} 2^{j-1}\left(\widetilde{Z}_{j}(t^{n})u^n+ f_j(t^n)\right) + \sum_{j=\widetilde{j}_1+1}^{J} 2^{\widetilde{j}_1}\left(\widetilde{Z}_{j}(t^{n})u^n+ f_j(t^n)\right) \right) \nonumber\\
			&-\Delta t 2^{\widetilde{j}_1} \Bigg(\sum_{j=J+1}^{j_1-1} \left(\widetilde{Z}_{j}(t^{n})u^n+ f_j(t^n)\right)  \Bigg) \nonumber\\
			&-\Delta t 2^{\widetilde{j}_1} \sum_{k=1}^{c_{j_1,1}}\left( \left(\sum_{s=1}^{c_{j_1,1}} A_{k,s}^{j_1,1}(t^n) + S_k^{j_1,1}(t^{n+1})\right)u^n + f_k^{j_1,1}(t^{n})\right) + O(\Delta t^2).
		\end{align}
		Repeating the process, we can show that for any $1\leq m^* \leq L_{J-1}$, we have
		\begin{align}
			&u^{n,j_1,m^*}_{k^*} \nonumber\\
			=&u^n-\Delta t  \left( \sum_{j=1}^{\widetilde{j}_1} 2^{j-1}\left(\widetilde{Z}_{j}(t^{n})u^n+ f_j(t^n)\right) + \sum_{j=\widetilde{j}_1+1}^{J} 2^{\widetilde{j}_1}\left(\widetilde{Z}_{j}(t^{n})u^n+ f_j(t^n)\right) \right) \nonumber\\
			&-\Delta t 2^{\widetilde{j}_1} \Bigg(\sum_{j=J+1}^{j_1-1} \left(\widetilde{Z}_{j}(t^{n})u^n+ f_j(t^n)\right)  \Bigg) \nonumber\\
			&- \Delta t 2^{\widetilde{j}_1} \sum_{m=1}^{m^*-1}\sum_{k=1}^{c_{j_1,m}}\left( \left(\sum_{s=1}^{c_{j_1,m-1}} A_{k,s}^{j_1,m-1}(t^n) + S_k^{j_1,m-1}(t^{n+1})\right)u^n + f_k^{j_1,m-1}(t^{n})\right) \nonumber\\
			&-\Delta t 2^{\widetilde{j}_1}c_{j_1,m^*} \left( \left(\sum_{s=1}^{c_{j_1,m^*-1}} A_{k^*,s}^{j_1,m^*}(t^n) + S_{k^*}^{j_1,m^*}(t^{n+1})\right)u^n + f_{k^*}^{j_1,m^*}(t^{n})\right) + O(\Delta t^2)
			\label{eq.full.induction.L.1}
		\end{align}
		and
		\begin{align}
			u^{n,j_1,m^*}=&\frac{1}{c_{j_1,m^*}} \sum_{k=1}^{c_{j_1,m^*}} u^{n,j_1,m^*} \nonumber\\
			=&u^n-\Delta t  \left( \sum_{j=1}^{\widetilde{j}_1} 2^{j-1}\left(\widetilde{Z}_{j}(t^{n})u^n+ f_j(t^n)\right) + \sum_{j=\widetilde{j}_1+1}^{J} 2^{\widetilde{j}_1}\left(\widetilde{Z}_{j}(t^{n})u^n+ f_j(t^n)\right) \right) \nonumber \\
			&-\Delta t 2^{\widetilde{j}_1} \Bigg(\sum_{j=J-1}^{j_1-1} \left(\widetilde{Z}_{j}(t^{n})u^n+ f_j(t^n)\right)  \Bigg) \nonumber\\
			&- \Delta t 2^{\widetilde{j}_1} \sum_{m=1}^{m^*}\sum_{k=1}^{c_{j_1,m}}\left( \left(\sum_{s=1}^{c_{j_1,m-1}} A_{k,s}^{j_1,m}(t^n) + S_k^{j_1,m}(t^{n+1})\right)u^n + f_k^{j_1,m}(t^{n})\right)+ O(\Delta t^2)
			\label{eq.full.induction.L.2}
		\end{align}
		Set $m^*=M_{j_1}$ in (\ref{eq.full.induction.L.1})--(\ref{eq.full.induction.L.2}), we compute
		\begin{align}
			&\bar{u}^{n,j_1,M_{j_1}}_{k^*} \nonumber\\
			=&\frac{1}{2} u_{k^*}^{n,j_1,M_{j_1}}+ \frac{1}{2} u_{k^*}^{n,\widetilde{j}_1,M_{\widetilde{j}_1}} \nonumber\\
			=&u^n-\Delta t  \left( \sum_{j=1}^{\widetilde{j}_1-1} 2^{j-1}\left(\widetilde{Z}_{j}(t^{n})u^n+ f_j(t^n)\right) + \sum_{j=\widetilde{j}_1+1}^J 2^{\widetilde{j}_1-1}\left(\widetilde{Z}_{j}(t^{n})u^n+ f_{j}(t^n)\right)  \right) \nonumber\\
			&-\Delta t2^{\widetilde{j}_1-2}\left(\widetilde{Z}_{\widetilde{j}_1}(t^n)u^n+f_{\widetilde{j}_1}(t^n)\right) \nonumber\\
			& -\Delta t 2^{\widetilde{j}_1-2}\sum_{m=1}^{M_{\widetilde{j}_1}-1} \sum_{k=1}^{c_{\widetilde{j}_1,m}} \left(\left(\sum_{s=1}^{c_{\widetilde{j}_1,m-1}} A_{k,s}^{\widetilde{j}_1,m}(t^n)  + S_k^{\widetilde{j}_1,m}(t^{n+1})\right)u^n + f_k^{\widetilde{j}_1,m}(t^n)\right) \nonumber\\
			&-\Delta t 2^{\widetilde{j}_1-2} c_{\widetilde{j}_1,M_{\widetilde{j}_1}}\left(\left(\sum_{s=1}^{c_{\widetilde{j}_1,M_{\widetilde{j}_1}}} A_{k^*,s}^{\widetilde{j}_1,M_{\widetilde{j}_1}}(t^n)+ S_{k^*}^{\widetilde{j}_1,M_{\widetilde{j}_1}}(t^{n+1}) \right)u^n+f_{k^*}^{\widetilde{j}_1,M_{\widetilde{j}_1}}(t^n)\right) \nonumber \\
			&-\Delta t 2^{\widetilde{j}_1-1} \Bigg(\sum_{j=J+1}^{j_1-1} \left(\widetilde{Z}_{j}(t^{n})u^n+ f_j(t^n)\right)  \Bigg) \nonumber\\
			&-\Delta t 2^{\widetilde{j}_1-1} \sum_{m=1}^{M_{j_1}-1}\sum_{k=1}^{c_{j_1,m}}\left( \left(\sum_{s=1}^{c_{j_1,m-1}} A_{k,s}^{j_1,m}(t^n) + S_k^{j_1,m}(t^{n+1})\right)u^n + f_k^{j_1,m}(t^{n})\right) \nonumber\\
			&-\Delta t 2^{\widetilde{j}_1-1}c_{j_1,M_{j_1}} \left( \left(\sum_{s=1}^{c_{j_1,M_{j_1}-1}} A_{k^*,s}^{j_1,M_{j_1}}(t^n) + S_{k^*}^{j_1,M_{j_1}}(t^{n+1})\right)u^n + f_{k^*}^{j_1,M_{j_1}}(t^{n})\right) + O(\Delta t^2)
		\end{align}
		and
		\begin{align}
			\bar{u}^{n,j_1,M_{j_1}}=&\frac{1}{c_{j_1,M_{j_1}}} \sum_{k=1}^{c_{j_1,M_{j_1}}} \bar{u}_k^{n,j_1,M_{j_1}}  \nonumber\\
			=& u^n-\Delta t  \left( \sum_{j=1}^{\widetilde{j}_1-1} 2^{j-1}\left(\widetilde{Z}_{j}(t^{n})u^n+ f_j(t^n)\right) + \sum_{j=\widetilde{j}_1}^{J} 2^{\widetilde{j}_1-1}\left(\widetilde{Z}_{j}(t^{n})u^n+ f_j(t^n)\right) \right) \nonumber \\
			&-\Delta t 2^{\widetilde{j}_1-1} \sum_{j=J+1}^{j_1}\left(\widetilde{Z}_{j}(t^{n})u^n+ f_{j}(t^n)\right) + O(\Delta t^2).
		\end{align}
		Therefore (\ref{eq.full.induction.1})--(\ref{eq.full.induction.2}) hold for any $J+1\leq j^*\leq 2J-1$.
		
		Set $j^*=2J-1$ in (\ref{eq.full.induction.1})--(\ref{eq.full.induction.2}), we have
		\begin{align}
			&\bar{u}^{n,2J-1,M_{2J-1}}_{k^*} \nonumber\\
			=& u^n-\Delta t  \Bigg( \sum_{j=2}^{J} \left(\widetilde{Z}_{j}(t^{n})u^n+ f_j(t^n)\right)\Bigg) \nonumber \\
			-& \Delta t 2^{-1}\Bigg(\widetilde{Z}_{1}(t^n)u^n+f_{1}(t^n)\Bigg) \nonumber\\
			-&\Delta t 2^{-1}\sum_{m=1}^{M_{1}-1} \sum_{k=1}^{c_{1,m}} \left(\left(\sum_{s=1}^{c_{1,m-1}} A_{k,s}^{1,m}(t^n)  + S_k^{1,m}(t^{n+1})\right)u^n + f_k^{1,m}(t^n)\right) \nonumber\\
			-& \Delta t 2^{-1}c_{1,M_1}\left(\left(\sum_{s=1}^{c_{1,M_1-1}} A_{k^*,s}^{1,M_{1}}(t^n)+ S_{k^*}^{1,M_{1}}(t^{n+1}) \right)u^n+f_{k^*}^{1,M_{1}}(t^n)\right) \nonumber\\
			-&\Delta t \Bigg(\left(\sum_{j=J+1}^{2J-2} \left(\widetilde{Z}_{j}(t^{n})u^n+ f_j(t^n)\right) \right) \nonumber \\
			-&\left(\sum_{m=1}^{M_{2J-1}-1} \sum_{k=1}^{c_{2J-1,m}} \left(\sum_{s=1}^{c_{2J-1,m-1}} A_{k,s}^{2J-1,m}(t^n)  + S_k^{2J-1,m}(t^{n+1})\right)u^n + f_k^{2J-1,m}(t^n)\right) \nonumber \\
			-& c_{2J-1,M_{2J-1}}\left(\sum_{s=1}^{c_{2J-1,M_{2J-1}-1}} \left(A_{k^*,s}^{2J-1,M_{2J-1}}(t^n)  + S_{k^*}^{2J-1,M_{2J-1}}(t^{n+1}) \right)u^n+f_{k}^{2J-1,M_{2J-1}}(t^n)\right) \Bigg) \nonumber\\
			+& O(\Delta t^2),
			\label{eq.full.induction21.1}\\
			&\bar{u}^{n,2J-1,M_{2J-1}}=u^n-\Delta t    \sum_{j=1}^{J} \left(\widetilde{Z}_{j}(t^{n})u^n+ f_j(t^n)\right)  -\Delta t  \sum_{j=J+1}^{2J-1} \left(\widetilde{Z}_{j}(t^{n})u^n+ f_j(t^n)\right)  + O(\Delta t^2).
			\label{eq.full.induction21.2}
		\end{align}
		According to (\ref{eq.hybrid.final}), we have
		\begin{align}
			u^{n+1}=&\left( \cI+\Delta t S^*(t^{n+1}) \right)^{-1}\left(\bar{u}^{n,2J-1,M_{2J-1}}-\Delta t \left(\sum_{s=1}^{c_{2J-1,M_{2J-1}}} A_s^*(t^n)\bar{u}_k^{n,2J-1,M_{2J-1}}+f^*(t^n)\right)\right) \nonumber\\
			=& (\cI-\Delta t S^*(t^{n+1}))\Bigg(\bar{u}^{n,2J-1,M_{2J-1}}-\Delta t \Bigg(\sum_{s=1}^{c_{2J-1,M_{2J-1}}} A_s^*(t^n)\bar{u}_k^{n,2J-1,M_{2J-1}} \nonumber \\ 
			&\hspace{10cm} +f^*(t^n)\Bigg)\Bigg) +O(\Delta t^2) \nonumber\\
			=& \bar{u}^{n,2J-1,M_{2J-1}}-\Delta t S^*(t^{n+1})\bar{u}^{n,2J-1,M_{2J-1}}-\Delta t \Bigg(\sum_{s=1}^{c_{2J-1,M_{2J-1}}} A_s^*(t^n)\bar{u}_k^{n,2J-1,M_{2J-1}} \nonumber\\
			& \hspace{10cm} +f^*(t^n)\Bigg) +O(\Delta t^2) \nonumber\\
			=& u^n-\Delta t    \sum_{j=1}^{J} \left(\widetilde{Z}_{j}(t^{n})u^n+ f_j(t^n)\right)  -\Delta t  \sum_{j=J+1}^{2J-1} \left(\widetilde{Z}_{j}(t^{n})u^n+ f_j(t^n)\right) \nonumber\\
			&- \Delta t \left(\widetilde{Z}^*(t^n)u^n+ f^*(t^n)\right) + O(\Delta t^2) \nonumber\\
			=& u^n-\Delta t \widetilde{Z}(t^{n})u^n- \Delta tf(t^n)+O(\Delta t^2),
			\label{eq.Un.final}
		\end{align}
  where $f(t^n)$ is defined in (\ref{eq.full.f}).
  
		Putting (\ref{eq.Un.final}) and (\ref{eq.tildeu}) together, we have
		\begin{align}
			\|\tilde{u}^n-u^n\|_{\infty}=O(\Delta t^2).
		\end{align}
		The local error is of $O(\Delta t^2)$. Thus the global error is of $O(\Delta t)$.
		
	\end{proof}
	
	\bibliographystyle{abbrv}
	\bibliography{ref}
	
\end{document}